\documentclass{article}
\pdfoutput=1
\usepackage{etoolbox}
\newtoggle{nips}
\togglefalse{nips}
\newtoggle{nips_final}
\toggletrue{nips_final}
\usepackage{multirow,nicefrac}
\usepackage{makecell}
\usepackage{footnote}
\usepackage{longtable}
\usepackage{tablefootnote}
\usepackage[T1]{fontenc}
\usepackage[utf8]{inputenc}
\usepackage{amsmath,amsfonts,amssymb,mathrsfs}
\usepackage{mathtools}
\PassOptionsToPackage{numbers, }{natbib}
\usepackage{natbib}
\DeclareMathSymbol{\shortminus}{\mathbin}{AMSa}{"39}

\usepackage[breaklinks=true]{hyperref}
\usepackage{soul}
\usepackage{amsthm}

\newcommand{\Ptime}{\mathsf{P}}
\newcommand{\NPtime}{\mathsf{NP}}

\usepackage{breakcites,bm}

\hypersetup{colorlinks,citecolor=blue,linkcolor = blue}
\usepackage{verbatim} 
\usepackage{latexsym}
\usepackage{relsize}
\usepackage{thm-restate}
\usepackage{appendix}

\usepackage[capitalise,nameinlink]{cleveref}
\usepackage{xcolor}
\usepackage{dsfont}

\numberwithin{equation}{section}

\Crefname{equation}{Eq.}{Eqs.}
\Crefname{assumption}{Assumption}{Assumptions}
\Crefname{condition}{Condition}{Conditions}
\Crefname{oracle}{Oracle}{Oracles}

\def\X{{\mathcal X}}

\DeclareFontFamily{U}{mathx}{\hyphenchar\font45}
\DeclareFontShape{U}{mathx}{m}{n}{
      <5> <6> <7> <8> <9> <10>
      <10.95> <12> <14.4> <17.28> <20.74> <24.88>
      mathx10
      }{}
\DeclareSymbolFont{mathx}{U}{mathx}{m}{n}
\DeclareFontSubstitution{U}{mathx}{m}{n}
\DeclareMathAccent{\widecheck}{0}{mathx}{"71}
\DeclareMathAccent{\wideparen}{0}{mathx}{"75}

\newcommand{\ignore}[1]{}

	\theoremstyle{plain}
	\newtheorem{theorem}{Theorem}[section]

	\newtheorem{lemma}{Lemma}[section]
	\newtheorem{claim}[lemma]{Claim}
	\newtheorem{fact}[lemma]{Fact}
	\newtheorem{corollary}{Corollary}[section]
	\newtheorem{proposition}[lemma]{Proposition}
	\newtheorem{observation}[lemma]{Observation}

	\theoremstyle{definition}

	\newtheorem{definition}{Definition}[section]
	\newtheorem{example}{Example}[section]

	\newtheorem{assumption}{Assumption}
	\newtheorem{condition}{Condition}[section]
	\newtheorem{remark}{Remark}[section]
  \newtheorem{oracle}{Oracle}

\makeatletter
\newcommand{\neutralize}[1]{\expandafter\let\csname c@#1\endcsname\count@}
\makeatother
\makesavenoteenv{algorithm}

\usepackage{algorithm}
\usepackage[noend]{algpseudocode}

\newtheorem*{theorem*}{Theorem}
\newtheorem*{lemma*}{Lemma}
\newtheorem*{corollary*}{Corollary}
\newtheorem*{proposition*}{Proposition}
\newtheorem*{claim*}{Claim}
\newtheorem*{fact*}{Fact}
\newtheorem*{observation*}{Observation}

\newtheorem*{definition*}{Definition}
\newtheorem*{remark*}{Remark}
\newtheorem*{example*}{Example}

\newtheoremstyle{named}{}{}{\itshape}{}{\bfseries}{}{.5em}{\Cref{#3} {\normalfont (informal)} }
{}
\theoremstyle{named}

\theoremstyle{plain}

\DeclareMathAlphabet{\mathbfsf}{\encodingdefault}{\sfdefault}{bx}{n}

\newcommand{\norm}[1]{\|#1\|}
\newcommand{\ceil}[1]{\lceil #1 \rceil}
\newcommand{\floor}[1]{\lfloor #1 \rfloor}

\newcommand{\E}{\mathbb{E}}

\newcommand{\var}{\mathrm{Var}}

\newcommand{\nipsfinal}[2]{\iftoggle{nips_final}{#1 }{ #2}}

\newcommand{\arxivfinal}[1]{\iftoggle{nips}{}{ #1}}

\DeclareMathOperator*{\argmin}{arg\,min}

\let\Pr\relax
\DeclareMathOperator{\Pr}{\mathbb{P}}
\newcommand{\Exp}{\mathbb{E}}
\newcommand{\N}{\mathbb{N}}
\newcommand{\R}{\mathbb{R}}

\DeclareMathOperator{\BigOm}{\mathcal{O}}

\newcommand{\BigOh}[1]{\BigOm\left({#1}\right)}

\newcommand{\iidsim}{\overset{\mathrm{i.i.d}}{\sim}}
\newcommand{\Regret}{\mathrm{Regret}}
\newcommand{\dimy}{d_{y}}
\newcommand{\dimx}{d_{x}}
\newcommand{\dimu}{d_{u}}

\newcommand{\radM}{R_{\Mclass}}

\newcommand{\I}{\mathbb{I}}

\newcommand{\op}{\mathrm{op}}

\newcommand{\fro}{\mathrm{F}}
\newcommand{\calK}{\mathcal{K}}
\newcommand{\calF}{\mathcal{F}}

\newcommand{\Proj}{\mathrm{Proj}}

\newcommand{\A}{\mathcal{A}}

\def\bb{\mathbf{b}}
\def\bB{\mathbf{B}}
\def\bC{\mathbf{C}}
\def\etil{\tilde{\mathbf{e}}}
\def\bu{\mathbf{u}}
\def\X{\mathbf{X}}
\def\x{\mathbf{x}}

\def\w{\mathbf{w}}

\def\bu{\mathbf{u}}

\def\bA{\mathbf{A}}

\def\what{\hat{{w}}}

\def\xbar{\bar{\mathbf{x}}}

\newcommand{\veps}{\varepsilon}

\newcommand{\alg}{\mathsf{alg}}

\newcommand{\Reg}{\mathrm{Reg}}

\newcommand{\fhat}{\widehat{f}}

\newcommand{\drc}{\textsc{Drc}}
\newcommand{\dac}{\textsc{Dac}}

\newcommand{\errtext}[2]{\underbrace{#1}_{\text{{\color{blue} (#2) }}}}

\newcommand{\Ghat}{\widehat{G}}
\newcommand{\uex}{\mathbf{v}}

\newcommand{\bx}{b_x}

\newcommand{\calM}{\mathcal{M}}

\newcommand{\matz}{\mathbf{z}}
\newcommand{\matzst}{\matz^\star}

\newcommand{\matzhat}{\hat{\matz}}
\newcommand{\matztil}{\tilde{\matz}}

\newcommand{\matzbar}{\bar{\matz}}

\newcommand{\matzstt}[1][t]{\matz^{\star}_{#1}}

\newcommand{\nabtil}{\tilde{\nabla}}

\newcommand{\Bernoulli}{\mathrm{Bernoulli}}

\newcommand{\Ahat}{\hat{A}}
\newcommand{\Bhat}{\hat{B}}
\newcommand{\calD}{\mathcal{D}}

\newcommand{\mscomment}[1]{\noindent{\textcolor{red}{\{{\bf MS:} \em #1\}}}}
\newcommand{\msedit}[1]{{\textcolor{red}{#1}}}
\newcommand{\eh}[1]{$\ll$\textsf{\color{blue} EH : #1}$\gg$}
\newcommand{\pg}[1]{$\ll$\textsf{\color{magenta} PG : #1}$\gg$}

\title{Online Control of Unknown Time-Varying Dynamical Systems}
\author{Edgar Minasyan$^{1, 4}$, Paula Gradu$^{2}$, Max Simchowitz$^{3}$, Elad Hazan$^{1, 4}$ \\ \\
  $^1$ Computer Science, Princeton University \\
  $^2$ EECS, UC Berkeley \\
  $^3$ EECS, MIT \\
  $^4$ Google AI Princeton \\ \\
  \texttt{minasyan@princeton.edu, pgradu@berkeley.edu} \\
  \texttt{msimchow@mit.edu, ehazan@princeton.edu }  \\
  }

\begin{document}

\maketitle

\begin{abstract}

We study online control of time-varying linear systems with unknown dynamics in the nonstochastic control model. At a high level, we demonstrate that this setting is \emph{qualitatively harder} than that of either unknown time-invariant or known time-varying dynamics, and complement our negative results with algorithmic upper bounds in regimes where sublinear regret is possible. More specifically, we study regret bounds with respect to common classes of policies: Disturbance Action (SLS), Disturbance Response (Youla), and linear feedback policies. While these three classes are essentially equivalent for LTI systems, we demonstrate that these equivalences break down for time-varying systems. 

We prove a lower bound that no algorithm can obtain sublinear regret with respect to the first two classes unless a certain measure of system variability also scales sublinearly in the horizon. Furthermore, we show that offline planning over the state linear feedback policies is NP-hard, suggesting hardness of the online learning problem. 

On the positive side, we give an efficient algorithm that attains a sublinear regret bound against the class of Disturbance Response policies up to the aforementioned system variability term. In fact, our algorithm enjoys sublinear \emph{adaptive} regret bounds, which is a strictly stronger metric than standard regret and is more appropriate for time-varying systems. We sketch extensions to Disturbance Action policies and partial observation, and propose an inefficient algorithm for regret against linear state feedback policies. 

\ignore{
\mscomment{\textbf{OLD Abstract:}
We consider the problem of online control of unknown and partially observable time-varying linear dynamical systems. 
In this setting we study regret bounds with respect to common classes of policies, including linear and disturbance-action policies. 
We first give an efficient algorithm that attains a sublinear adaptive regret bound vs. the class of disturbance-action policies up to an additional system-variability term. We note that adaptive regret, a strictly stronger metric than standard regret, is the more appropriate performance measure for time-varying systems as introduced by recent work in this direction.
Next, we study whether the additional system-variability term, which can be linear in the worst case, is necessary. We answer this positively for two different reference policy classes. 
First, we prove that it is statistically impossible to obtain sublinear regret bounds against the class of disturbance-action controllers. This shows that the system-variability term is inherently necessary by any algorithm.
 Finally, we prove that it is computationally hard to obtain sublinear regret bounds vs. the class of time-invariant linear policies. This is in contrast to recent results in non-stochastic control for time-invariant linear systems. 
 }
}

\end{abstract}

\section{Introduction}

The control of linear time-invariant (LTI) dynamical systems is well-studied and understood. This includes classical methods from optimal control such as LQR and LQG, as well as robust $H_\infty$ control. Recent advances study regret minimization and statistical complexity for online linear control, in both stochastic and adversarial perturbation models. 
Despite this progress, 
rigorous mathematical guarantees for nonlinear control remain elusive: {nonlinear control is both statistically and computationally} intractable in general. 


In the face of these limitations, recent research has begun to study the rich continuum of settings which lie between LTI systems and generic nonlinear ones. The hope is to provide efficient and robust algorithms to solve the most general control problems that are tractable, and at the same time, to characterize precisely at which degree of nonlinearity no further progress can be made.

This paper studies the control of linear, time-varying (LTV) dynamical systems as one such point along this continuum. This is because the first-order Taylor approximation to the dynamics of any smooth nonlinear system about a given trajectory is an LTV system.
These approximations are widely popular because they allow for efficient planning,  as demonstrated by the success of iLQR and iLQG methods for nonlinear receding horizon control.
We study online control of discrete-time LTV systems, with dynamics and time-varying costs 
\begin{align}\label{eq:system}
 x_{t+1} &= A_t x_t + B_t u_t + w_t, \quad c_t(x_t,u_t):(x,u) \to \R.
\end{align}
Above,  $x_t$ is the state of the system, $u_t$ the control input, $w_t$ the disturbances, and $A_t,B_t$ the system matrices. Our results extend naturally to partial-state observation, where the controller observes linear projections of the state $y_t = C_t x_t$. 
We focus on the challenges introduced when the system matrices $A_t,B_t$ and perturbations $w_t$ are not known to the learner in advance, and can only be determined by live interaction with the changing systems.  

In this setting, we find that the overall change in system dynamics across time characterizes the difficulty of controlling the unknown LTV system. We define a measure, called system variability, which  quantifies this.  
We show both statistical and computational lower bounds as well as algorithmic upper bounds in terms of the system variabilility. Surprisingly, system variability does not impede the complexity of control when the dynamics are known \cite{gradu2020adaptive}.

\subsection{Contributions}

We consider the recently popularized nonstochastic model of online control, and study regret bounds with respect to common classes of policies: Disturbance Action (\textsc{Dac}/SLS \cite{wang2019system}), Disturbance Response (\textsc{Drc}/Youla \cite{youla1976modern}), and linear feedback policies. Planning over the third class of feedback policies in LTI systems admits efficient convex relaxations via the the first two parametrizations, \textsc{Dac} and \textsc{Drc}. This insight has been the cornerstone of both robust \cite{zhou1996robust,wang2019system} and online \cite{agarwal2019online,simchowitz2020improper} control. 

\textbf{Separation of parametrizations.} For linear time-varying systems, however, we find that equivalences between linear feedback, \dac{} and \drc{} fail to hold: we show that there are cases where any one of the three parametrizations exhibits strictly better control performance than the other two.

\textbf{Regret against convex parametrizations.} Our first set of results pertain to  \textsc{Dac} and \textsc{Drc} parametrizations, which are convex and  admit efficient optimization.   
We demonstrate that no algorithm can obtain sublinear regret with respect to these classes when faced with \emph{unknown}, LTV dynamics unless a certain measure of \emph{system variability} also scales sublinearly in the horizon. This is true even under full observation, controllable dynamics, and fixed control cost.
This finding is in direct contrast to recent work which shows sublinear regret \emph{is} attainable over LTV system dynamics if they are known  \cite{gradu2020adaptive}. 

We give an efficient algorithm that attains sublinear regret against these policy classes up to an additive penalty for the aforementioned system variability term found in our lower bound. When the system variability is sufficiently small, our algorithm recovers state-of-the-art results for unknown LTI system dynamics up to logarithmic factors.

In fact, our algorithm enjoys sublinear \emph{adaptive} regret \cite{hazan2009efficient}, a strictly stronger metric than standard regret which is more appropriate for time-varying systems. We also show that the stronger notion of adaptivity called strongly adaptive regret \cite{daniely15strongly} is out of reach in the partial information setting. 

\textbf{Regret against state feedback.}\label{paragraph:linear} Finally, we consider the  class of state feedback policies, which are linear feedback with memory length one. 
We show that full-information optimization over state feedback policies is computationally hard. This suggests that obtaining sublinear regret relative to these policies may be computationally prohibitive, though does not entirely rule out the possibility of improper learning.  However, improper learning cannot be done via the \drc{} or \dac{} relaxations in light of our policy class separation results.
Finally, we include an inefficient algorithm which attains sublinear (albeit nonparametric-rate) regret against state feedback control policies. 

\nipsfinal{
\subsubsection*{Paper Structure}
Discussion of relevant literature and relation to our work can be found in \Cref{sec:related}. In \Cref{sec:setting}, we formally introduce the setting of LTV nonstochastic control, the policy classes we study and our key result regarding their non-equivalence in the LTV setting (\Cref{thm:informal_separation}). Motivated by this non-equivalence, the remainder of the paper is split into the study of convex policies (\Cref{sec:dac_drc_body}) and of state feedback policies (\Cref{sec:linear_main}). In \Cref{sec:dac_drc_body}, we show that regret against the \dac{} and \drc{} classes cannot be sublinear unless the metric system variability (\Cref{def:variability}) itself is sublinear (\Cref{thm:main_lb}), and also propose \Cref{alg:control_unknown_main} whose adaptive regret scales at the rate of our lower bound plus a $T^{2/3}$ term (\Cref{thm:control_unknown_regret}). On the other hand, in \Cref{sec:linear_main} we show sublinear regret against state feedback policies is technically possible (\Cref{thm:K_regret_body}) with a computationally inefficient algorithm, but also provide a computational lower bound (\Cref{thm:body_hard_ness_lb}) {\em for planning} which reveals significant difficulties imposed by the LTV dynamics in this scenario as well. Finally, in \Cref{sec:conclusions} we pose several future directions, concerning both questions in LTV control, as well as the extension to nonlinear control.
}{}

\subsection{Related Work}\label{sec:related}

Our study of LTV systems is motivated by the widespread practical popularity of iterative linearization for nonlinear receding horizon control; e.g., the iLQR \citep{iLQR}, iLC \citep{moore2012iterative}, and iLQG \cite{todorov2005generalized} algorithms. Recent research has further demonstrated that near-optimal solutions to LTV approximations of dynamics confer stability guarantees onto the original nonlinear system of interest \cite{westenbroek2021stability}.

\textbf{Low-Regret Control:}  We study algorithms which enjoy sublinear regret for online control of LTV systems; that is, whose performance tracks a given benchmark of policies up to a term which is vanishing relative to the problem horizon.  \cite{abbasi2011regret} initiated the study of online control under the regret benchmark by introducing the online LQR problem:  where a learner is faced with an unknown LTI system, fixed costs and i.i.d. Gaussian disturbances, and must attain performance relative to the LQR-optimal policy. Bounds for this setting were later improved and refined in \cite{dean2018regret,mania2019certainty,cohen2019learning,simchowitz2020naive}, and extended to partial-state observation in \cite{anima3,anima2}. 
Our work instead adopts the \emph{nonstochastic control setting} \cite{agarwal2019online}, where the adversarially chosen (i.e. non-Gaussian) noise is considered to model the drift terms that arise in linearizations of nonlinear terms, and where costs may vary with time. \cite{agarwal2019online} consider known system dynamics, later extended to unknown systems under both full-state \cite{hazan2019nonstochastic} and partial-state observation \cite{simchowitz2020improper,simchowitz2020making}. The study of nonstochastic control of known LTV dynamics was taken up in \cite{gradu2020adaptive}, with  parallel work by \cite{qu2021stable} considering known LTV dynamics under stochastic noise.

\textbf{Unknown LTV dynamics:} Our work is the first to consider online (low-regret) control of unknown LTV systems in \emph{any} model. There is, however, a rich body of classical work on adaptive control of LTV systems \cite{middleton1988adaptive,tsakalis1993linear}. These guarantees focus more heavily on error sensitivity and stability; they only permit dynamical recovery up to error that scales linearly in system noise,  and thus guarantee only (vacuous) linear-in-horizon regret. More recent work has studied identification (but not online control) of an important LTV class called switching systems \cite{ozay2015set,sarkar2019nonparametric}.

    

\textbf{Online Convex Optimization:} We make extensive use of techniques from the field of online convex optimization \citep{cesa2006prediction,hazan2016introduction}. Most relevant to our work is the literature on adapting to changing environments in online learning, which starts from the works of \cite{herbster1998tracking,bousquet2002tracking}. The notion of adaptive regret was introduced in \cite{hazan2009efficient} and significantly studied since as a metric for adaptive learning in OCO \citep{adamskiy2016closer,zhang2019adaptive}. \cite{daniely15strongly} proposed to strengthen adaptive regret and the stronger metric has been shown to imply results over dynamic regret \citep{zhang18dynamic}.




\textbf{Recent nonlinear control literature:} 
 Recent research has also studied provably guarantees in various complementary (but incomparable) models: planning regret in nonlinear control \cite{agarwal2021regret}, adaptive nonlinear control under linearly-parameterized uncertainty \cite{boffi2020regret}, online model-based control with access to non-convex planning oracles \cite{kakade2020information}, and control with nonlinear observation models \cite{mhammedi2020learning,dean2020certainty}.

\newcommand{\Htwo}{\mathcal{H}_2}
\newcommand{\Hinf}{\mathcal{H}_{\infty}}
\newcommand{\Id}{\mathbb{I}}
\newcommand{\boldK}{\mathbf{K}} %
\newcommand{\Os}{\mathcal{O}}
\newcommand{\Otil}{\widetilde{\mathcal{O}}}
\newcommand{\Otilst}{\widetilde{\mathcal{O}}^{\star}}
\newcommand{\oco}{\textsc{Oco}}

\newcommand{\clip}{\mathrm{clip}}
\newcommand{\Gbar}{\bar{G}}
\renewcommand{\uex}{u^{\mathrm{ex}}}
\newcommand{\xnat}{x^{\mathrm{nat}}}
\newcommand{\xnathat}{\hat{x}^{\mathrm{nat}}}
\newcommand{\unathat}{\hat{u}^{\mathrm{nat}}}
\newcommand{\Gstt}[1][t]{G_{\star,#1}}
\newcommand{\unat}{u^{\mathrm{nat}}}
\newcommand{\vnat}{v^{\mathrm{nat}}}
\newcommand{\uhat}{\hat{u}}
\newcommand{\calG}{\mathcal{G}}
\newcommand{\Pidrc}{\Pi_{\mathrm{drc}}}
\newcommand{\Pidac}{\Pi_{\mathrm{dac}}}
\newcommand{\Pifeed}{\Pi_{\mathrm{feed}}}
\newcommand{\Pistate}{\Pi_{\mathrm{state}}}
\renewcommand{\radM}{R_M}
\newcommand{\calA}{\mathcal{A}}
\newcommand{\radnat}{R_{\mathrm{nat}}}
\section{Problem Setting}\label{sec:setting}
We study control of a linear time-varying (LTV) system \Cref{eq:system} with state $x_t \in \R^{d_x}$, control input $u_t \in \R^{d_u}$ chosen by the learner, and the external disturbance $w_t \in \R^{d_x}$ chosen by Nature.  The system is characterized by time-varying matrices $A_t \in \R^{d_x \times d_x}, B_t \in \R^{d_x \times d_u}$. For simplicity, the initial state is $x_1=0$.  At each time $t$,  \emph{oblivious}\footnote{An oblivious adversary chooses the matrices, costs and perturbations  prior to the control trajectory.} 
adversary picks the system matrices $(A_t, B_t)$, disturbances $w_t$ and cost functions $c_t : \R^{d_x} \times \R^{d_u} \to \R$. The dynamics $(A_t, B_t)$ are \emph{unknown} to the learner: one observes only the next state $x_{t+1}$ and current cost $c_t(\cdot, \cdot)$ after playing control $u_t$. 

\textbf{Adaptive Regret.}
The goal of the learner is to minimize regret w.r.t. a policy class $\Pi$,  i.e. the difference between the cumulative cost of the learner and the best policy $\pi^{\star} \in \Pi$ in hindsight. Formally, the regret of an algorithm $\calA$ with control inputs $u_{1:T}$ and corresponding states $x_{1:T}$, over an interval $I = [r, s] \subseteq [T]$, is defined as
\begin{equation}\label{eq:regret_def}
    \Regret_I(\calA; \Pi) = \sum_{t \in I} c_t(x_t, u_t) - \inf_{\pi \in \Pi} \sum_{t\in I} c_t(x_t^{\pi}, u_t^{\pi}) ~.
\end{equation}
Here $u_t^{\pi}, x_t^{\pi}$ indicate the control input and the corresponding state when following policy $\pi$. For a randomized algorithm $\calA$, we consider the expected regret.  In this work, we focus on designing control algorithms that minimize \emph{adaptive} regret, i.e. guarantee a low regret relative to the best-in-hindsight policy $\pi^{\star}_I \in \Pi$  on \emph{any} interval $I \subseteq [T]$. This performance metric of adaptive regret is more suitable for control over LTV dynamics given its agility to compete against different local optimal policies $\pi^{\star}_I \in \Pi$ at different times \citep{gradu2020adaptive}. To illustrate this point, we describe the implications of standard vs. adaptive regret for $k$-switching LQR.

\begin{example}[$k$-switching LQR.]\label{ex:k_switch_lqr} Consider the problem of $k$-switching LQR in which the system evolves according to the fixed $(A_j, B_j)$ over each time interval $I_j = \left[\ceil{(j-1) \cdot T/k}, \floor{j \cdot T/k}\right]$ for $j \in [1, k]$. An adaptive regret guarantee ensures good performance against $\pi^\star_j = \argmin_{\pi\in\Pi} \sum_{t\in I_j} c_t(x_t^\pi, u_t^\pi)$ on every interval $I_j$, in contrast to standard regret which only ensures good performance against a single, joint comparator $\pi^\star = \argmin_{\pi\in\Pi} \sum_{t=1}^T c_t(x_t^\pi, u_t^\pi)$. Clearly over every interval $I_j$ the policy $\pi^\star_j$ is a suitable comparator while $\pi^\star$ is not.
\end{example}



\textbf{Key objects.} A central object in our study is the sequence of Nature's x's $\xnat_{1:T}$  that arises from playing zero control input $u_t = 0$ at each $t \in [T]$, i.e. $\xnat_{t+1} = A_t \xnat_t + w_t$ \cite{simchowitz2020improper}. \arxivfinal{This object allows us to split any state into a component independent of the algorithm's actions and a component that is the direct effect of the chosen actions. To capture this intuition in equation form, we} define the following operators for all $t$,
\begin{align*}
    \Phi_t^{[0]} = \Id, \quad \forall h \in [1, t), \, \Phi_t^{[h]} = \prod_{k=t}^{t-h+1} A_k, \quad \forall i \in [0, t), \, G_t^{[i]} = \Phi_t^{[i]} B_{t-i},
\end{align*}
where the matrix product $\prod_{s}^r$ with $s\geq r$ is taken in the indicated order $k = s, \dots, r$. The following identities give an alternative representation for the Nature's x's $\xnat_t$ and state $x_t$ with control input $u_t$ in terms of the  \emph{Markov operator} at time $t$, $G_t = [G_t^{[i]}]_{i \geq 0}$:
\begin{equation*}
    \xnat_{t+1} = \sum_{i=0}^{t-1} \Phi_t^{[i]} w_{t-i}, \quad x_{t+1} = \xnat_{t+1} + \sum_{i=0}^{t-1} G_t^{[i]} u_{t-i} ~.
\end{equation*}
\nipsfinal{These operators and the alternative representation capture the dynamics by decoupling the disturbance and the control action effects. Observe that the operators $[\Phi_t^{[i]}]_{i \in [0, t)}$ capture the contribution of the perturbations on the state $x_{t+1}$ and the Markov operators $[G_t^{[i]}]_{i \in [0, t)}$ that of the controls.}{}

\textbf{Assumptions.} We make the three basic assumptions:  we require from (i) the disturbances to not blow up the system with no control input, (ii) the system to have decaying effect over time, and (iii) the costs to be well-behaved and admit efficient optimization. Formally, these assumptions are: 
\begin{assumption}\label{asm:radnat} For all $t \in [T]$, assume $\|\xnat_t\| \le \radnat$.
\end{assumption}
\begin{assumption} \label{asm:Gdecay} Assume there exist $R_G \ge 1$ and $\rho \in (0,1)$ s.t. for any $h \geq 0$ and for all $t \in [T]$
\begin{align*}
    \sum_{i \geq h} \|G_t^{[i]}\|_{\op} \leq R_G \cdot \rho^{h} := \psi(h) ~.
\end{align*}
\end{assumption}
\begin{assumption} \label{asm:cost} Assume the costs $c_t : \R^{d_x} \times \R^{d_u} \to \R$ are general convex functions that satisfy the conditions $0 \le c_t(x,u) \le L\max\{1,\|x\|^2 +\|u\|^2\}$, and $\|\nabla c_t(x,u)\| \le L \max\{1,\|x\| + \|u\|\}$ for some constant $L>0$, where $\nabla$ denotes any subgradient \cite{boyd2004convex}. 
\end{assumption}
The conditions in \Cref{asm:cost} allow for functions whose values and gradient grow as quickly as quadratics (e.g. the costs in LQR) , and the $\max\{1, \cdot\}$ term ensures the inclusion of standard bounded and Lipschitz functions as well. \Cref{asm:radnat,asm:Gdecay} arise from the assumption our LTV system is \emph{open-loop stable}; \Cref{sec:app:stabilizable} extends to the case where a nominal stabilizing controller is known, as in prior work \cite{agarwal2019online,simchowitz2020improper}. While these two assumptions may seem unnatural at first, they can be derived from the basic conditions of disturbance norm bound and sequential stability. 
\begin{lemma} Suppose that there exist $C_1 \ge 1, \rho_1 \in (0,1)$ such that $\|\Phi_{t}^{[h]}\|_{\op} \le C_1 \rho_1^{h}$ for any $h \geq 0$ and all $t \in [T]$, and suppose that $\max_t \|w_t\| \le R_w$. Then, \Cref{asm:radnat} holds with $\radnat = \frac{C_1}{1-\rho_1} R_w$, and \Cref{asm:Gdecay} holds with $\rho = \rho_1$ and $R_G = \max\{1,\max_t\|B_t\|_{\op} \cdot \frac{C_1}{1-\rho_1} \}$.
\end{lemma}
Note that \Cref{asm:Gdecay} implies that $\|G_t\|_{\ell_1, \op} = \sum_{i \geq 0} \| G_t^{[i]} \|_{\op} \leq R_G$. It also suggests that for a sufficiently large $h$ the effect of iterations before $t-h$ are negligible at round $t$. This prompts introducing a truncated Markov operator: denote $\Gbar^h_{t} = [G_t^{[i]}]_{i < h}$ to be the $h$-truncation of the true Markov operator $G_t$. It follows that their difference is $\| \Gbar_t^h - G_t \|_{\ell_1, \op} = \sum_{i \geq h} \| G_t^{[i]} \|_{\op} \leq \psi(h)$ negligible in operator norm for a sufficiently large $h$. Define the bounded set of $h$-truncated Markov operators to be $\calG(h, R_G) = \{ G = [G^{[i]}]_{0 \le i < h} : \|G\|_{\ell_1, \op} \leq R_G \}$ with $\Gbar_t^h \in \calG(h, R_G)$ for all $t$.
\subsection{Benchmarks and Policy Classes}
The performance of an algorithm, measured by \Cref{eq:regret_def}, directly depends on the policy class $\Pi$ that is chosen as a benchmark to compete against. In this work, we consider the following three policy classes: \drc, \dac, and linear feedback. \drc{} parameterizes control inputs in terms of Nature's x's $\xnat_t$, \dac{} does so in terms of the disturbances $w_t$ and linear feedback in terms of the states $x_t$. We express all three in terms of a length-$m$ parameter $M = [M^{[i]}]_{i < m}$ in a bounded ball $\calM(m, R_M)$:
\begin{equation*}
    \calM(m,R_M) = \{(M^{[0]},\dots,M^{[m-1]}): \textstyle \sum_{i=0}^{m-1}\|M^{[i]}\|_{\op} \le \radM\} ~.
\end{equation*}
\begin{definition}[\drc{} policy class]\label{defn:policy_drc}
A \drc{} control policy $\pi_{\mathrm{drc}}^M$ of length $m$ is given by $u_t^M = \sum_{i=0}^{m-1} M^{[i]} \xnat_{t-i}$ where $M = [M^{[i]}]_{i < m}$ is the parameter of the policy. Define the bounded \drc{} policy class as $\Pidrc(m, R_M) = \{ \pi_{\mathrm{drc}}^M : M \in \calM(m, R_M) \}$.
\end{definition}
\begin{definition}[\dac{} policy class]\label{defn:policy_dac}
A \dac{} control policy $\pi_{\mathrm{dac}}^M$ of length $m$ is given by $u_t^M = \sum_{i=0}^{m-1} M^{[i]} w_{t-i}$ where $M = [M^{[i]}]_{i < m}$ is the parameter of the policy. Define the bounded \dac{} policy class as $\Pidac(m, R_M) = \{ \pi_{\mathrm{dac}}^M : M \in \calM(m, R_M) \}$.
\end{definition}
\begin{definition}[Feedback policy class]\label{defn:policy_feedback}
A feedback control policy $\pi_{\mathrm{feed}}^M$ of length $m$ is given by $u_t^M = \sum_{i=0}^{m-1} M^{[i]} x_{t-i}$ where $M = [M^{[i]}]_{i < m}$ is the parameter of the policy. Define the bounded feedback policy class as $\Pifeed(m, R_M) = \{ \pi_{\mathrm{feed}}^M : M \in \calM(m, R_M) \}$. In the special case of memory $m=1$, denote the \emph{state feedback} policy class as $\Pistate = \Pifeed(m=1)$.
\end{definition}
\textbf{Convexity. }Both the \drc{} and \dac{} policy classes are \emph{convex parametrizations}: a policy $\pi \in \Pidrc \cup \Pidac$ outputs controls $u_t$ that are linear in the \emph{policy-independent} sequences $\xnat_{1:T}$ and $w_{1:T}$, and thus the mapping from parameter $M$ to resulting states and inputs (resp. costs) is affine  (resp. convex).  Hence, we refer to these as the \emph{convex classes}. In contrast, feedback policies select inputs based on \emph{policy-dependent} states, and are therefore non-convex \cite{fazel2018global}.

We drop the arguments $m, R_M$ when they are clear from the context. The state feedback policies $\Pistate$ encompass the $\Htwo$ and $\Hinf$ optimal control laws under full observation. For LTI systems, \drc{} and \dac{} are equivalent \cite{youla1976modern, wang2019system} and approximate all linear feedback policies to arbitrarily high precision \cite{agarwal2019online, simchowitz2020improper}. However, we show that these relationships between the classes break down for LTV systems: there exist scenarios where any one of the three classes strictly outperforms the other two.

\begin{theorem}[Informal]\label{thm:informal_separation}
For each class $\Pi$ in $ \{ \Pidrc, \Pidac, \Pifeed \}$ there exists a sequence of well-behaved $(A_t,B_t,w_t,c_t)$ such that a policy $\pi^\star \in \Pi$ suffers $0$ cumulative cost, but each of the other two classes $\Pi' \in \{ \Pidrc, \Pidac, \Pifeed \} \setminus \Pi$ suffers $\Omega(T)$ cost on all their constituent policies $\pi \in \Pi'$.
\end{theorem}
The formal theorem that includes the definition of a well-behaved instance sequence and the final statement dependence on $m, \radM$ along with its proof can be found in \Cref{sec:app:separation}.

\paragraph{Notation.} The norm $\| \cdot \|$ refers to Euclidean norm unless otherwise stated, $[n]$ is used as a shorthand for $[1, n]$, $T$ is used as a subscript shorthand for $[T]$. The asymptotic notation $\Os(\cdot), \Omega(\cdot)$ suppress all terms independent of $T$, $\Otil(\cdot)$ additionally suppresses terms logarithmic in $T$. We define $\Otilst(\cdot)$ to suppress absolute constants, polynomials in $\radnat, R_G, \radM$ and logarithms in $T$.

\newcommand{\Gseq}{\mathbf{G}}
\newcommand{\predalg}{\mathcal{A}_{\mathrm{pred}}}
\newcommand{\controlalg}{\mathcal{A}_{\mathrm{ctrl}}}
\newcommand{\adapred}{\textsc{Ada}\text{-}\textsc{Pred}}
\newcommand{\adacontrol}{\textsc{Ada}\text{-}\textsc{Ctrl}}
\newcommand{\calS}{\mathcal{S}}
\newcommand{\drcogd}{\textsc{Drc}\text{-}\textsc{Ogd}}

\section{Online Control over Convex Policies}\label{sec:dac_drc_body}

This section considers online control of unknown LTV systems so as to compete with the convex \drc{} and \dac{} policy classes. The fundamental quantity which appears throughout our results is the \emph{system variability}, which measures the variation of the time-varying Markov operators $G_t$ over intervals $I$.
\begin{definition}\label{def:variability}
Define the system variability of an LTV dynamical system with Markov operators $\Gseq = G_{1:T}$ over a contiguous interval $I \subseteq [T]$ to be
\begin{equation*}
    \var_I(\Gseq) = \min_{G} \frac{1}{|I|}\sum_{t \in I} \| G - G_t \|_{\ell_2, F}^2 = \frac{1}{|I|}\sum_{t\in I} \|G_I - G_t\|_{\ell_2, F}^2,
\end{equation*}
where $\| \cdot \|_{\ell_2, F}$ indicates the $\ell_2$ norm of the fully vectorized operator and $G_I = |I|^{-1} \sum_{t\in I} G_t$ is the empirical average of the operators that correspond to $I$. Recall that $\var_T(\Gseq)$ corresponds to $I = [T]$.
\end{definition}
Our results in this section for both upper and lower bounds focus on expected regret: high probability results are possible as well with more technical effort using standard techniques.

\subsection{A Linear Regret Lower Bound}
Our first contribution is a negative one: that the regret against the class of either \dac{} or \drc{} policies cannot scale sublinearly in the time horizon. Informally, our result shows that the regret against these classes scales as $T \sigma$, where $\sigma^2$ is the system variability. 

 More precisely, for any $\sigma^2 \in (0,1/8]$, we construct a distribution $\mathcal{D}_{\sigma}$ over sequences $(A_t,B_t,c_t,w_t)$, formally specified in \Cref{sec:app:drc_lower}. Here, we list the essential properties of  $\mathcal{D}_{\sigma}$: \emph{(i)} $A_t \equiv 0$, \emph{(ii)} $c_t \equiv c$ is a fixed cost satisfying \Cref{asm:cost} with $L \le 4$, \emph{(iii)} the matrices $(B_t)$ are i.i.d., with $\|B_t\|_{\op} \le 2$ almost surely, and $\E[\|B_t - \E[B_t]\|_{\fro}^2] = \sigma^2$, and \emph{(iv)} $\|w_t\| \le 4$ for all $t$. These conditions imply that \Cref{asm:Gdecay,asm:radnat} hold for $R_G = 2,\rho =0,\radnat = 4$. The condition $A_t \equiv 0$ implies that $\xnat_t = w_t$ for all $t$, so the classes \drc{} and \dac{} are equivalent and the lower bound holds over both. Moreover, by Jensen's inequality, this construction ensures that
\begin{align*}
\E[\var_I(\Gseq)] &= |I|^{-1} \E[\min_{G} \textstyle \sum_{t \in I} \| G - G_t \|_{\ell_2, F}^2 ] \\
&= |I|^{-1} \E[\min_{B} \textstyle \sum_{t \in I} \| B - B_t \|_{\fro}^2 ] \le \E[\|B_t - \E[B_t]\|_{\fro}^2] = \sigma^2. 
\end{align*} 
In particular,  $\E[\var_T(\Gseq)] \le \sigma^2$. For the described construction, we show the following lower bound:
 \begin{theorem}\label{thm:main_lb} Let $C$ be a universal, positive constant. For any $\sigma \in (0,1/8]$ and any online control algorithm $\calA$, there exists a \drc{} policy $\pi^{\star} \in \Pidrc(1,1)$ s.t. expected regret incurred by $\calA$ under the distribution $\mathcal{D}_{\sigma}$ and cost $c(x,u)$ is at least
\begin{align*}
\Exp_{\mathcal{D}_{\sigma},\calA}  [\Regret_T(\calA; \{\pi^\star\})] \ge C \cdot T \sigma \ge C\cdot T\cdot \sqrt{\E[\var_T(\Gseq)]}, \end{align*}
\end{theorem}
A full construction and proof of \Cref{thm:main_lb} is given in  \Cref{sec:app:drc_lower}. In particular, for $\sigma = 1/8$, we find that no algorithm can attain less than $\Omega(T)$ expected regret; a stark distinction from either unknown LTI \cite{simchowitz2020improper, hazan2019nonstochastic} or known LTV \cite{gradu2020adaptive} systems.

\subsection{Estimation of Time-Varying Vector Sequences}\label{sec:sys_est_main}
To devise an algorithmic upper bound that complements the result in \Cref{{thm:main_lb}}, we first consider the setting of online prediction under a partial information model. This setting captures the \emph{system identification} phase of LTV system control and is used to derive the final control guarantees.
Formally, consider the following repeated game between a learner and an \emph{oblivious} adversary: at each round $t \in [T]$, the adversary picks a target vector $\matzst_t \in \calK$ from a convex decision set $\calK$ contained in a $0$-centered ball of radius $R_z$; simultaneously, the learner selects an estimate $\matzhat_t \in \calK$ and suffers quadratic loss $\ell_t(\matzhat_t) = \| \matzhat_t - \matzst_t \|^2$. The only feedback the learner has access to is via the following noisy and costly oracle.

\begin{oracle}[Noisy Costly Oracle]\label{orac:noisy_main} At each time $t \in [T]$, the learner selects a decision $b_t \in \{0,1\}$ indicating whether a query is sent to the oracle. If $b_t = 1$, the learner receives an unbiased estimate $\matztil_t$ as response such that $\| \matztil_t \| \leq \tilde{R}_z$ and $\E[\matztil_t \mid \calF_t, b_t=1] = \matzst_t$. The filtration $\calF_t$ is the sigma algebra generated by $\matztil_{1:t-1}, b_{1:t-1}$ and the choices of the oblivious adversary $\matz_{1:T}^{\star}$. A completed query results in a unit cost $\lambda > 0$ for the learner.
\end{oracle}

The performance metric of an online prediction algorithm $\predalg$ is expected quadratic loss regret along with the extra cumulative oracle query cost. It is defined over each interval $I = [r, s] \subseteq [T]$ as

\begin{equation}\label{eq:est_regret_main}
    \mathrm{Regret}_I(\predalg; \lambda) = \Exp_{\calF_{1:T}}\left[ \sum_{t \in I} \ell_t(\matzhat_t) - \min_{\matz \in \calK} \sum_{t \in I} \ell_t(\matz) + \lambda \sum_{t \in I} b_t \right] ~.
\end{equation}


\begin{algorithm}
\caption{Adaptive Estimation Algorithm (\adapred)}\label{alg:est_combined}
\begin{algorithmic}[1]
    \State \textbf{Input:} parameter $p$, decision set $\calK$
    \State \textbf{Initialize:}  $\matzhat_1^{(1)} \in \calK$, working dictionary $\calS_1 = \{(1:\matzhat_1^{(1)})\}$, $q_1^{(1)} = 1$, parameter $\alpha = \frac{p}{(R_z+\tilde{R}_z)^2}$
    \For{$t=1, \ldots, T$}
    \State{}{Play} iterate $\matzhat_t = \sum_{(i,\matzhat_t^{(i)}) \in \calS_t} q_t^{(i)} \matzhat_t^{(i)}$
    \State{}{Draw/Receive} $b_t \sim \Bernoulli(p)$
    \If{$b_{t} = 1$}
    \State{}{Request} estimate $\matztil_t$ from \Cref{orac:noisy_main}
    \State{}{Let} $\tilde{\ell}_t(\matz) = \frac{1}{2p}||\matz - \matztil_t||^2$ and $\tilde{\nabla}_t = \frac{1}{p}(\matz_t -\matztil)$
    \Else
    \State{}{Let} $\matztil_t \leftarrow \emptyset$ and $\tilde{\ell}_t(\matz) = 0$ and $\tilde{\nabla}_t = 0$
    \EndIf
    \State{}{Update} predictions $\matzhat^{(i)}_{t+1} \leftarrow \Proj_{\calK} (\matzhat^{(i)}_{t} - \eta_t^{(i)} \tilde{\nabla}_t)$ for all $(i,\matzhat^{(i)}_t) \in S_{t}$
    \State{}{Form} new dictionary $\tilde{\calS}_{t+1} = (i,\matzhat^{(i)}_{t+1})_{i \in \mathrm{keys}(\calS_t)}$
    \State{}{Construct} proxy new weights $\bar{q}_{t+1}^{(i)} =  \tfrac{t}{t+1} \cdot \tfrac{q_t^{(i)} e^{-\alpha \tilde{\ell}_t(\matzhat_t^{(i)})}}{\sum_{j\in \mathrm{keys}(\calS_t)} q_t^{(j)} e^{-\alpha \tilde{\ell}_t(\matzhat_t^{(j)})}}$ for all $i\in \mathrm{keys}(\calS_t)$ \label{line:weight_update_main}
    \State{}{Add} new instance $\tilde{\calS}_{t+1} \leftarrow \tilde{\calS}_{t+1} \cup (t+1, \matzhat_{t+1}^{(t+1)})$ for arbitrary $\matzhat_{t+1}^{(t+1)} \in \calK$ with $\bar{q}_{t+1}^{(t+1)} = \frac{1}{t+1}$
    \State{}{Prune}  $\tilde{\calS}_{t+1}$ to form $\calS_{t+1}$ (see \Cref{sec:app:est_adaptive})
    \State{}{Normalize} $q_{t+1}^{(i)} = \frac{\bar{q}_{t+1}^{(i)}}{\sum_{j\in \mathrm{keys}(\calS_{t+1})}\bar{q}_{t+1}^{(j)}}$
    \EndFor
\end{algorithmic}
\end{algorithm}

To attain \emph{adaptive} regret, i.e. bound \Cref{eq:est_regret_main} for each interval $I$, we propose \Cref{alg:est_combined} constructed as follows. First, suppose we wanted non-adaptive (i.e. just $I = [T]$) guarantees. In this special case, we propose to sample $b_t \sim \mathrm{Bernoulli}(p)$ for an appropriate parameter $p \in (0,1)$, and perform a gradient descent update on the importance-weighted square loss $\tilde{\ell}_t(\matz) = \frac{1}{2p}\I\{b_t = 1\}\|\matz - \matztil_t\|^2$. 
To extend this method to enjoy adaptive regret guarantees, we adopt the approach of \cite{hazan2009efficient}: the core idea in this approach is to initiate an instance of the base method at each round $t$ and use a weighted average of the instance predictions as the final prediction (Line 4). The instance weights are multiplicatively updated according to their performance (Line 13). To ensure computational efficiency, the algorithm only updates instances from a working dictionary $\calS_t$ (Line 11). These dictionaries are pruned each round (Line 15) such that $|\calS_t| = O(\log T)$ (see \Cref{sec:app:est_adaptive} for details).
\begin{theorem}\label{thm:adaptive_est_main} Given access to queries from \Cref{orac:noisy_main} and with stepsizes $\eta_t^{(i)} = \frac{1}{t-i+1}$, \Cref{alg:est_combined} enjoys the following adaptive regret guarantee: for all $I = [r, s] \subseteq [T]$,
\begin{equation}\label{eq:est_adaptive_bound_main}\Regret_I(\adapred; \lambda) \leq \frac{2 (R_z + \tilde{R}_z)^2 (1 + \log{s} \cdot \log |I| )}{p} + \lambda p |I| ~.
\end{equation}
\end{theorem}
When $I = [T]$, the optimal choice of parameter $p = \log T/\sqrt{\lambda T}$ yields regret scaling roughly as $\sqrt{\lambda T \log^2 T}$. Unfortunatelly, this gives regret scaling as $\sqrt{T}$ for all interval sizes: to attain $\sqrt{|I|}$ regret on interval $I$, the optimal choice of $p$ would yield $\sim T/\sqrt{|I|}$ regret on $[T]$, which is considerably worse for small $|I|$.  One may ask if there exists a \emph{strongly adaptive} algorithm which adapts $p$ as well, so as to enjoy regret polynomial in $|I|$ for all intervals $I$ \emph{simultaneously} \cite{daniely15strongly}. The following result shows this is not possible:


\begin{theorem}[Informal]\label{thm:no_strong_adap}
For all $\gamma>0$ and $\lambda > 0$, there exists no online algorithm $\calA$ with feedback access to \Cref{orac:noisy_main} that enjoys strongly adaptive regret of $\Regret_I(\calA; \lambda) = \tilde{O}(|I|^{1-\gamma})$.
\end{theorem}
Hence, in a sense, \Cref{alg:est_combined} is as adaptive as one could hope for: it ensures a regret bound for all intervals $I$, but not a strongly adaptive one. The lower bound construction, formal statement, and proof of \Cref{thm:no_strong_adap} are given in \Cref{sec:app:lower_strongly}.

\subsection{Adaptive Regret for Control of Unknown Time-Varying Dynamics}

We now apply our adaptive estimation algorithm (\Cref{alg:est_combined}) to the online control problem. Our proposed algorithm, \Cref{alg:control_unknown_main}, takes in two sub-routines: a prediction algorithm $\predalg$ which enjoys low prediction regret in the sense of the previous section, and a control algorithm $\controlalg$ which has low regret for control of \emph{known} systems. Our master algorithm trades off between the two methods in epochs $\tau = 1,2,\dots$ of length $h$: each epoch corresponds to one step of $\predalg$ indexed by $[\tau]$.

\begin{algorithm}
\caption{DRC-OGD with Adaptive Exploration (\adacontrol)}\label{alg:control_unknown_main}
\begin{algorithmic}[1]
    \State \textbf{Input:} parameters $p, h$, prediction algorithm $\predalg \leftarrow \adapred(p, \hat{G}_0, \mathcal{G}(h, R_G))$, control algorithm $\controlalg \leftarrow \drcogd(m, R_\calM)$
    \For{$\tau = 1, \ldots, T/h$} 
    \Comment{let $t_{\tau} = (\tau-1)h + 1$}
    \State Set $\hat{G}_{t_\tau},\hat{G}_{t_\tau+1}, \ldots, \hat{G}_{t_{\tau}+h-1} $ equal to $\tau$-th iterate $\hat{G}_{[\tau]}$ from $\predalg$ \label{line:setting_G_hat1}
    \State Draw $b_{[\tau]} \sim \text{Bernoulli}(p)$
    \For{$t = t_{\tau}, \ldots, t_{\tau} + h - 1$}
    \If{$b_{[\tau]} = 1$}
    \State Play control $u_{t} \sim \{\pm 1\}^{d_u}$ \label{line:random_control1}
    \Else 
    \State Play control $u_{t}$ according to the $t$-th input chosen by $\controlalg$ \label{line:random_control2}
    \EndIf
    \State Suffer cost $c_t(x_{t}, u_{t})$ , observe new state $x_{t+1}$ 
    \State Extract $\hat{x}_{t+1}^{\mathrm{nat}} = \Proj_{\mathbb{B}_{R_{\mathrm{nat}}}}\left(x_{t+1} - \sum_{i=0}^{h-1} \hat{G}_t^{[i]} u_{t-i}\right)$
    \State Feed cost, Markov operator and Nature's x estimates $(c_t, \hat{G}_t, \hat{x}^{\mathrm{nat}}_{t+1})$ to $\controlalg$.
    \EndFor
    \If{$b_{[\tau]} = 1$}
    \State Feed  $(b_{[\tau]}, \tilde{G}_{[\tau]})$ to $\predalg$, where  $\tilde{G}_{[\tau]}^{[i]} = x_{t_\tau+h} u_{t_\tau + h - i}^\top,\, i = 0,1,\dots,h-1$.
    \Else
    \State Feed  $(b_{[\tau]}, \tilde{G}_{[\tau]})$ to $\predalg$, where $\tilde{G}_{[\tau]} \leftarrow \emptyset$.
    \EndIf
    \EndFor
\end{algorithmic}
\end{algorithm}

At each epoch, the algorithm receives Markov operator estimates from $\predalg$ (Line 3) and makes a binary decision $b_{[\tau]} \sim \mathrm{Bernoulli}(p)$. If $b_{[\tau]} = 1$, then it explores using i.i.d. Rademacher inputs (Line 7), and sends the resulting estimator to $\predalg$ (Line 14). This corresponds to one query from \Cref{orac:noisy_main}. Otherwise, it selects inputs in line with $\controlalg$ (Line 9), and does not give a query to $\predalg$ (Line 16). Regardless of exploration decision, the algorithm feeds costs, current estimates of the Markov operator and Nature's x's based on the Markov operator estimates to $\controlalg$ (Lines 10-12), which it uses to select inputs and update its internal parameter.

The prediction algorithm $\predalg$ is taken to be $\adapred$ with the decision set $\calK = \mathcal{G}(h, R_G)$: the projection operation onto it and the ball $\mathbb{B}_{\radnat}$ is done by clipping when the norm of the argument exceeds the indicated bound. The control algorithm $\controlalg$ is taken to be $\drcogd$ \cite{simchowitz2020improper} for \emph{known} systems. \arxivfinal{The core technique behind $\drcogd$ is running online gradient descent on the \drc{} parameterization (\Cref{defn:policy_drc}). In \Cref{sec:app:drc-ogd} we spell out the algorithm and extend previous analyses to both LTV systems and adaptive regret guarantees. The final result guarantees low \emph{adaptive} regret as long as the system variability is sublinear.} 


\begin{theorem}\label{thm:control_unknown_regret}
For $h = \dfrac{\log{T}}{\log{\rho^{-1}}}$, $p=T^{-1/3}$ and $m\leq \sqrt{T}$, on any contiguous interval $I \subseteq [T]$, \Cref{alg:control_unknown_main} enjoys the following adaptive regret guarantee:
\begin{align*}\label{eq:control_reg_bound_main} \E\left[\Regret_I(\adacontrol); \Pidrc(m, R_M)\right] \leq \widetilde{\mathcal{O}}^{\star}\left(Lm\left(|I| \sqrt{ \Exp[\var_{I}(\Gseq)]} + d_u T^{2/3}\right)  \right) 
\end{align*}
\end{theorem}
\begin{proof}[Proof Sketch]
The analysis proceeds by reducing the regret incurred to that over a known system, accounting for: 1) the additional exploration penalty $(O(p|I|))$, 2) the system misspecification induced error $(\sim \sum_{t\in I} \|\hat{G}_t - \bar{G}_t^h\|_{\ell_1,\mathrm{op}})$, and 3) truncation errors ($\sim \psi(h) |I|$). Via straightforward computations, the system misspecification error can be expressed in terms of the result in \Cref{thm:adaptive_est_main}, ultimately leading to an error contribution $\sim |I|\sqrt{\Exp[\var_I(\Gseq)]} + p^{-1/2} |I|^{1/2}$. The analysis is finalized by noting that the chosen $p$ ideally balances $p|I|$ and $p^{-1/2}|I|^{1/2}$, and that the chosen $h$ ensures that the truncation error is negligible. The full proof can be found in \Cref{sec:ctrl_red}.
\end{proof}
The adaptive regret bound in \Cref{thm:control_unknown_regret} has two notable terms. Note that the first term $|I|\sqrt{\Exp[\var_I(\Gseq)]}$ for $I=[T]$ matches the regret lower bound in \Cref{thm:main_lb}. Furthermore, our algorithm is adaptive in this term for all intervals $I$. On the other hand, for unknown LTI systems with $\var_I(\Gseq) = 0$, the algorithm recovers the state-of-the-art bound of $T^{2/3}$ \cite{hazan2019nonstochastic}. However, the $T^{2/3}$ term is not adaptive to the intervals $I$ consistent with the lower bound against strongly adaptive algorithms in \Cref{thm:no_strong_adap}.

\newcommand{\maxsat}{\textsc{Max}\text{-}\textsc{3Sat}}
\section{Online Control over State Feedback}\label{sec:linear_main}
Given the impossibility of sublinear regret against \drc/\dac{}  without further restrictions on system variability, this section studies whether sublinear regret is possible against the class of linear feedback policies. For simplicity, we focus on the \emph{state feedback} policies $u_t = Kx_t$, that is, linear feedback policies with memory $m = 1$ (\Cref{defn:policy_feedback}). We note that state feedback policies were the class which motivated the relaxation to \dac{} policies in the first study of nonstochastic control \cite{agarwal2019online}. 

We present two results, rather qualitative in nature. First, we show that obtaining \emph{sublinear regret is, in the most literal sense, possible}. The following result considers regret relative to a class $\calK$ of static feedback controllers which satisfy the restrictive assumption that each $K \in \calK$ stabilizes the time varying dynamics $(A_t,B_t)$; see \Cref{sec:sublinear_reg_state_feedback} for the formal algorithm, assumptions, and guarantees. We measure the regret against this class $\calK$:
\begin{align*}
\Regret_T(\calK) := \sum_{t=1}^T c_t(x_t,u_t) - \inf_{K \in \calK}\sum_{t=1}^Tc_t(x_t^K, u_t^K),
\end{align*}
where $(x_t^K, u_t^K)$ are the iterates arising under the control law $u_t = K x_t$. 
\begin{theorem}[Sublinear regret against state-feedback]\label{thm:K_regret_body} Under a suitable stabilization assumption, there exists a computationally inefficient control algorithm which attains sublinear expected regret:
\begin{align*}
\E[\Regret_T(\calK) ] \le e^{\Omega(\dimx \dimu/2)} \cdot T^{1-\frac{1}{2(\dimx \dimu +3)}}.
\end{align*}
\end{theorem}
Above, $\Omega(\cdot)$ suppresses a universal constant and exponent base, both of which are made explicit in a formal theorem statement in \Cref{sec:sublinear_reg_state_feedback}. The bound follows by running the \textsc{Exp3} bandit algorithm on a discretization of the set $\mathcal{K}$ (high probability regret can be obtained by instead using  \textsc{Exp3.P}  \cite{bubeck2012regret}). The guarantee in \Cref{thm:K_regret_body} is neither practical nor sharp; its sole purpose is to confirm the possibility of sublinear regret. Due to the bandit reduction and exponential size of the cover of $\calK \subset \R^{\dimu \times \dimx}$, the algorithm is computationally inefficient and suffers a \emph{nonparametric} rate of regret \citep{rakhlin2014online}:  $\epsilon$-regret requires $T = \epsilon^{-\Omega(\mathrm{dimension})}$. 

One may wonder if one can do much better than this naive bandit reduction. For example, is there structure that can be leveraged? For LTV systems, we show that there is strong evidence to suggest that, at least from a computational standpoint, attaining polynomial regret (e.g. $T^{1-\alpha}$ for $\alpha > 0$ independent of dimension) is computationally prohibitive. 
\begin{theorem}\label{thm:body_hard_ness_lb} There exists a reduction from \maxsat{} on $m$-clauses and $n$-literals to the problem of finding a state-feedback controler $K$ which is within a small constant factor of optimal for the cost $\textstyle \sum_{t=1}^T c_t(x^K_t,x^K_t)$ on a sequence of sequentially stable LTV systems and convex costs $(A_t,B_t,c_t)$ with no disturbance ($w_t \equiv 0$), with state dimension $n+1$, input dimension $2$, and horizon $T = \Theta(mn)$. Therefore, unless $\Ptime = \NPtime$, the latter cannot be solved in time polynomial in $n$ \cite{haastad2001some}. 
\end{theorem}
A more precise statement, construction, and proof are given in \Cref{sec:NP_hardness}. \Cref{thm:body_hard_ness_lb} demonstrates that solving the \emph{offline optimization} problem over state feedback controllers $K$ to within constant precision is $\NPtime$-Hard. In particular, this means that any sublinear regret algorithm which is \emph{proper and convergent}, in the sense that $u_t = K_t x_t$ for some sequence $K_t$ converges to a limit as $T \to \infty$, must be computationally inefficient. This is true \emph{even if the costs and dynamics are known in advance}. Our result suggests it is computationally hard to obtain sublinear regret, but it does not rigorously imply it. For example, there may be more clever convex relaxations (other than \drc{} and \dac{}, which provably cannot work) that yield efficient and sublinear regret. Secondly, this lower bound does not rule out the possibility of an computationally inefficient algorithm which nevertheless attains polynomial regret. 


\section{Discussion and Future Work}\label{sec:conclusions}

This paper provided guarantees for and studied the limitations of sublinear additive regret in online control of an unknown, linear time-varying (LTV) dynamical system. 

Our setting was motivated by the fact that the first-order Taylor approximation (Jacobian linearization) of smooth, nonlinear systems about any smooth trajectory is LTV. One would therefore hope that low-regret guarantees against LTV systems may imply convergence to first-order stationary points of general nonlinear control objectives \cite{roulet2019iterative}, which in turn may enjoy  stability properties \cite{westenbroek2021stability}. Making this connection rigorous poses several challenges. Among them, one would need to extend our low-regret guarantees against oblivious adversaries to hold against adaptive adversaries, the latter modeling how nonlinear system dynamics evolve in response to the learner's control inputs. This may require parting from our current analysis, which leverages the independence between exploratory inputs and changes in system dynamics.


Because we show that linear-in-$T$ regret is unavoidable for changing systems with large system variability, at least for the main convex policy parametrizations, it would be interesting to study our online setting under other measures of performance. In particular, the \emph{competive ratio}, or the \emph{ratio} of total algorithm cost to optimal cost in hindsight (as opposed to the \emph{difference} between the two measured by regret) may yield a complementary set of tradeoffs, or lead to new and exciting principles for adaptive controller design. Does system variability play the same deciding roles in competive analysis as it does in regret? And, in either competitive or regret analyses, what is the correct measure of system variability (e.g. variability in which norm/geometry, or of which system parameters) which best captures sensitivity of online cost to system changes?


\section*{Acknowledgments}
Elad Hazan and Edgar Minasyan have been  supported in part by NSF grant \#1704860. This work was done in part when Paula Gradu was at Google AI Princeton and Princeton University. Max Simchowitz is generously supported by an Open Philanthropy AI fellowship.
\newpage
\bibliography{ref}

\begin{thebibliography}{10}

\bibitem{abbasi2011regret}
Yasin Abbasi-Yadkori and Csaba Szepesv{\'a}ri.
\newblock Regret bounds for the adaptive control of linear quadratic systems.
\newblock In {\em Proceedings of the 24th Annual Conference on Learning
  Theory}, pages 1--26, 2011.

\bibitem{adamskiy2016closer}
Dmitry Adamskiy, Wouter~M Koolen, Alexey Chernov, and Vladimir Vovk.
\newblock A closer look at adaptive regret.
\newblock {\em The Journal of Machine Learning Research}, 17(1):706--726, 2016.

\bibitem{agarwal2019online}
Naman Agarwal, Brian Bullins, Elad Hazan, Sham Kakade, and Karan Singh.
\newblock Online control with adversarial disturbances.
\newblock In {\em International Conference on Machine Learning}, pages
  111--119, 2019.

\bibitem{agarwal2021regret}
Naman Agarwal, Elad Hazan, Anirudha Majumdar, and Karan Singh.
\newblock A regret minimization approach to iterative learning control.
\newblock {\em arXiv preprint arXiv:2102.13478}, 2021.

\bibitem{boffi2020regret}
Nicholas~M. Boffi, Stephen Tu, and Jean-Jacques~E. Slotine.
\newblock Regret bounds for adaptive nonlinear control, 2020.

\bibitem{bousquet2002tracking}
Olivier Bousquet and Manfred~K Warmuth.
\newblock Tracking a small set of experts by mixing past posteriors.
\newblock {\em Journal of Machine Learning Research}, 3(Nov):363--396, 2002.

\bibitem{boyd2004convex}
Stephen Boyd, Stephen~P Boyd, and Lieven Vandenberghe.
\newblock {\em Convex optimization}.
\newblock Cambridge university press, 2004.

\bibitem{bubeck2012regret}
S{\'e}bastien Bubeck and Nicolo Cesa-Bianchi.
\newblock Regret analysis of stochastic and nonstochastic multi-armed bandit
  problems.
\newblock {\em arXiv preprint arXiv:1204.5721}, 2012.

\bibitem{cesa2006prediction}
Nicolo Cesa-Bianchi and G{\'a}bor Lugosi.
\newblock {\em Prediction, learning, and games}.
\newblock Cambridge university press, 2006.

\bibitem{cohen2019learning}
Alon Cohen, Tomer Koren, and Yishay Mansour.
\newblock Learning linear-quadratic regulators efficiently with only $\sqrt{T}$
  regret.
\newblock In {\em International Conference on Machine Learning}, pages
  1300--1309, 2019.

\bibitem{daniely15strongly}
Amit Daniely, Alon Gonen, and Shai Shalev-Shwartz.
\newblock Strongly adaptive online learning.
\newblock In Francis Bach and David Blei, editors, {\em Proceedings of the 32nd
  International Conference on Machine Learning}, volume~37 of {\em Proceedings
  of Machine Learning Research}, pages 1405--1411, Lille, France, 07--09 Jul
  2015. PMLR.

\bibitem{dean2018regret}
Sarah Dean, Horia Mania, Nikolai Matni, Benjamin Recht, and Stephen Tu.
\newblock Regret bounds for robust adaptive control of the linear quadratic
  regulator.
\newblock In {\em Advances in Neural Information Processing Systems}, pages
  4188--4197, 2018.

\bibitem{dean2020certainty}
Sarah Dean and Benjamin Recht.
\newblock Certainty equivalent perception-based control.
\newblock {\em arXiv preprint arXiv:2008.12332}, 2020.

\bibitem{even2005experts}
Eyal Even-Dar, Sham~M Kakade, and Yishay Mansour.
\newblock Experts in a markov decision process.
\newblock In {\em Advances in neural information processing systems}, pages
  401--408, 2005.

\bibitem{fazel2018global}
Maryam Fazel, Rong Ge, Sham Kakade, and Mehran Mesbahi.
\newblock Global convergence of policy gradient methods for the linear
  quadratic regulator.
\newblock In {\em International Conference on Machine Learning}, pages
  1467--1476, 2018.

\bibitem{gradu2020adaptive}
Paula Gradu, Elad Hazan, and Edgar Minasyan.
\newblock Adaptive regret for control of time-varying dynamics.
\newblock {\em arXiv preprint arXiv:2007.04393}, 2020.

\bibitem{haastad2001some}
Johan H{\aa}stad.
\newblock Some optimal inapproximability results.
\newblock {\em Journal of the ACM (JACM)}, 48(4):798--859, 2001.

\bibitem{hazan2016introduction}
Elad Hazan.
\newblock Introduction to online convex optimization.
\newblock {\em Foundations and Trends{\textregistered} in Optimization},
  2(3-4):157--325, 2016.

\bibitem{hazan2019introduction}
Elad Hazan.
\newblock Introduction to online convex optimization, 2019.

\bibitem{hazan2019nonstochastic}
Elad Hazan, Sham Kakade, and Karan Singh.
\newblock The nonstochastic control problem.
\newblock In {\em Proceedings of the 31st International Conference on
  Algorithmic Learning Theory}, pages 408--421. PMLR, 2020.

\bibitem{hazan2009efficient}
Elad Hazan and Comandur Seshadhri.
\newblock Efficient learning algorithms for changing environments.
\newblock In {\em Proceedings of the 26th annual international conference on
  machine learning}, pages 393--400. ACM, 2009.

\bibitem{herbster1998tracking}
Mark Herbster and Manfred~K Warmuth.
\newblock Tracking the best expert.
\newblock {\em Machine learning}, 32(2):151--178, 1998.

\bibitem{kakade2020information}
Sham Kakade, Akshay Krishnamurthy, Kendall Lowrey, Motoya Ohnishi, and Wen Sun.
\newblock Information theoretic regret bounds for online nonlinear control.
\newblock {\em arXiv preprint arXiv:2006.12466}, 2020.

\bibitem{anima2}
Sahin Lale, Kamyar Azizzadenesheli, Babak Hassibi, and Anima Anandkumar.
\newblock Logarithmic regret bound in partially observable linear dynamical
  systems, 2020.

\bibitem{anima3}
Sahin Lale, Kamyar Azizzadenesheli, Babak Hassibi, and Anima Anandkumar.
\newblock Regret minimization in partially observable linear quadratic control,
  2020.

\bibitem{mania2019certainty}
Horia Mania, Stephen Tu, and Benjamin Recht.
\newblock Certainty equivalence is efficient for linear quadratic control.
\newblock In {\em Advances in Neural Information Processing Systems}, pages
  10154--10164, 2019.

\bibitem{mhammedi2020learning}
Zakaria Mhammedi, Dylan~J Foster, Max Simchowitz, Dipendra Misra, Wen Sun,
  Akshay Krishnamurthy, Alexander Rakhlin, and John Langford.
\newblock Learning the linear quadratic regulator from nonlinear observations.
\newblock {\em arXiv preprint arXiv:2010.03799}, 2020.

\bibitem{middleton1988adaptive}
Richard~H Middleton and Graham~C Goodwin.
\newblock Adaptive control of time-varying linear systems.
\newblock {\em IEEE Transactions on Automatic Control}, 33(2):150--155, 1988.

\bibitem{moore2012iterative}
Kevin~L Moore.
\newblock {\em Iterative learning control for deterministic systems}.
\newblock Springer Science \& Business Media, 2012.

\bibitem{oymak2019non}
Samet Oymak and Necmiye Ozay.
\newblock Non-asymptotic identification of lti systems from a single
  trajectory.
\newblock In {\em 2019 American control conference (ACC)}, pages 5655--5661.
  IEEE, 2019.

\bibitem{ozay2015set}
Necmiye Ozay, Constantino Lagoa, and Mario Sznaier.
\newblock Set membership identification of switched linear systems with known
  number of subsystems.
\newblock {\em Automatica}, 51:180--191, 2015.

\bibitem{qu2021stable}
Guannan Qu, Yuanyuan Shi, Sahin Lale, Anima Anandkumar, and Adam Wierman.
\newblock Stable online control of ltv systems stable online control of linear
  time-varying systems.
\newblock {\em arXiv preprint arXiv:2104.14134}, 2021.

\bibitem{rakhlin2014online}
Alexander Rakhlin and Karthik Sridharan.
\newblock Online non-parametric regression.
\newblock In {\em Conference on Learning Theory}, pages 1232--1264. PMLR, 2014.

\bibitem{roulet2019iterative}
Vincent Roulet, Siddhartha Srinivasa, Dmitriy Drusvyatskiy, and Zaid Harchaoui.
\newblock Iterative linearized control: stable algorithms and complexity
  guarantees.
\newblock In {\em International Conference on Machine Learning}, pages
  5518--5527. PMLR, 2019.

\bibitem{sarkar2019nonparametric}
Tuhin Sarkar, Alexander Rakhlin, and Munther Dahleh.
\newblock Nonparametric system identification of stochastic switched linear
  systems.
\newblock In {\em 2019 IEEE 58th Conference on Decision and Control (CDC)},
  pages 3623--3628. IEEE, 2019.

\bibitem{sarkar2019finite}
Tuhin Sarkar, Alexander Rakhlin, and Munther~A Dahleh.
\newblock Finite-time system identification for partially observed lti systems
  of unknown order.
\newblock {\em arXiv preprint arXiv:1902.01848}, 2019.

\bibitem{simchowitz2020making}
Max Simchowitz.
\newblock Making non-stochastic control (almost) as easy as stochastic, 2020.

\bibitem{simchowitz2020naive}
Max Simchowitz and Dylan Foster.
\newblock Naive exploration is optimal for online lqr.
\newblock In {\em International Conference on Machine Learning}, pages
  8937--8948. PMLR, 2020.

\bibitem{simchowitz2020improper}
Max Simchowitz, Karan Singh, and Elad Hazan.
\newblock Improper learning for non-stochastic control, 2020.

\bibitem{iLQR}
Y.~{Tassa}, T.~{Erez}, and E.~{Todorov}.
\newblock Synthesis and stabilization of complex behaviors through online
  trajectory optimization.
\newblock In {\em 2012 IEEE/RSJ International Conference on Intelligent Robots
  and Systems}, pages 4906--4913, 2012.

\bibitem{todorov2005generalized}
Emanuel Todorov and Weiwei Li.
\newblock A generalized iterative lqg method for locally-optimal feedback
  control of constrained nonlinear stochastic systems.
\newblock In {\em Proceedings of the 2005, American Control Conference, 2005.},
  pages 300--306. IEEE, 2005.

\bibitem{tsakalis1993linear}
Kostas~S Tsakalis and Petros~A Ioannou.
\newblock {\em Linear time-varying systems: control and adaptation}.
\newblock Prentice-Hall, Inc., 1993.

\bibitem{vershynin2018high}
Roman Vershynin.
\newblock {\em High-dimensional probability: An introduction with applications
  in data science}, volume~47.
\newblock Cambridge university press, 2018.

\bibitem{wang2019system}
Yuh-Shyang Wang, Nikolai Matni, and John~C Doyle.
\newblock A system-level approach to controller synthesis.
\newblock {\em IEEE Transactions on Automatic Control}, 64(10):4079--4093,
  2019.

\bibitem{westenbroek2021stability}
Tyler Westenbroek, Max Simchowitz, Michael~I Jordan, and S~Shankar Sastry.
\newblock On the stability of nonlinear receding horizon control: a geometric
  perspective.
\newblock {\em arXiv preprint arXiv:2103.15010}, 2021.

\bibitem{youla1976modern}
Dante Youla, Hamid Jabr, and Jr~Bongiorno.
\newblock Modern wiener-hopf design of optimal controllers--part ii: The
  multivariable case.
\newblock {\em IEEE Transactions on Automatic Control}, 21(3):319--338, 1976.

\bibitem{zhang2019adaptive}
Lijun Zhang, Tie-Yan Liu, and Zhi-Hua Zhou.
\newblock Adaptive regret of convex and smooth functions.
\newblock {\em arXiv preprint arXiv:1904.11681}, 2019.

\bibitem{zhang18dynamic}
Lijun Zhang, Tianbao Yang, rong jin, and Zhi-Hua Zhou.
\newblock Dynamic regret of strongly adaptive methods.
\newblock In Jennifer Dy and Andreas Krause, editors, {\em Proceedings of the
  35th International Conference on Machine Learning}, volume~80 of {\em
  Proceedings of Machine Learning Research}, pages 5882--5891. PMLR, 10--15 Jul
  2018.

\bibitem{zhou1996robust}
Kemin Zhou, John~Comstock Doyle, Keith Glover, et~al.
\newblock {\em Robust and optimal control}, volume~40.
\newblock Prentice hall New Jersey, 1996.

\end{thebibliography}
\bibliographystyle{plain}
\newpage
\tableofcontents
\appendix
\newcommand{\Kstab}{K^{\mathrm{stb}}}
\newcommand{\Phistab}{\Phi^{\mathrm{stb}}}
\newcommand{\ynat}{y^{\mathrm{nat}}}
\newcommand{\ynathat}{\hat{y}^{\mathrm{nat}}}
\section{Extensions \label{app:extensions}}
\subsection{Affine Offsets}
For many systems, performance improves dramatically for controllers with constant affine terms, that is
\begin{align*}
u^M_t = \bar{u}^M + \sum_{t=0}^{h-1}M^{[i]}\xnat_{t-i},
\end{align*}
where $\bar{u}^M$ is a constant affine term encoded by $M = (M^{[0]},\dots,M^{[h-1]},\bar{u}^{M})$. All our arguments apply more generally to control policies of this form. Moreover, we can even allow linear combinations of time varying terms:
\begin{align*}
u^M_t &= \sum_{t=0}^{h-1}M^{[i]}\xnat_{t-i} + \sum_{t=0}^{h'-1}M^{[i]'} \psi_i(t)\\
&\quad M = (M^{[0]},\dots,M^{[h-1]},M^{[0]'},\dots,M^{[h'-1]'}),
\end{align*}
where now $\psi_i(t)$ are fixed, possibly time varying basis functions (which do not depend on $M$). The case of constant affine terms corresponds to $h' = 1$, and $\psi_i(t) = 1$ for all $t$.

\subsection{Changing Stabilizing Controllers}\label{sec:app:stabilizable}
Our results extend naturally to the following setting: for each time $t = 1,2,\dots,T$, the algorithm has access to a static feedback control policy $\Kstab_t$ such that the closed loop matrices $(A_t + B_t \Kstab_t)$ are sequentially stable, that is
\begin{align*}
\Phistab_{s+h,s} := \prod_{i=s}^{s+h} (A_t + B_t \Kstab_t)
\end{align*} 
has geometric decay. We let $\xnat_t$ denote the iterates produced by the updates
\begin{align*}
\xnat_{t+1} = (A_t + B_t \Kstab_t)\xnat_t, \unat_t = K\xnat_t, \quad \xnat_1 = 0.
\end{align*}
We compute the stabilized policies of the form
\begin{align*}
u_t^M = \Kstab_t x_t + \sum_{i=0}^{h-1}M^{[i]}\xnat_{t-i}.
\end{align*}
To facillicate the extension, we define the stabilized Markov operator
\begin{align*}
\Gstt^{[0]} = \begin{bmatrix} 0\\
I_{\dimu}
\end{bmatrix}, \quad \Gstt^{[i]} := \begin{bmatrix} I_{\dimx}\\
K_t
\end{bmatrix}\Phistab_{t,t-i+1} B_{t-i},
\end{align*}
\newcommand{\util}{\tilde{u}}
This Markov operator satisfies
\begin{align*}
\begin{bmatrix}x_t \\ u_t \end{bmatrix} &= \begin{bmatrix}\xnat_t \\\unat \end{bmatrix} + \sum_{i=0}^{t-1} G_{\star,t}^{[i]}\util_t, \quad \text{s.t. } u_t = K_t x_t + \util_t.
\end{align*}
With similar techniques, we obtain estimates $\Ghat_{t}$, back out estimates of the Nature's sequence $(\xnat,\unat)$ via 
\begin{align*}
\begin{bmatrix}\xnathat_t \\\unathat \end{bmatrix} = \clip_{r}\left(\begin{bmatrix}x_t \\ u_t \end{bmatrix} - \sum_{i=0}^{h} \Ghat_t^{[i]}(u_{t-i} - \Kstab_{t-i} x_{t-i})\right),
\end{align*}
for a truncation radius $r > 0$ suitably chosen. Recall that we apply this $\clip_r(\cdot)$ operator to ensure the estimates of the Nature's sequenc does not grow unbounded and exert feedback. We then select inputs
\begin{align*}
\uhat_t(M) = \Kstab_t x_t + \sum_{i=0}^{h-1}M^{[i]}\xnathat_{t-i},
\end{align*}
 Markov operator with low adaptive regret, and apply our \oco{}-with-memory algorithm to losses
\begin{align*}
\fhat_t(M) = c_t\left(\begin{bmatrix}\xnathat_t \\ \unathat_t \end{bmatrix} -   \sum_{i=0}^h\Ghat_t^{[i]}\uhat_{t-i}(M)\right).
\end{align*}

\subsection{Partial Observation}
Our results further extend to partially observed systems. We explain this extension for sequentially stable systems; extensions to sequentialy stabilized systems by time varying linear dynamic controllers follows from the exposition in \cite{simchowitz2020improper}, Appendix C. 

For partially observed systems, we have the same state transition dynamics $x_{t+1} =A_tx_t + B_tu_t + w_t$, but, for a time-varying observation matrix $C_t$ and process noise $e_t$, we observe outputs $y_t \in \R^{\dimy}$
\begin{align*}
y_t = C_tx_t + e_t.
\end{align*}
Costs $c_t(y_t,u_t)$ are suffered on input and outputs.
As for full observation, the Nature's sequence $\xnat$ and $\ynat$ correspond to states $x_t$ and outputs $y_t$ which arise under identially zero input $u_t \equiv 0$. The \drc{} parametrization selects linear conmbinations of Nature's y's:
\begin{align*}
u_t^M  = \sum_{i=0}^{m-1} M^{[i]} \ynat_{t-i}.
\end{align*}
Recalling
\begin{align*}
\Phi_{s+h,s} := \prod_{i=s}^{s+h} (A_t + B_t),
\end{align*} 
the relevant Markov operators $G_{\star t}$ are the ones mapping inputs to outputs:
\begin{align*}
\Gstt^{[0]} = 0, \quad \Gstt^{[i]} = C_t\Phi_{t,t-i+1} B_{t-i},
\end{align*}
With similar techniques, we obtain estimates $\Ghat_{t}$, back out estimates of the Nature's sequence $\ynat$ via 
\begin{align*}
\ynathat = \clip_r\left(y_t - \sum_{i=0}^{h} \Ghat_t^{[i]} u_{t-i}\right) 
\end{align*}
for trunctation radius $r > 0$ suitably chosen, we select inputs
\begin{align*}
\uhat_t(M) = \sum_{i=0}^{h-1}M^{[i]}\ynathat_{t-i}),
\end{align*}
update parameters with the \oco-with-memory losses
\begin{align*}
\fhat_t(M) = \ell_t(\ynathat_t - \sum_{i=0}^h\Ghat_t^{[i]} M^{[i]} \uhat_{t-i}(M), u_t).
\end{align*}
\subsection{The \dac{} parametrization}
Here we sketch an algorithm to compete with \dac{}-parametrized control policies \citep{agarwal2019online}. For simplicity, we focus on sequentially stable system, though the discussion extends to systems sequentially stabilized by sequences of controllers $(\Kstab_t)$. Note that \dac{} \emph{does not} apply under partial observation.

Recall that, in the \dac{} parametrization, the inputs are selected as linear combinations of past disturbances:
\begin{align*}
u_t^M = \sum_{i=0}^{h-1} M^{[i]}w_{t-i-1}.
\end{align*}
To implement \dac{}, we therefore need empiricals estimate $\what_t$ of $w_t$. As per \cite{hazan2019nonstochastic}, it suffices to construct estimates $(\Ahat_t,\Bhat_t)$ of $(A_t,B_t)$, and choose
\begin{align*}
\what_t = \clip_r\left(x_{t} - \Ahat_{t-1} x_{t-1} - \Bhat_{t-1}u_{t-1}\right),
\end{align*}
again clipped at a suitable radius $r > 0$ to block compounding feedback. Given these estimates, our algorithm extends to \dac{} control in the expected way.

How does one obtain the estimates $(\Ahat_t,\Bhat_t)$? First, we observe that since
\begin{align*}
\Gstt^{[0]} = 0, \quad \Gstt^{[i]} =   \Phi_{t,t-i+1} B_{t-i}, \quad \Phi_{s+h,s} = \prod_{i=s}^{s+h} A_i,
\end{align*}
we have $B_t = \Gstt^{[1]}$. Hence, we can select $\Bhat_t$ as $\Ghat_t^{[1]}$. 

The estimate of $A_t$ is more involved. For linear, \emph{time-invariant} systems, $A$ can be recovered from $G_\star$ via the Ho-Kalman procedure as is does in \cite{hazan2019nonstochastic} (see also \cite{oymak2019non,sarkar2019finite}). For time-varying systems, this become more challenging. Ommitting details in the interest of brevity, one can use the robustness properties of Ho-Kalman to argue that if the system matrices are slow moving (an assumption required for low regret), $\Gstt$ is close to a stationarized analogue $\bar{G}_{\star,t}$ given by
\begin{align*}
\bar{G}_{\star,t}^{[i]} := A_t^{i-1}B_t.
\end{align*}
Hence, we can view any estimate $\Ghat_t$ of $\Gstt$ as an estimate of $\bar{G}_{\star,t}$, and apply Ho-Kalman to the latter.
\newcommand{\ftilt}{\tilde{f}_t} 

\section{Adapive Regret for Time-Varying DRC-OGD}\label{sec:app:drc-ogd}

We first extend the DRC-OGD algorithm from \cite{simchowitz2020improper} to the setting of \emph{known} linear time-varying dynamics. We spell out \Cref{alg:DRC-OGD} and prove it attains $\tilde{\mathcal{O}}\left(\sqrt{T}\right)$ adaptive regret over general convex costs under fully adversarial noise (\Cref{thm:drc_ogd_ltv}) with respect to the DRC policy class. The main technique, using OGD over the DRC parametrization, remains unchanged from the original paper and we show it generalizes naturally to LTV systems. 

\begin{algorithm}[H]
\begin{algorithmic}[1] 
\State{} \textbf{Input:} stepsize $\eta$, memory $m$, radius $R_\mathcal{M}$
\State{} Initialize $M_1 \in \mathcal{M}(m, R_\mathcal{M})$ arbitrarily
\State{} Receive initial state $x_1$, set $\xnat_1 = x_1$ and $\xnat_{\leq 0} = 0$
\For{$t$ = $1 \ldots T$}
        \State{} Play control $u_t = \sum_{i=0}^{m-1} M_t^{[i]} x^{\text{nat}}_{t-i} $
        \State{} Suffer $c_t(x_t, u_t)$ and observe cost function $c_t(\cdot, \cdot)$
        \State{} Construct $f_t(M_{0}, \ldots, M_h) \doteq c_t(\hat{x}_t(M_{0:{h-1}}), u_t(M_h))$ and let $\ftilt(M) \doteq f_t(M, \ldots, M)$ \label{line:counterfact_loss_drc_ogd}
        \State{} Update $M_{t+1} \leftarrow \Pi_\mathcal{M}\left(M_t - \eta \nabla \tilde{f}_t(M_t)\right)$
        \State{} Receive new system $G_t$
        \State{} Receive new state $x_{t+1}$ and extract $\xnat_{t+1} = x_{t+1} - \sum_{i=0}^{t-1} G_t^{[i]} u_{t-i}$ \emph{or} receive $\xnat_{t+1}$
        
\EndFor
 \caption{Disturbance Response Control via Online Gradient Descent (DRC-OGD)}\label{alg:DRC-OGD}
\end{algorithmic}
\end{algorithm}
\begin{theorem}\label{thm:drc_ogd_ltv} Running \Cref{alg:DRC-OGD} with $\eta = \frac{\sqrt{d_{\mathrm{min}}} R_\calM^2}{2 L R^2_{\mathrm{sys}} (h+1)^{5/4} \sqrt{T}}$ guarantees the following regret bound on every interval $I=[r,s]$:
$$\sum_{t=r}^s c_t(x_t, u_t) - \min_{\pi\in \Pi_{\mathrm{drc}}} \sum_{t=r}^s c_t(x_t^{\pi}, u_t^{\pi}) \leq 6L R_{\mathrm{sys}}^2 \left(3 \sqrt{d_{\min}}m (h+1)^{5/4} \sqrt{T} + \psi(h) |I|\right)$$
where $R_{\mathrm{sys}} = R_G R_\calM R_{\mathrm{nat}}$ and $d_{\mathrm{min}} = \min\{d_x, d_u\}$.
\end{theorem}

\subsection{Adaptive Regret of OGD for functions with memory}\label{sec:ada_ogd_mem}

We first prove that OGD with a fixed stepsize attains $O(\sqrt{T})$ adaptive regret for functions with memory.

\begin{algorithm}[H]
\begin{algorithmic}[1] 
\State{} \textbf{Input:} stepsize $\eta$, memory $m$, set $\mathcal{K}$
\State{} Initialize $x_{1} \in\mathcal{K}$ arbitrarily, set $x_{\leq 0} = x_1$ by convention
\For{$t$ = $1 \ldots T$}
        \State{} Play $x_t$
        \State{} Suffer $f_t(x_{t-h}, \ldots, x_t)$ and observe loss function $f_t(\cdot, \ldots, \cdot)$
        \State{} Construct proxy loss function $\tilde{f}_t(x) \doteq f_t(x, \ldots, x)$ \label{def:proxy_loss_oco_mem}
        \State{} Update $x_{t+1} = \Pi_\mathcal{K} \left(x_t - \eta \nabla \tilde{f}_t(x_t)\right)$
        
\EndFor
 \caption{Online Gradient Descent for OCOwMem (Mem-OGD)}\label{alg:Mem-OGD}
\end{algorithmic}
\end{algorithm}

\begin{theorem}\label{thm:ogd_oco_mem_adaptive}
Let $\{f_t : \mathcal{K}^{h+1}\rightarrow [0,1]\}_{t=1}^T$ be a sequence of L coordinate-wise Lipschitz loss functions with memory such that $\ftilt$ (Line 6) is convex. Then, on any interval $I=[r,s] \subseteq [T]$, \Cref{alg:Mem-OGD} enjoys the following adaptive policy regret guarantee:
$$\sum_{t=r}^s f_t(x_{t-h}, \ldots, x_t) - \min_{x\in\calK} \sum_{t=r}^s \ftilt(x) \leq \frac{D^2}{\eta} + 2 \eta L^2(h+1)^{5/2} |I|$$
where $D = \text{diam}(\calK)$.
\end{theorem}

First we state and prove the following well-known fact about vanilla projected OGD over (memory-less) loss functions:

\begin{fact}\label{fact:ogd_adaptive}
Let $\{\tilde{f}_t : \mathcal{K} \rightarrow [0,1]\}_{t=1}^T$ be a sequence of convex loss functions with $\|\nabla \tilde{f}(x)\| \leq G$. Then, on any interval $I=[r,s] \subseteq [T]$, projected OGD enjoys the following guarantee:
$$\sum_{t=r}^s \tilde{f}_t(x_t) - \min_{x\in \mathcal{K}}\sum_{t=r}^s \tilde{f}(x) \leq \frac{D^2}{\eta} + \eta |I| G^2$$
where $D = \text{diam}(\calK)$.
\end{fact}

\begin{proof} Consider an arbitary interval $I= [r,s]\subseteq [T]$
Let $\x^\star = \arg\min_{x\in\mathcal{K}} \sum_{t=r}^s \tilde{f}_t(x)$ and denote $\nabla_t \doteq \nabla \tilde{f}_t(x_t)$ for simplicity. By convexity we have

\begin{equation}\label{eq:convexity_ogd_analysis}
\ftilt(x_t) - \ftilt(x^\star) \leq \nabla_t^\top (x_t - x^\star) 
\end{equation}

By the Pythagorean theorem

\begin{align}
\norm{x_{t+1} - x^\star}^2 &= \norm{\Pi_\calK(x_t - \eta \nabla_t)  - x^\star}^2 \notag \\
&\leq \norm{x_t - \eta \nabla_t - x^\star}^2  \notag\\ 
&\leq \norm{x_t - x^\star}^2 + \eta^2 \norm{\nabla_t}^2 - 2\eta \nabla_t^\top (x_t - x^\star) \label{eq:pythagoras_ogd_analysis}\end{align}
Hence we can bound the interval regret as:
\begin{align*}
\Rightarrow 2 \sum_{t=r}^s \left( f_t(x_t) - f_t(x^\star) \right) & \leq 2 \sum_{t=r}^s \nabla_t^\top (x_t - x^\star) & \Cref{eq:convexity_ogd_analysis}\\
&\leq \sum_{t=r}^s \left( \frac{\norm{x_t - x^\star}^2 - \norm{x_{t+1} - x^\star}^2}{\eta} + \eta G^2 \right) & \Cref{eq:pythagoras_ogd_analysis}\\
&= \frac{\norm{x_r - x^\star}^2}{\eta} + |I| \eta G^2 \\
&\leq \frac{D^2}{\eta} + \eta |I| G^2
\end{align*}
which yields the desired adaptive regret bound.
\end{proof}

Using this simple fact and Lipschitzness we are able to easily prove the desired guarantee for \Cref{alg:Mem-OGD}.

\begin{proof}[Proof of \Cref{thm:ogd_oco_mem_adaptive}]
First note that \Cref{alg:Mem-OGD} is just doing gradient descent on the proxy convex loss functions $\tilde{f}_t:\calK \rightarrow [0,1]$. Hence, as long as we identify the gradient bound we can apply \Cref{fact:ogd_adaptive} to get a bound on $\tilde{f}$-regret. 
Observe that
\begin{align*}
|\ftilt(x) - \ftilt(y)| &= |f_t(x, \ldots, x) - f_t(y, \ldots, y)| \\
&\leq |f_t(x, \ldots, x) - f_t(y, \ldots, x)| + |f_t(y, x, \ldots,x) - f_t(y, y, \ldots, x)| + \\
& \ldots + |f_t(y, \ldots, y, x) - f_t(y, \ldots, y)|\\
&\leq L (h+1) \norm{x-y}
\end{align*}

so $\ftilt$ is $L(h+1)$-Lipschitz and hence has a gradient bound of $L(h+1)$. So we can apply \Cref{fact:ogd_adaptive} to get

\begin{equation}\label{eq:proxy_loss_regret}
    \sum_{t=r}^s \tilde{f}_t(x_t) - \min_{x\in \mathcal{K}}\sum_{t=r}^s \tilde{f}(x) \leq \frac{D^2}{\eta} + \eta |I| L^2 (h+1)^2
\end{equation}

We can use this to bound the adaptive policy regret. First note that

$$\norm{x_t - x_{t-1}} = \norm{\eta\nabla\ftilt(x_{t-1})} \leq \eta L (h+1)$$

and by the triangle inequality \begin{equation}\label{eq:iterate_distance}
\norm{x_t - x_{t-i}} \leq \sum_{j=1}^i \norm{x_{t-j+1} - x_{t-j}} \leq \eta L (h+1) \cdot i
\end{equation}

Using \Cref{eq:iterate_distance} and Lipschitzness we have:

\begin{align}
f_t(x_{t-h}, \ldots, x_t) - \ftilt(x_t) &\leq L\norm{(x_{t-h}, \ldots, x_t) - (x_t, \ldots, x_t)} \notag\\
&\leq L \sqrt{\sum_{i=1}^h \norm{x_t - x_{t-i}}^2} \notag\\
&\leq \eta L^2 (h+1) \cdot \sqrt{\sum_{i=1}^h i^2} \notag\\
&\leq \eta L^2 (h+1)^{5/2} \label{eq:lipscitz_relate}
\end{align}
\end{proof}
Combining everything we get
\begin{align*}
\sum_{t=r}^s f_t(x_{t-h}, \ldots, x_t) - \min_{x\in \calK}\sum_{t=r}^s\ftilt(x) &= \errtext{\sum_{t=r}^s f_t(x_{t-h}, \ldots, x_t) - \ftilt(x_t)}{dist. to proxy loss [\Cref{eq:lipscitz_relate}]} + \errtext{\sum_{t=r}^s \ftilt(x_t) - \min_{x\in\calK}\sum_{t=r}^s \ftilt(x_t)}{$\tilde{f}$-regret [\Cref{eq:proxy_loss_regret}]} \\
&\leq 2 \eta L^2(h+1)^{5/2} |I| + \frac{D^2}{\eta} 
\end{align*}

\subsection{Proof of \Cref{thm:drc_ogd_ltv}}

We first prove that the constructed loss function satisfies key properties for efficient optimization.

\begin{lemma}[Convexity]\label{lem:drc_ogd_convexity}
The loss functions $\ftilt$ constructed in \Cref{line:counterfact_loss_drc_ogd} of \Cref{alg:DRC-OGD} are convex in $M$. 
\end{lemma}

\begin{proof} By definition we have that:
\begin{align}
\hat{x}_t(M) &= \xnat_t + \sum_{i=1}^h G_t^{[i]} u_{t-i} \notag\\
&= \xnat_t + \sum_{i=1}^h G_t^{[i]} \left(\sum_{j=0}^{m-1} M^{[j]} \xnat_{t-i-j}\right) \label{eq:proxy_state_expansion}
\end{align}
which is affine in $M$. Even more simply, we have $u_t(M) = \sum_{i=0}^{m-1} M^{[i]} \xnat_{t-i}$.

Since $\hat{x}_t(M)$ and $u_t(M)$ are affine, and, respectively, linear functions of $M$ and composition with the convex cost $c_t$ preserves convexity we get the desired property.
\end{proof}

\begin{lemma}[Lipschitzness]\label{lem:drc_ltv_lipschitz}
The loss functions $f_t$ constructed in \Cref{line:counterfact_loss_drc_ogd} of \Cref{alg:DRC-OGD} are $L_f$ coordinate-wise Lipschitz for $L_f = 3LR_{\text{nat}}^2 R_G^2 R_\calM \sqrt{m}$.
\end{lemma}

\begin{proof} Observe that by \Cref{eq:proxy_state_expansion} we have $\norm{\hat{x}_t(M_{0:h)}} \leq R_{\text{nat}}(1+ R_G R_\calM)$. Straightforwardly, $\norm{u_t(M_h)} \leq R_{\text{nat}} R_\calM$ as well.

For an arbitrary $i\in \overline{0, h}$, denoting $M_{0:h} = (M_0, \ldots, M_i, \ldots M_h)$, $\tilde{M}_{0:h} = (M_0, \ldots,\tilde{M}_i, \ldots, M_h)$ and using the sub-quadratic lipschitzness of the costs we have:

\begin{align*}
|f_t(M_{0:h}) - f_t(\tilde{M}_{0:h})| &= |c_t(\hat{x}_t(M_{0:h}), u_t(M_{0:h})) - c_t(\hat{x}_t(\tilde{M}_{0:h}), u_t(\tilde{M}_{0:h}))| \\
&\leq 3 L R_{\text{nat}} R_G R_\calM \bigg| \bigg| G_t^{[h-i]} \left(\sum_{j=0}^{m-1} (M_{i}^{[j]} - \tilde{M}_{i}^{[j]}) \xnat_{t-i-j}\right) \bigg| \bigg| & (i < h)\\
& \text {or }\leq 3 L R_{\text{nat}} R_G R_\calM \left(\sum_{j=0}^{m-1} (M_{i}^{[j]} - \tilde{M}_{i}^{[j]}) \xnat_{t-i-j}\right) & (i = h)\\
&\leq 3L R_{\text{nat}}^2 R_G^2 R_\calM \sqrt{m} \|M_i - \tilde{M}_i\|_F
\end{align*}
so the function is coordinate-wise Lipschitz with constant $L_f = 3LR_{\text{nat}}^2 R_G^2 R_\calM \sqrt{m}$.
\end{proof}

\begin{lemma}[Euclidean Diameter]\label{lem:drc_ltv_diam} The euclidean diameter of $\mathcal{M}(m, R_\mathcal{M})$ is at most $D = 2 \sqrt{m \cdot \min\{d_x, d_u\}} R_\mathcal{M}$.
\end{lemma}

\begin{proof}
That for an arbitrary $M\in \mathcal{M}(m, R_\mathcal{M})$, we have

\begin{align*}
\|M\|_F &= \sqrt{\sum_{i=0}^{m-1} \|M^{[i]}\|_F^2} \\
&\leq \sqrt{m \cdot \max_{i\in[m-1]}\|M^{[i]}\|_F^2} \\
&= \sqrt{m} \max_{i\in[m-1]} \min\{d_x, d_u\} \|M^{[i]}\|_{op} \\
&\leq \sqrt{m \cdot \min\{d_x, d_u\}} \|M\|_{\ell_1, op} \\
&\leq \sqrt{m \cdot \min\{d_x, d_u\}} R_\mathcal{M}
\end{align*}

and the euclidean diameter is at most twice the maximal euclidean norm, concluding our statement.
\end{proof}

The three lemmas above will allow us to use the results in \Cref{sec:ada_ogd_mem} to obtain adaptive regret guarantees in terms of $f_t$ which truncated the effect on the state of actions further than $h$ in the past. To convert guarantees in terms of $f_t$ to ones in terms of $c_t$, we prove that the effect of the past is minimal:

\begin{lemma}[Truncation Error]\label{lem:truncation_error} For a changing DRC policy that acts according to $M_1, \ldots, M_t$ up to time $t$ we have that:
$$c_t(x_t, u_t) - c_t\left(\hat{x}_t(M_{t-h:t-1}), u_t(M_t)\right) \leq 3L R_{\text{nat}}^2 R_\calM^2 R_G \psi(h)$$
\end{lemma}

\begin{proof} By the sub-quadratic Lipschitzness (and noting $\|u_t\|_2 \leq R_\mathcal{M} R_{\text{nat}}$, $\|x_t\|_2 \leq R_nat(1+R_G R_\mathcal{M})$, and $u_t = u_t(M_t) $) we have:

\begin{align*}
c_t(x_t, u_t) - c_t\left(\hat{x}_t(M_{t-h:t-1}), u_t(M_t)\right) &\leq
 3 L R_{\text{nat}} R_G R_\calM \norm{x_t(M_{1:{t-1}} - \hat{x}_t(M_{t-h:t-1})} \\
 &=  3 L R_{\text{nat}} R_G R_\calM \norm{\sum_{i=h}^{t-1} G_{t-1}^{[i]}u_{t-1-i}} \\
 &\leq 3L R_{\text{nat}}^2 R_\calM^2 R_G \psi(h)
\end{align*}
\end{proof}

Having proven all these preliminary results the proof of the main theorem is immediate:

\begin{proof}[Proof of \Cref{thm:drc_ogd_ltv}]
By the definition of the proxy loss in Line 7 of \Cref{alg:DRC-OGD}, we can expand the regret of \Cref{alg:DRC-OGD} over interval $I=[r,s]$ as:

\begin{align*}
\text{Regret} &= \sum_{t=r}^s c_t(x_t, u_t) - \min_{M \in \mathcal{M}} \sum_{t=r}^s c_t(x_t^M, u_t^M) \\
&= \errtext{\sum_{t=r}^s c_t(x_t, u_t) - \sum_{t=r}^s c_t(\hat{x}_t(M_{t-h:t-1}), u_t(M_t))}{truncation error I} + \errtext{\sum_{t=r}^s f_t(M_{t-h:t}) - \min_{M\in \mathcal{M}} \ftilt(M)}{f-regret} \\&+ \errtext{\min_{M\in \mathcal{M}} \sum_{t=r}^s c_t(\hat{x}_t(M), u_t(M)) - \min_{M \in \mathcal{M}} \sum_{t=r}^s c_t(x_t^M, u_t^M)}{truncation error II}
\end{align*}

The first truncation error is bounded directly by \Cref{lem:truncation_error}. For the second truncation error, let $M^\star = \arg \min_{M \in \mathcal{M}} \sum_{t=r}^s c_t(x_t^M, u_t^M)$. Clearly we have $$\min_{M\in \mathcal{M}} \sum_{t=r}^s c_t(\hat{x}_t(M), u_t(M)) \leq \sum_{t=r}^s c_t(\hat{x}_t(M^\star), u_t(M^\star))$$ and hence we can apply \Cref{lem:truncation_error} to bound:

\begin{align*}
\text{truncation error II} &= \min_{M\in \mathcal{M}} \sum_{t=r}^s c_t(\hat{x}_t(M), u_t(M)) - \min_{M \in \mathcal{M}} \sum_{t=r}^s c_t(x_t^M, u_t^M) \\
&\leq \sum_{t=r}^s c_t(\hat{x}_t(M^\star), u_t(M^\star)) -\sum_{t=r}^s c_t(x_t^{M^\star}, u_t^{M^\star}) \\
&\leq 3L R_{\text{nat}}^2 R_\calM^2 R_G \psi(h) |I|
\end{align*}

Finally, due to \Cref{lem:drc_ogd_convexity}, \Cref{lem:drc_ltv_lipschitz} and \Cref{lem:drc_ltv_diam} we can apply \Cref{thm:ogd_oco_mem_adaptive} to get

$$\text{f-regret} \leq \frac{D^2}{\eta} + 2\eta L_f^2(h+1)^{5/2}|I|$$

Summing everything up and plugging in the Lipschitz and diameter constants, we have:

\begin{align*}
\text{Regret} &\leq 6L R_{\text{nat}}^2 R_\calM^2 R_G \psi(h) |I| + \frac{4m\min\{d_x, d_u\}R_\calM^2}{\eta} + 18 \eta (L^2 R_{\text{nat}}^4 R_G^4 R_\calM^2)m (h+1)^{5/2} |I|
\end{align*}
Setting $\eta \doteq \frac{\sqrt{\min\{d_x, d_u\}}}{2LR_G^2 R_{\text{nat}}^2 (h+1)^{5/4} \sqrt{T}}$, we get

\begin{align*}
\text{Regret} &\leq 6L R_{\text{nat}}^2 R_\calM^2 R_G \psi(h) |I|  + 17 \sqrt{\min\{d_x, d_u\}} (LR_G^2 R_{\text{nat}}^2R_\calM^2) \cdot (h+1)^{5/4} m \cdot \sqrt{T} \\
&\leq 6L R_{\text{sys}}^2 \left(3 \sqrt{d_{\min}}m (h+1)^{5/4} \sqrt{T} + \psi(h) |I|\right)
\end{align*} 
where we denote $R_{\text{sys}}\doteq R_G R_\calM R_{\text{nat}}$ and $d_{\text{min}} \doteq \min\{d_x, d_u\}$.
\end{proof}

\section{Estimation of Time-Varying Vector Sequences}\label{sec:sys_est}

In this section we segway into the setting of online prediction under a partial information model. The goal is to estimate a sequence of vectors under limited noisy feedback where the feedback access is softly restricted via additional cost. As shown in the following section, this setting captures the \emph{system identification} phase of controlling an unknown time-varying dynamical system. We first extensively study the simplified setting as below, and afterwards transfer our findings into meaningful results in control.

Formally, consider the following repeated game between a learner and an \emph{oblivious} adversary: at each round $t \in [T]$, the adversary picks a target vector $\matzst_t \in \calK$ from a convex decision set $\calK$ contained in a $0$-centered ball of radius $R_z$; simultaneously, the learner selects an estimate $\matzhat_t \in \calK$ and suffers quadratic loss $f_t(\matzhat_t) = \| \matzhat_t - \matzst_t \|^2$. The only feedback the learner has access to is via the following noisy and costly oracle.

\begin{oracle}[Noisy Costly Oracle]\label{orac:noisy} At each time $t \in [T]$, the learner selects a decision $b_t \in \{0,1\}$ indicating whether a query is sent to the oracle. If $b_t = 1$, the learner receives an unbiased estimate $\matztil_t$ as response such that $\| \matztil_t \| \leq \tilde{R}_z$ and $\E[\matztil_t \mid \calF_t, b_t=1] = \matzst_t$ where $\calF_t$ is the filtration sigma algebra generated by the entire sequence $\matzst_{1:T}$ and the past $\matztil_{1:t-1}, b_{1:t-1}$. A completed query results in a unit cost for the learner denoted $b_t$ as well by abuse of notation.
\end{oracle}

The idea behind this setting is to model a general estimation framework for a time-varying system which focuses only on exploration. Committing to exploration, however, cannot realistically be free hence the additional cost for the number of calls to \Cref{orac:noisy}. Our goal is to design an algorithm $\calA$ that minimizes the quadratic loss regret along with the extra oracle cost, defined over each interval $I = [r, s] \subseteq [T]$ as
\begin{equation}\label{eq:est_regret}
    \mathrm{Regret}_I(\calA; \lambda) = \Exp\left[ \sum_{t \in I} f_t(\matzhat_t) \right] - \min_{\matz \in \calK} \sum_{t \in I} f_t(\matz) + \lambda \Exp \left[ \sum_{t \in I} b_t \right],
\end{equation}
where $\lambda \geq 0$ is a scaling constant independent of the horizon $T$. For the $I = [T]$ entire interval, we use $T$ as a subscript instead of $[T]$. The expectation above is taken over both the (potential) randomness of the algorithm and the stochasticity of the oracle responses; it is taken in the round order $t=1, \dots, T$ at each round conditioning on the past iterations. 

In terms of estimation itself, the metric to consider over interval $I$ is given by $\mathrm{Regret}_I(\calA; 0)$ that ignores the oracle call costs. Furthermore, we observe that the best-in-hindsight term in \eqref{eq:est_regret} is in fact a fundamental quantity of the vector sequence as defined below. This formulation will be used, and is more appropriate, when transferring our findings to the setting of control.

\begin{definition}\label{def:variance}
Define the variability of a time-varying vector sequence $\matz_{1:T}$ over an interval $I \subseteq [T]$ to be
\begin{equation*}
    \var_I(\matz_{1:T}) = \frac{1}{|I|}\min_{\matz \in \calK} \sum_{t \in I} \| \matz - \matz_t \|^2 = \frac{1}{|I|} \sum_{t\in I} ||\matzbar_I - \matz_t||^2,
\end{equation*}
where $\matzbar_I = |I|^{-1} \sum_{t\in I} \matz_t \in \calK$ is the empirical average of the members of the sequence that correspond to $I$.
\end{definition}

This definition concludes the setup of our abstraction to general estimation of vector sequences. Regarding algorithmic results, we first present a base method that achieves logarithmic regret over the entire trajectory $[1, T]$. The idea is for the learner to uniformly query \Cref{orac:noisy} with probability $p$: once an estimate $\matztil_t$ is received, construct a stochastic gradient with expectation equal to the true gradient, and perform a gradient update. The algorithm is described in detail in \Cref{alg:est_base}, and its guarantee given in the theorem below.

\begin{algorithm}
\caption{Base Estimation Algorithm}\label{alg:est_base}
\begin{algorithmic}[1]
    \State \textbf{Input:} $p$, $\matzhat_1 \in \calK$
    \For{$t=1, \ldots, T$}
    \State{}\textbf{Play} iterate $\matzhat_t$
    \State{}\textbf{Draw/Receive} $b_t \sim \Bernoulli(p)$ 
    \If{$b_{t} = 1$}
    \State{}\textbf{Receive} estimate $\matztil_t$ from \Cref{orac:noisy}
    \State{}\textbf{Construct} importance weighted gradient $\nabtil_t := \frac{1}{p}\left(\matzhat_t - \matztil_t\right)$
    \Else
    \State{}\textbf{Set} $\nabtil_t = 0$. 
    \EndIf
    \State{}\textbf{Update} $\matzhat_{t+1} = \Proj_{\calK}(\matzhat_t - \eta_t \nabtil_t)$, $\eta_t = \frac{1}{t}$.
    \EndFor
\end{algorithmic}
\end{algorithm}

\begin{theorem}\label{thm:est_base} Given access to queries from \Cref{orac:noisy}, with stepsizes $\eta_t = \frac{1}{t}$, \Cref{alg:est_base} enjoys the following regret guarantee:
\begin{equation}\label{eq:est_regret_bound}
\mathrm{Regret}_T(\Cref{alg:est_base}; \lambda) \leq \frac{(R_z+\tilde{R}_z)^2 (1+\log T)}{p} + \lambda pT ~.
\end{equation}
\end{theorem}

\begin{proof}[Proof of \Cref{thm:est_base}]
To prove the bound in the theorem, we construct the following proxy loss functions: if $b_t=1$ denote $\tilde{f}_t(\matz) = \frac{1}{2p} \| \matz - \matztil_t \|^2$, otherwise for $b_t=0$ denote $\tilde{f}_t(\matz) = 0$. The stochastic gradients of these functions at the current iterate can be written as $\nabla_{\matz} \tilde{f}_t(\matzhat_t) = \nabtil_t = \frac{\I\{b_t = 1\}}{p}(\matzhat_t - \matztil_t)$  and are used by the algorithm in the update rule. The idealized gradients are $\nabla_t = (\matzhat_t - \matzst_t)$ which we would use given access to the true targets $\matzst_t$. Recall that $\calF_t$ denotes the sigma-algebra generated by the true target sequence $\matz^{\star}_{1:T}$, as well as randomness of the past rounds $b_{1:t-1}$ and $\matztil_{1:t-1}$. Note then that  $\matzhat_t$ is $\calF_t$ measurable. We characterize two essential properties of the stochastic gradients: 
\begin{lemma}\label{lem:grad_ids} Let $\matzbar^{\star} \in \calK$ be the minimizer of $\sum_{t=1}^T \|\matz - \matzst_t\|^2$, i.e. empirical average of $\matzst_{1:T}$. Then,
\begin{equation*}
\Exp[ \langle \nabtil_t, \matzhat_t - \matzbar^{\star} \rangle ] = \Exp[\langle \nabla_t, \matzhat_t - \matzbar^{\star} \rangle] ~.
\end{equation*}
Moreover, $\Exp[\|\nabtil_t\|^2] \le (R_z+\tilde{R}_z)^2/p$.
\end{lemma}
\begin{proof}
Using the \Cref{orac:noisy} assumption on $\matztil_t$, we get
\begin{align*}
\Exp[\nabtil_t \mid \calF_t] &=\frac{1}{p}\Exp[ \I\{b_t = 1\} (\matzhat_t - \matztil_t) \mid \calF_t]\\
&=  \frac{1}{p}\Exp[ \I\{b_t = 1\} \cdot \matzhat_t  \mid \calF_t] - \frac{1}{p}\Exp[ \I\{b_t = 1\} \cdot \matztil_t \mid \calF_t] \\
&= \matzhat_t - \Exp[ \matztil_t \mid \calF_t, b_t = 1] \overset{(i)}{=} \matzhat_t - \matzstt  = \nabla_t,
\end{align*}
where $(i)$ uses the unbiasedness property of \Cref{orac:noisy}. 
Next, since $\matzbar^{\star}$ is determined by $\matz^{\star}_{1:T}$ it is therefore $\calF_t$ measurable for all $t$. Thus, $\matzhat_t - \matzbar^{\star}$ is $\calF_t$ measurable, so
\begin{align*}
\Exp[ \langle\nabtil_t, \matzhat_t - \matzbar^{\star} \rangle] &= \Exp\left[ \langle\Exp[\nabtil_t \mid \calF_t], \matzhat_t - \matzbar^{\star} \rangle \right] = \Exp[\langle \nabla_t, \matzhat_t - \matzbar^{\star} \rangle] ~.
\end{align*}
Finally, using the norm bound $\| \matzhat_t \| \leq R_z$ since $\matzhat_t \in \calK$ and the assumption that $\| \matztil_t \| \leq \tilde{R}_z$ from \Cref{orac:noisy}, we conclude
\begin{equation*}
\Exp[\|\nabtil_t\|^2] = \frac{1}{p^2}\Exp[\I\{b_t = p\} \|\matztil_t - \matzhat_t\|^2] \leq \frac{1}{p^2}\Exp[\I\{b_t = p\} (R_z+\tilde{R}_z)^2] = \frac{(R_z+\tilde{R}_z)^2}{p} ~.
\end{equation*}
\end{proof}

The rest of the theorem proof mirrors that of Theorem 3.3 in \cite{hazan2019introduction} but accounting for the stochastic gradient. We can view \Cref{alg:est_base} as running online stochastic gradient descent over strongly convex functions on losses  $\frac12 f_t(\matz) =  \frac{1}{2} \| \matz - \matzst_t \|^2$ with true gradient $\nabla_t$ and stochastic gradient $\nabtil_t$ at the iterate $\matzhat_t$. Since the losses $\frac12 f_t$ are $1$-strongly convex, $\Exp[b_t] = p$ and using the claim from \Cref{lem:grad_ids} we get,

\begin{align*}
\frac12 \mathrm{Regret}_T =  \frac{1}{2} \Exp\left[\sum_{t=1}^T (f_t(\matzhat_t) -  f_t(\matzbar^{\star}) + \lambda \cdot b_t) \right]
&\leq \Exp\left[\sum_{t=1}^T (\langle \nabla_t, \matzhat_t - \matzbar^{\star} \rangle - \frac12 \|\matzhat_t - \matzbar^{\star}\|^2) \right] + \frac12 \lambda p T \\
&= \frac12 \Exp\left[\sum_{t=1}^T ( 2 \langle \nabtil_t, \matzhat_t - \matzbar^{\star} \rangle - \|\matzhat_t - \matzbar^{\star}\|^2) \right] + \frac12 \lambda p T.
\end{align*}
The update rule is given as $\matzhat_{t+1} = \Proj_{\calK}(\matzhat_{t} - \eta_{t}\nabtil_t)$, so from the Pythagorean theorem for the projection
\begin{align*}
\|\matzhat_{t+1} - \matzbar^{\star}\|^2 &\leq \|\matzhat_{t} - \eta_{t} \nabtil_t  - \matzbar^{\star}\|^2 = \|\matzhat_{t} - \matzbar^{\star}\|^2 + \eta_{t}^2\|\nabtil_t\|^2 - 2 \eta_{t} \langle \nabtil_t, \matzhat_{t} - \matzbar^{\star}\rangle.\\
2\langle \nabtil_t, \matzhat_t - \matzbar^{\star} \rangle &\leq \frac{\|\matzhat_{t} - \matzbar^{\star}\|^2 - \|\matzhat_{t+1} - \matzbar^{\star}\|^2}{\eta_t} + \eta_t \|\nabtil_t\|^2.
\end{align*}
Combining the above bounds results in
\begin{equation*}
\frac12 \mathrm{Regret}_T \leq \frac{1}{2}\Exp\left[\sum_{t=1}^T \left( \frac{\|\matzhat_{t} - \matzbar^{\star}\|^2 - \|\matzhat_{t+1} - \matzbar^{\star}\|^2}{\eta_t} - \|\matzhat_t - \matzbar^{\star}\|^2 \right) + \sum_{t=1}^T \eta_t \|\nabtil_t\|^2  \right] + \frac12 \lambda p T.
\end{equation*}
The telescoping sum inside the parentheses is equal to $0$, the gradient term is bounded $\Exp[\|\nabtil_t\|^2] \leq \frac{(R_z+\tilde{R}_z)^2}{p}$ according to \Cref{lem:grad_ids} and the stepsize sum is bounded by $\sum_{t=1}^T \eta_t \leq 1 + \log T$, yielding the final result
\begin{equation*}
\mathrm{Regret}_T(\Cref{alg:est_base}; \lambda) \leq \frac{(R_z+\tilde{R}_z)^2 (1+ \log T)}{p} + \lambda p T.
\end{equation*}
\end{proof}



\subsection{Adaptive Regret Bound}\label{sec:app:est_adaptive}
The guarantee in \Cref{thm:est_base} ensures that the predicted sequence $\matzhat_{1:T}$ performs comparably to the empirical mean $\matzbar^{\star}$ of the entire target sequence $\matzst_{1:T}$. However, that doesn't imply much about the performance of \Cref{alg:est_base} on a given local interval $I\subseteq [T]$ since $\matzbar^{\star}_I$ can be very different from $\matzbar^{\star}$. Hence, we would like to extend our results to hold for any interval $I$, i.e. derive adaptive regret results as introduced in \cite{hazan2009efficient}. To do so we will use the approach of \cite{hazan2009efficient} using \Cref{alg:est_base} as a subroutine. The resulting algorithm, presented in \Cref{alg:adaptive_est}, suffers only a logarithmic computational overhead over \Cref{alg:est_base} with its performance guarantee stated in the theorem below.
\begin{algorithm}
\caption{Adaptive Estimation Algorithm}\label{alg:adaptive_est}
\begin{algorithmic}[1]
    \State \textbf{Input:} parameter $p$, decision set $\calK$, base estimation algorithm $\calA$, $\matzhat_1 \in \calK$
    \State \textbf{Initialize:}  $\calA_1 \leftarrow \calA(p, \matzhat_1)$, working set $\calS_1 = \{1\}$, $q_1^{(1)} = 1$, parameter $\alpha = \frac{p}{(R_z+\tilde{R}_z)^2}$
    \For{$t=1, \ldots, T$}
    \State{}\textbf{Compute} predictions $\matzhat_t^{(i)} \leftarrow \mathcal{A}_i$ for $i \in \calS_t$
    \State{}\textbf{Play} iterate $\matzhat_t = \sum_{i \in \calS_t} q_t^{(i)} \matzhat_t^{(i)}$
    \State{}\textbf{Draw/Receive} $b_t \sim \Bernoulli(p)$
    \If{$b_{t} = 1$}
    \State{}\textbf{Request} estimate $\matztil_t$ from \Cref{orac:noisy}
    \State{}\textbf{Let} $\tilde{\ell}_t(\matz) = \frac{1}{2p}||\matz - \matztil_t||^2$
    \Else
    \State{}\textbf{Let} $\matztil_t \leftarrow \emptyset$ and $\tilde{\ell}_t(\matz) = 0$
    \EndIf
    \State{}\textbf{Update} expert algorithms $\mathcal{A}_i(b_t, \matztil_t)$ for all $i \in \calS_{t}$
    \State{}\textbf{Form} new set $\tilde{\calS}_{t+1} = (i)_{i \in \calS_t}$
    \State{}\textbf{Construct} proxy new weights $\bar{q}_{t+1}^{(i)} =  \tfrac{t}{t+1} \cdot \tfrac{q_t^{(i)} e^{-\alpha \tilde{\ell}_t(\matzhat_t^{(i)})}}{\sum_{j\in \calS_t} q_t^{(j)} e^{-\alpha \tilde{\ell}_t(\matzhat_t^{(j)})}}$ for all $i\in \calS_t$ \label{line:weight_update_main_app}
    \State{}\textbf{Add} new instance $\tilde{\calS}_{t+1} \leftarrow \tilde{\calS}_{t+1} \cup {t+1}$ for arbitrary $\calA_{t+1} \leftarrow \calA(p, \matzhat_{1}^{(t+1)} = \matzhat_1)$ with $\bar{q}_{t+1}^{(t+1)} = \frac{1}{t+1}$
    \State{}\textbf{Prune}  $\tilde{\calS}_{t+1}$ to form $\calS_{t+1}$
    \State{}\textbf{Normalize} $q_{t+1}^{(i)} = \frac{\bar{q}_{t+1}^{(i)}}{\sum_{j \in \calS_{t+1}}\bar{q}_{t+1}^{(j)}}$
    \EndFor
\end{algorithmic}
\end{algorithm}

    

\begin{theorem}\label{thm:adaptive_est} Taking the base estimation algorithm $\mathcal{A}$ to be \Cref{alg:est_base} and given access to queries from \Cref{orac:noisy}, \Cref{alg:adaptive_est} enjoys the following guarantee:
\begin{equation}\label{eq:est_adaptive_bound}
\forall I = [r, s] \subseteq [T], \quad \mathrm{Regret}_I(\Cref{alg:adaptive_est}; \lambda) \leq \frac{2 (R_z+\tilde{R}_z)^2 (1 + \log{s} \cdot \log |I|)}{p} + \lambda p |I| ~.
\end{equation}
\end{theorem}

\begin{corollary}\label{cor:est_error}
The estimation error over each interval $I = [r, s] \subseteq [T]$ is bounded as follows,
\begin{equation*}
    \Exp \left[ \sum_{t \in I} \| \matzhat_t - \matzst_t \|^2 \right] \leq \var_I(\matzst_{1:T}) + \frac{2 (R_z+\tilde{R}_z)^2 (1 + \log{s} \cdot \log |I|)}{p} ~.
\end{equation*}
\end{corollary}

\begin{proof}[Proof of \Cref{thm:adaptive_est}]

First observe that 
$\tilde{\ell}_t$ is $\alpha$-exp concave with $\alpha = \frac{p}{(R_z+\tilde{R}_z^2)}$. This is evident given its construction: $\tilde{\ell}_t(\matz) = \frac{1}{2p} \|\matz - \matztil_t\|^2$ with $\|\matz\| \leq R_z$ since $\matz \in \calK$ and $\| \matztil_t \| \leq \tilde{R}_z$ according to \Cref{orac:noisy}. The rest of the algorithm uses the approach of \cite{hazan2009efficient}, in particular Algorithm 1, over exp concave functions to derive the guarantee in the theorem statement. 

We note that Claim 3.1 in \cite{hazan2009efficient} holds identically in our case, i.e. for any $I = [r, s]$ the regret of \Cref{alg:adaptive_est} with respect to $\calA_r$ is bounded by $\frac{2}{\alpha} (\ln r + \ln |I|)$ if $\calA_r$ stays in the working set. We combine this fact with the bound given in \Cref{thm:est_base} to get that \Cref{alg:adaptive_est} enjoys regret $\frac{3}{\alpha} (\log r + \log |I|)$ over $I = [r, s]$ if $\calA_r$ stays in the working set $\calS_t$ throughout $I$. Finally, an induction argument along with the working set properties detailed in \Cref{sec:working} identical to that of Lemma 3.2 in \cite{hazan2009efficient} yields the desired result for $\tilde{\ell}_t$. Notice that this is our desired result in expectation,

\begin{observation}\label{obs:eq_in_exp} We have the following identity for any $t$ and $r\leq t$:
$$\E\left[\tilde{\ell}_t(\matzhat_t) - \tilde{\ell}_t(\matzhat_t^{(r)})\right] = \E\left[ {\ell}_t(\matzhat_t) - {\ell}_t(\matzhat_t^{(r)})\right]$$
\end{observation}

\begin{proof}[Proof of \Cref{obs:eq_in_exp}]
We can expand:
\begin{align*}
2\E\left[\tilde{\ell}_t(\matzhat_t) - \tilde{\ell}_t(\matzhat_t^{(r)})\right] &= 2 p\cdot\E\left[\tilde{\ell}_t(\matzhat_t) - \tilde{\ell}_t(\matzhat_t^{(r)}) | b_t = 1\right]
+ 2 (1-p)\cdot\E\left[\tilde{\ell}_t(\matzhat_t) - \tilde{\ell}_t(\matzhat_t^{(r)}) | b_t = 0\right]
\\&= p \cdot \E\left[\dfrac{1}{p} \cdot \left(\norm{\matzhat_t}^2 + \langle \matzhat_t, \matztil_t \rangle + \norm{\matztil_t}^2 - \norm{\matzhat_t^{(r)}}^2 -  \langle \matzhat_t^{(r)}, \matztil_t \rangle -  \norm{\matztil_t}^2\right)\right] + 0 \\
&= \E\left[\norm{\matzhat_t}^2  - \norm{\matzhat_t^{(r)}}^2\right] + \E\left[\langle \matzhat_t - \matzhat_t^{(r)}, \matztil_t\rangle \right]
\end{align*}
By the linearity of expectation, the fact that $\matzhat_t, \matzhat_t^{(r)}$ are completely determined given $\mathcal{F}_{t-1}$, and the law of total expectation we have that 
\begin{align*}
\E\left[\langle \matzhat_t - \matzhat_t^{(r)}, \matztil_t\rangle \right] &= \E\left[\E\left[\langle \matzhat_t - \matzhat_t^{(r)}, \matztil_t\rangle|\mathcal{F}_{t-1}\right] \right] \\
&= \E\left[\langle \matzhat_t - \matzhat_t^{(r)}, \matzstt\rangle \right]
\end{align*}
Plugging this in above we, adding and subtracting $\norm{\matzstt}^2$, and rearranging we have:
\begin{align*}
2\E\left[\tilde{\ell}_t(\matzhat_t) - \tilde{\ell}_t(\matzhat_t^{(r)})\right] 
&= \E\left[\norm{\matzhat_t}^2  - \norm{\matzhat_t^{(r)}}^2 + \langle \matzhat_t - \matzhat_t^{(r)}, \matzstt\rangle + \norm{\matz}^2 - \norm{\matz}^2 \right] \\
&= \E\left[\norm{\matzhat_t}^2  +  \langle \matzhat_t, \matzstt\rangle + \norm{\matz}^2 - \norm{\matzhat_t^{(r)}}^2 - \langle \matzhat_t^{(r)}, \matzstt\rangle  - \norm{\matz}^2 \right] \\
&=  2\E\left[{\ell}_t(\matzhat_t) - {\ell}_t(\matzhat_t^{(r)})\right] 
\end{align*}
as desired.
\end{proof}

Combining \Cref{obs:eq_in_exp} with the fact that $\tilde{\ell}_t$ are $\alpha$-exp concave for $\alpha = \frac{p}{(R_z+\tilde{R}_z)^2}$ we conclude the final statement of \Cref{thm:adaptive_est}.
\end{proof}




\subsubsection{Working Set Construction}\label{sec:working}
Our \Cref{alg:adaptive_est} makes use of the working sets $\{\calS_t\}_{t \in [T]}$ along with its properties in Claim \ref{claim:workingsets}. In this section, we show the explicit construction of these working sets as in \cite{hazan2009efficient} and prove the claim.

\begin{claim}\label{claim:workingsets}
    The following properties hold for the working sets $S_t$ for all $t \in [T]$: (i) $|\calS_t| = O(\log T)$; (ii) $[s, (s+t)/2] \cap \calS_t \neq \emptyset$ for any $s \in [t]$; (iii) $\calS_{t+1} \backslash \calS_t = \{t+1\}$; (iv) $|\calS_t \backslash \calS_{t+1}| \leq 1$.
\end{claim}

For any $i \in [T]$, let it be given as $i = r 2^k$ with $r$ odd and $k$ nonnegative. Denote $m = 2^{k+2}+1$, then $i \in S_t$ if and only if $t \in [i, i+m]$. This fully describes the construction of the working sets $\{S_t\}_{t \in [T]}$, and we proceed to prove its properties.

\begin{proof}[Proof of Claim \ref{claim:workingsets}]
    For all $t \in [T]$ we show the following properties of the working sets $S_t$. 
    
    (i) $|S_t| = O(\log T)$: if $i \in S_t$ then $1 \leq i = r 2^k \leq t$ which implies that $0 \leq k \leq \log_2{t}$. For each fixed $k$ in this range, if $r 2^k = i \in S_t$ then $i \in [t-2^{k+2}-1, t]$ by construction. Since $[t-2^{k+2}-1, t]$ is an interval of length $2^{k+2}+2 = 4 \cdot 2^k + 2$, it can include at most $3$ numbers of the form $r 2^k$ with $r$ odd. Thus, there is at most $3$ numbers $i = r 2^k \in S_t$ for each $0 \leq k \leq \log_2{t}$ which means that $|S_t| = O(\log t) = O(\log T)$.
    
    (ii) $[s, (s+t)/2] \cap S_t \neq \emptyset$ for all $s \in [t]$: this trivially holds for $s=t-1, t$. Let $2^l \leq (t-s)/2$ be the largest such exponent of $2$. Since the size of the interval $[s, (s+t)/2]$ is $\lfloor (t-s)/2 \rfloor$, then there exists $u \in [s, (s+t)/2]$ that divides $2^l$. This means that the corresponding $m \geq 2^{l+2}+1 > t-s$ for $u \geq s$ is large enough so that $t \in [u, u+m]$, and consequently, $u \in S_t$.
    
    (iii) $S_{t+1} \backslash S_t = \{t+1\}$: let $i \in S_{t+1}$ and $i \not \in S_t$, which is equivalent to $t+1 \in [i, i+m]$ and $t \not \in [i, i+m]$. Clearly, $i = t+1$ satisfies these conditions and is the only such number.
    
    (iv) $|S_t \backslash S_{t+1}| \leq 1$: suppose there exist two $i_1, i_2 \in S_t \backslash S_{t+1}$. This implies that $i_1 + m_1 = t = i_2 + m_2$ which in turn means $2^{k_1} (r_1+4) = 2^{k_2} (r_2+4)$. Since both $r_1+4, r_2+4$ are odd, then $k_1 = k_2$, and consequently, $r_1=r_2$ resulting in $i_1=i_2$. Thus, there can not exist two different members of $S_t \backslash S_{t+1}$ which concludes that $|S_t \backslash S_{t+1}| \leq 1$.
\end{proof}

\subsection{No Strong Adaptivity}\label{sec:app:lower_strongly}
Notice that even though the guarantee of \Cref{thm:adaptive_est} applies to all intervals $I$, it does not entail meaningful guarantees for all. The reason is the choice of parameter $p$: if one wishes to optimize $\mathrm{Regret}_T$ then $p = \mathcal{O}(T^{-1/2})$ implies $\mathcal{O}(\sqrt{T})$ regret, but this choice is meaningless for intervals with length $|I| << \sqrt{T}$; on the other hand, optimizing the bound for small intervals leads to large bounds for the entire horizon. One might then ask whether there exist methods with \emph{strongly adaptive} guarantees, and we answer this question with a negative.
\begin{theorem}\label{thm:no_strong_ada}
    For any $\gamma > 0$ and oracle cost $\lambda > 0$, there exists no online algorithm $\calA$ with feedback access to \Cref{orac:noisy} that enjoys the following strongly adaptive regret guarantee: $\mathrm{Regret}_I(\calA; \lambda) = \tilde{\mathcal{O}}(|I|^{1-\gamma})$.
\end{theorem}
\begin{proof}
The proof of this impossibility results follows a simple construction: the idea behind it is that strongly adaptive guarantees imply both large and small amount of exploration. Let us suppose there exists such an algorithm $\calA$ and arrive at a contradiction: $\forall I = [r, s] \subseteq [T]$ algorithm $\calA$ has a regret bound $\mathrm{Regret}_I(\calA; \lambda) \leq C \cdot |I|^{1-\gamma} = \tilde{\mathcal{O}}(|I|^{1-\gamma})$ over any oblivious sequence $\matzst_{1:T}$ where $C$ depends on problem parameters, $\lambda$ and $\log T$.

Construct the following oblivious sequence: let $k = T^{1-\gamma/2}$ and $I_1, \dots, I_k$ be consecutive disjoint intervals such that $\cup_{j \in [k]} I_j = [T]$, $I_j \cap I_l = \emptyset$ for all $j \neq l$, and $|I_j| = T/k = T^{\gamma/2}$ for all $j \in [k]$ (w.l.o.g. we assume $T$ divides $k$). Now for each interval $I_j$, $j \in [k]$, sample a fresh $q_j \in \{\pm 1\} \sim Rad(1/2)$ and let $\matzst_t = q_j$ for all $t \in I_j$.

According to the assumed guarantee, the overall regret is bounded as $\mathrm{Regret}_T(\calA; \lambda) leq C \cdot T^{1-\gamma}$ which by definition implies that $\sum_{t=1}^T b_t \leq \frac{C}{\lambda} T^{1-\gamma} < k$ where the last inequality is true for sufficiently large horizon $T$. Since there are $k$ consecutive disjoint intervals $I_1, \dots, I_k$ and less than $k$ overall calls to \Cref{orac:noisy}, there exists an interval $I \in \{ I_1, \dots, I_k\}$ such that $\sum_{t \in I} b_t = 0$.

On the other hand, the assumed guarantee for $\calA$ implies that the interval $I$ of size $|I| = T^{\gamma/2}$ enjoys sublinear regret, i.e. $\mathrm{Regret}_I(\calA; \lambda) = o(|I|)$. We show that this is a contradiction given that $\sum_{t \in I} b_t = 0$. As there were no oracle calls for the interval $I$, the predictions of $\calA$, $\matzhat_t$ over $t \in I$, are independent from the Rademacher sample of the interval $q_I$: this is true since the samples for each interval in $I_1, \dots, I_k$ are independent. Therefore, $\matzhat_t \perp q_I$ for all $t \in I$ which means that since the best loss in hindsight over $I$ is equal to $0$ as $\matzst_t = q_I$ for all $t\in I$,
\begin{equation*}
    \mathrm{Regret}_I(\calA; \lambda) \geq \Exp_{q_I} \left[ \sum_{t \in I} \ell_t(\matzhat_t) \right] = \sum_{t \in I} \Exp_{q_I} [  \| \matzhat_t - q_I \|^2 ] = \Omega(|I|) ~. 
\end{equation*}
Hence, for the interval $I$, the regret of $\calA$ cannot be sublinear, which contradict the assumption that $\calA$ exhibits strongly adaptive guarantees. This concludes that no strongly adaptive online algorithm exists in the described partial information model.
\end{proof}

\section{Adaptive Regret for Control of Changing Unknown Dynamics}\label{sec:ctrl_red}

In this section we give our full control algorithm which attains sublinear regret with respect to $\Pi_\mathrm{drc}$ up to an additive system variability term. A key component is the system estimation for which we will use \Cref{alg:adaptive_est} and its guarantees from \Cref{sec:sys_est}. More specifically, our algorithm is based on the canonical explore-exploit approach: it explores with some probability $p$ by inputting random controls into the system, and otherwise outputs a control according to DRC-OGD (\Cref{alg:DRC-OGD}). Note that due to the long-term consequences which appear in control, we need to explore for $h$ consecutive steps in order to get an estimate for the $h$-truncation of the Markov operator. Hence, our algorithm will determine whether it explores or exploits in blocks of length $h$. Furthermore, we will define the set of Markov operator of length $h$ and $\ell_1,op$-norm bounded by $R_G$ as $\mathcal{G}(h, R_G)$:

$$\mathcal{G}(h, R_G) \doteq \{G = (G^{[0]}, \ldots, G^{[h-1]}) \in \mathbb{R}^{h\times d_x\times d_u} \text{ s.t. } \sum_{i=0}^{h-1} ||G||_{op} \leq R_G\}$$

\begin{remark}\label{remark:truncated_G_set_euc_radius}
Note that the radius of $\mathcal{G}(h, R_G)$ is bounded by $\bar{R}_G = \sqrt{h \cdot d_{\mathrm{min}}} R_G$ where $d_{\mathrm{min}} = \min\{d_x, d_u\}$.
\end{remark}
\begin{proof}[Proof of \Cref{remark:truncated_G_set_euc_radius}]
$\forall G \in \mathcal{G}(h, R_G) $, we have:
\begin{align*}
||G||_F &= \sqrt{\sum_{i=0}^{h-1} ||G^{[i]}||^2_F} \\
&\leq \sqrt{h} \sqrt{\max_i ||G^{[i]}||_F^2} \\
&\leq \sqrt{h\min\{d_x, d_u\}} \max_i ||G^{[i]}||_{op} \\
&\leq \sqrt{h\min\{d_x, d_u\}} R_G
\end{align*}
and denoting $d_{\mathrm{min}} = \min\{d_x, d_u\}$ yields the result.
\end{proof}

\begin{remark}
By abuse of notation, we will consider $G^{[i]} \doteq 0$ for $G \in \mathcal{G}(h, R_G)$.
\end{remark}

\begin{remark}
For simplicity, we assume $T$ divisible by $h$ (this is w.l.o.g. up to an extra $O(h)$ cost which for us is negligble).
\end{remark}

We spell out the full procedure in \Cref{alg:control_unknown} and give its guarantee below in \Cref{thm:control_drc_full}.

\begin{algorithm}
\caption{DRC-OGD with Exploration}\label{alg:control_unknown}
\begin{algorithmic}[1]
    \State \textbf{Input:} $p, h, \hat{G}_0, \mathcal{A} \leftarrow \Cref{alg:adaptive_est}(p, \hat{G}_0, \mathcal{G}(h, R_G))$, $\mathcal{C} \leftarrow \Cref{alg:DRC-OGD}(\eta, m, R_\calM)$
    \For{$\tau_1 = 0, \ldots, T/h - 1$}
    \State Request $\hat{G}_{\tau_1\cdot h + 1}  \leftarrow \mathcal{A}$ and set $\hat{G}_{\tau_1 \cdot h + 2}, \ldots, \hat{G}_{(\tau_1 +1 ) \cdot h} \leftarrow \hat{G}_{\tau_1\cdot h + 1}$\label{line:setting_G_hat}
    \State Draw $b_{\tau_1 + 1} \sim \text{Bernoulli}(p)$
    \For{$\tau_2 = 1, \ldots,h$}
    \Comment{let $t\doteq \tau_1 \cdot h + \tau_2$}
    \If{$b_{\tau_1} = 1$}
    \State Play control $u_{t} \sim \{\pm 1\}^{d_u}$
    \Else 
    \State Play control $u_{t} \leftarrow \mathcal{C}$
    \EndIf
    \State Suffer cost $c_t(x_{t}, u_{t})$ , observe new state $x_{t+1}$ 
    \State Extract $\hat{x}_{t+1}^{\mathrm{nat}} = \Proj_{\mathbb{B}_{R_{\mathrm{nat}}}}\left[x_{t+1} - \sum_{i=0}^{h-1} \hat{G}_t^{[i]} u_{t-i}\right]$
    \State Update $\mathcal{C} \leftarrow (c_t, \hat{G}_t, \hat{x}^{\mathrm{nat}}_{t+1})$
    \EndFor
    \If{$b_{\tau_1} = 1$}
    \State Let $\tilde{G}_{(\tau_1 + 1)\cdot h}^{[i]} = x_{(\tau_1 +1)\cdot h + 1} u_{(\tau_1 + 1)\cdot h - i}^\top$ for $i=\overline{0, h-1}$
    \Else
    \State Let $\tilde{G}_{(\tau_1 + 1) \cdot h} \leftarrow \emptyset$
    \EndIf
    \State Update $\mathcal{A} \leftarrow (b_{\tau_1 + 1}, \tilde{G}_{(\tau_1+1)\cdot h})$
    \EndFor
\end{algorithmic}
\end{algorithm}

\begin{theorem}\label{thm:control_drc_full}
For $h = \dfrac{\log{T}}{\log{\rho^{-1}}}$, $p=T^{-1/3}$ and $m\leq \sqrt{T}$, on any contiguous interval $I \subseteq [T]$, \Cref{alg:control_unknown_main} enjoys the following adaptive regret guarantee\footnote{For precise constants please see \Cref{eq:full_constant_thm}.}: 
\begin{align*} \E\left[\Regret_I(\adacontrol); \Pidrc(m, R_M)\right] \leq \widetilde{\mathcal{O}}^{\star}\left(Lm\left(|I| \sqrt{ \Exp[\var_{I}(\Gseq)]} + d_u T^{2/3}\right)  \right) 
\end{align*}
\end{theorem}

The proof of this theorem will proceed in terms of a quantity which we call \emph{total system variability} which captures the total (rather than average) deviation from the mean operator for each interval. More precisely,

\begin{definition}\label{def:tot_variability}
Define the \emph{total} system variability of an LTV dynamical system with Markov operators $\Gseq = G_{1:T}$ over a contiguous interval $I \subseteq [T]$ to be
\begin{equation*}
    \var^{\mathrm{tot}}_I(\Gseq)= \min_{G} \sum_{t \in I} \| G - G_t \|_{\ell_2, F}^2 = \sum_{t\in I} \|G_I - G_t\|_{\ell_2, F}^2,
\end{equation*}
where $\| \cdot \|_{\ell_2, F}$ indicates the $\ell_2$ norm of the fully vectorized operator and $G_I = |I|^{-1} \sum_{t\in I} G_t$ is the empirical average of the operators that correspond to $I$.
\end{definition}

\subsection{Estimation of the Markov Operator}\label{sec:markov_est}

\newcommand{\Gtil}{\tilde{G}}

Note that the estimation component of \Cref{alg:control_unknown} directly operates in the setting of \Cref{sec:sys_est} and effectively solves the problem of adaptively estimating the sequence $\Gbar_{1 \cdot h}, \ldots, \Gbar_{(T/h) \cdot h}$ of the $h$-truncations of the true Markov operators $G_{1 \cdot h}, \ldots, G_{(T/h) \cdot h}$. To formally be able to apply \Cref{thm:adaptive_est}, we first show that the estimates sent to \Cref{alg:adaptive_est} satisfy the properties of \Cref{orac:noisy}.

\begin{claim}
The estimators $\Gtil_{\tau_1 \cdot h}, \; \tau_1=\overline{1, T/h}$ satisfy the properties of \Cref{orac:noisy} with $\tilde{R}_{G} = \sqrt{h\cdot d_u}(\radnat + R_G \max\{\sqrt{d_u}, R_{\mathrm{nat}} R_{\mathcal{M}}\})$.
\end{claim}

\begin{proof} There are only two things to prove:
\begin{enumerate}
    \item \textbf{Boundedness.} Because we clip the nature's x estimates $\hat{x}^{\mathrm{nat}}_t$ to $R_{\mathrm{nat}}$ and when the DAC policy lies in $\mathcal{M}(m, R_\mathcal{M})$, we have that if $b_t = 0$, $||u_t|| \leq R_{\mathrm{nat}} R_{\mathcal{M}}$. If $u_t$ is an exploratory action then $||u_t|| = \sqrt{d_u} \leq \max\{\sqrt{d_u}, R_{\mathrm{nat}} R_{\mathcal{M}}\}$ by design. By the equation of the progression of the state, we have that
    \begin{equation} \label{eq:x_bdd}
    \|x_t\| \leq \radnat + R_G \max\{\sqrt{d_u}, R_{\mathrm{nat}} R_{\mathcal{M}}\}
    \end{equation}
    and by Cauchy Scwarz we get $$||\Gtil^{[i]}_{\tau_1 \cdot h}||_F = ||u^\top_{\tau_1 \cdot h -i} x_{\tau_1\cdot h + 1}|| \leq \sqrt{d_u}(\radnat + R_G \max\{\sqrt{d_u}, R_{\mathrm{nat}} R_{\mathcal{M}}\})$$
    and hence $$||\Gtil_{\tau_1 \cdot h}||_F \leq \sqrt{h\cdot d_u}(\radnat + R_G \max\{\sqrt{d_u}, R_{\mathrm{nat}} R_{\mathcal{M}}\})$$
    \item \textbf{Unbiasedness.} Plugging in $x_{\tau_1 \cdot h} = \xnat_{\tau_1 \cdot h} + \sum_{i=0}^{t} G_{\tau_1 \cdot h}^{[i]} u_{\tau_1 \cdot h - i}$, we get exactly that $\mathbb{E}[\tilde{G}_{\tau_1 \cdot h}^{[i]}] = G_{\tau_1 \cdot h}^{[i]}$. Since this holds for the selected truncation $h$, we have $\mathbb{E}[\tilde{G}_{\tau_1 \cdot h}] = \bar{G}_{\tau_1 \cdot h}$ for $\Gbar$ as defined.
\end{enumerate}
\end{proof}

Hence we can simply apply the guarantees of \Cref{sec:sys_est} to obtain the following guarantee:

\begin{corollary}\label{cor:sys_est_control_primitive}
On any interval $J = [k, l]\subseteq[T/h]$, we have that:
$$\mathbb{E}\left[\sum_{\tau=k}^l ||\Gtil_{\tau\cdot h} - \Gbar_{\tau\cdot h}||_F^2\right] \leq \var^{\mathrm{tot}}_{J\cdot h}({\Gbar_{1:T}}) + \mathrm{Regret}_J(\mathcal{A}(p, \bar{R}_G, \tilde{R}_G; 0)) $$
where we use $J\cdot h$ to denote the set $[k\cdot h, \ldots, l\cdot h]$.
\end{corollary}

However, to properly analyze the additional regret introduced in the control framework by our estimation error, we need to convert \Cref{cor:sys_est_control_primitive} into a guarantee in terms of $\ell_1, op$ norm which holds for any contiguous interval $I=[r,s]\subseteq [T]$. This step is rather straightforward and only relies on a few basic properties/observations which we collect in \Cref{obs:reduction_props} below.

\begin{observation}\label{obs:reduction_props} We will make the following observations:
\begin{enumerate}
\item  $||A||_{op} \leq ||A||_F$ for any matrix $A$,
\item $\var^{\mathrm{tot}}_J(\matz_{1:T})\leq \var^{\mathrm{tot}}_I(\matz_{1:T})$ for any set of indices $J \subseteq I$, and any sequence $\matz_{1:T}$,
\item $||\hat{G}_t - \bar{G}_t||_F^2 = \sum\limits_{i=1}^h ||\hat{G}_t^{[i]} - \bar{G}_t^{[i]}||_F^2$,
\item $||G_t - \bar{G}_t||_{\ell_1, op} \leq \psi(h)$,
\item For any $I = [r,s]$, $\sum_{t=r+1}^s ||\bar{G}_t - \bar{G}_{t-1}||^2_F \leq 4 \var^{\mathrm{tot}}_I(\Gbar_{1:T}) \leq 4 \var^{\mathrm{tot}}_I(\Gseq)$.
\end{enumerate}
\end{observation}

\begin{proof}
Properties $1$-$4$ follow from the definitions of the relevant quantities, or are general well-known facts. For $5$, by the triangle inequality, we have that, for any $\bar{G}$,
$$||\bar{G}_t - \bar{G}_{t-1}||_F \leq ||\bar{G}_t - \bar{G}||_F +||\bar{G}_{t-1} -\bar{G}||_F$$
$$\Rightarrow ||\bar{G}_t - \bar{G}_{t-1}||_F^2 \leq 2 (||\bar{G}_t - \bar{G}||_F^2 +||\bar{G}_{t-1} -\bar{G}||_F^2)$$
summing the above from $r+1$ to $s$ and taking $\bar{G}$ to be the sample mean over $I$ yields the desired result. The second inequality simply follows by the fact that truncation can only decrease variance.
\end{proof}

As a first step, we first bound the Frobenius norm error of the truncated operators over an arbitrary contiguous interval $I=[r,s]\subseteq[T]$.

\begin{lemma}\label{lem:f_error} On any interval $I=[r,s]\subseteq[T]$ with $r-1, s$ divisible by $h$\footnote{This is w.l.o.g. and only assumed for simplicity of presentation.}, we can bound the Frobenius estimation error of the truncated operators as:
\[\E\left[\sum_{t=r}^s ||\hat{G}_t - \bar{G}_t||^2_F\right] \leq 10 h^2 \var^{\mathrm{tot}}_I(\Gseq) + 2h\mathrm{Regret}_I(\mathcal{A}(p, \bar{R}_G, \tilde{R}_G; 0))\]
\end{lemma}

\begin{proof} By \Cref{line:setting_G_hat}, we can write
\begin{align*}
\sum_{t=r}^s||\hat{G}_t - \bar{G}_t||_F^2 &= \sum_{\tau_1 = (r-1)/h}^{s/h - 1} \sum_{\tau_2 = 1}^{h}||\hat{G}_{\tau_1\cdot h + \tau_2} - \bar{G}_{\tau_1 \cdot h + \tau_2}||_F^2 \\
&= \sum_{\tau_1 = (r-1)/h}^{s/h - 1} \sum_{\tau_2 = 1}^{h}||\hat{G}_{(\tau_1 + 1)\cdot h} - \bar{G}_{\tau_1 \cdot h + \tau_2}||_F^2 \\
&\leq 2 \sum_{\tau_1 = (r-1)/h}^{s/h - 1} \sum_{\tau_2 = 1}^{h}\left(||\hat{G}_{(\tau_1 + 1)\cdot h} - \bar{G}_{(\tau_1 + 1)\cdot h}||_F^2 +||\bar{G}_{(\tau_1 + 1)\cdot h} - \bar{G}_{\tau_1 \cdot h + \tau_2}||_F^2\right) \\
&= 2 h \errtext{\sum_{\tau = r/h}^{s/h} ||\hat{G}_{\tau \cdot h} - \bar{G}_{\tau \cdot h}||_F^2}{\Cref{alg:adaptive_est} estimation error} +  2 \errtext{\sum_{\tau_1 = (r-1)/h}^{s/h - 1} \sum_{\tau_2 = 1}^{h} ||\bar{G}_{(\tau_1 + 1)\cdot h} - \bar{G}_{\tau_1 \cdot h + \tau_2}||_F^2}{$\bar{G}$ movement}
\end{align*}
For the first term, we will simply apply the above corollary after taking expectation. We will therefore focus on bounding the $\bar{G}$ movement term.

\begin{align*}
\forall \tau_2: \;\;\; ||\bar{G}_{(\tau_1 + 1)\cdot h} - \bar{G}_{\tau_1 \cdot h + \tau_2}||_F^2 &\leq \left(\sum_{i = 1}^{h-1} ||\bar{G}_{\tau_1 \cdot h + i + 1} - \bar{G}_{\tau_1 \cdot h + i}||_F\right)^2 & (\triangle\text{-ineq.}) \\
&\leq h \sum_{i = 1}^{h-1} ||\bar{G}_{\tau_1 \cdot h + i + 1} - \bar{G}_{\tau_1 \cdot h + i}||_F^2 & \text{(C.S.)}
\end{align*}
This implies that
\begin{align*}\sum_{\tau_2 = 1}^{h} ||\bar{G}_{(\tau_1 + 1)\cdot h} - \bar{G}_{\tau_1 \cdot h + \tau_2}||_F^2 \leq h^2 \sum_{i = 1}^{h-1} ||\bar{G}_{\tau_1 \cdot h + i + 1} - \bar{G}_{\tau_1 \cdot h + i}||_F^2 \end{align*}

So finally we have that

\begin{align*}
\bar{G}\text{-movement} &\leq h^2\sum_{t = r+1}^s ||\bar{G}_{t} - \bar{G}_{t}||_F^2 \\
&\leq 4 h^2 \var^{\mathrm{tot}}_I(\Gseq) & \Cref{obs:reduction_props} \;(5)
\end{align*}

Finally, taking expectation, plugging in \Cref{cor:sys_est_control_primitive} and noting that for $J= [r/h, s/h]$, $\var^{\mathrm{tot}}_{J\cdot h}({\Gbar_{1:T}}) \leq \var^{\mathrm{tot}}_I(\Gseq)$ (\Cref{obs:reduction_props} (2)) and $\mathrm{Regret}_J(\mathcal{A}(p, \bar{R}_G, \tilde{R}_G; 0)) \leq \mathrm{Regret}_I(\mathcal{A}(p, \bar{R}_G, \tilde{R}_G; 0))$, we get:

\begin{align*}
\E\left[\sum_{t=r}^s||\hat{G}_t - \bar{G}_t||_F^2\right] &\leq 10 h^2 \E[\var^{\mathrm{tot}}_I(\Gseq)] + 2h\mathrm{Regret}_I(\mathcal{A}(p, \bar{R}_G, \tilde{R}_G; 0))
\end{align*}
\end{proof}

\begin{lemma}\label{lem:l1_op_sq_err}
On any interval $I = [r, s]\subseteq [T]$ with $r-1, s$ divisible by $h$, we can bound the squared $\ell_1, op$ estimation error of the truncated operators as:
\[\E\left[\sum_{t=r}^s ||\hat{G}_t - \bar{G}_t||_{\ell_1, op}^2\right] \leq 20h^3 \E[\var^{\mathrm{tot}}_I(\Gseq)] + 4h^2 \mathrm{Regret}_I(\mathcal{A}(p, \bar{R}_G, \tilde{R}_G; 0))\]
\end{lemma}

\begin{proof} We have that
\begin{align*}
\sum_{t=r}^s ||\hat{G}_t - \bar{G}_t||_{\ell_1, op}^2 & \leq \sum_{t=r}^s \left(\sum_{i=0}^{h-1} ||\hat{G}_t^{[i]} - \bar{G}_t^{[i]}||_{op} \right)^2 & \triangle\text{-ineq.} \\
&\leq 2 \sum_{t=r}^s \left(\sum_{i=0}^{h-1} ||\hat{G}_t^{[i]} - \bar{G}_t^{[i]}||_{op}\right)^2\\
&\leq 2h \sum_{t=r}^s \sum_{i=0}^{h-1} ||\hat{G}_t^{[i]} - \bar{G}_t^{[i]}||_F^2 & \Cref{obs:reduction_props} \; (1) \text{ \& C.S.} \\
&= 2h \sum_{t=r}^s ||\hat{G}_t - \bar{G}_t||_F^2 &\Cref{obs:reduction_props} \; (3)
\end{align*}
Taking expectation and plugging in the bound in \Cref{lem:f_error} yields the promised result.
\end{proof}

Finally we can use \Cref{lem:l1_op_sq_err} and Cauchy-Schwarz to get a result in terms of the linear (rather than squared) $\ell_1, op$ error accumulated over an interval:

\begin{proposition}\label{prop:l1_op_err}
On any interval $I = [r, s]\subseteq [T]$ with $r-1, s$ divisible by $h$, we can bound the squared $\ell_1, op$ estimation error of the truncated operators as:
\[\E\left[\sum_{t=r}^s ||\hat{G}_t - \bar{G}_t||_{\ell_1, op}\right] \leq 5 h^{3/2} |I|^{1/2} \sqrt{\var^{\mathrm{tot}}_I(\Gseq)} + 2h |I|^{1/2} \mathrm{Regret}_I^{1/2}(\mathcal{A}(p, \bar{R}_G, \tilde{R}_G; 0)) \]
\end{proposition}

\begin{proof} By Cauchy-Schwarz and Jensen (since $\sqrt{\cdot}$ is concave) we have:
\begin{align*}
\E\left[\sum_{t=r}^s ||\hat{G}_t - \bar{G}_t||_{\ell_1, op}\right] &\leq |I|^{1/2}\E\left[\left(\sum_{t=r}^s ||\hat{G}_t - \bar{G}_t||_{\ell_1, op}^2\right)^{1/2}\right] \\
&\leq |I|^{1/2} \left(\E\left[\sum_{t=r}^s ||\hat{G}_t - \bar{G}_t||_{\ell_1, op}^2\right]\right)^{1/2} \\
&\leq |I|^{1/2} \sqrt{20h^3 \E[\var^{\mathrm{tot}}_I(\Gseq)] + 4h^2 \mathrm{Regret}_I(\mathcal{A}(p, \bar{R}_G, \tilde{R}_G; 0))} \\
&\leq 5 h^{3/2} |I|^{1/2} \sqrt{\E[\var^{\mathrm{tot}}_I(\Gseq)}] + 2h |I|^{1/2} \mathrm{Regret}_I^{1/2}(\mathcal{A}(p, \bar{R}_G, \tilde{R}_G; 0))
\end{align*}
\end{proof}

\subsection{Error Sensitivity}\label{sec:err_sens}

We now analyze concretely how the $G_t$ estimation errors induce additional regret over the case of \emph{known} systems. We can decompose the expected regret over an interval $I=[r,s]$ as:

	\begin{align*}
	\Regret_I &= \sum_{t=r}^s c_t(x_t, u_t) - \min_{\pi\in \Pi_{\mathrm{drc}}} \sum_{t=r}^s c_t(x_t^\pi, u_t^\pi)\\
	&= \errtext{\sum_{t=r}^s c_t(x_t, u_t) - c_t(\hat{x}_t(M_{t-h:t-1}), \hat{u}_t(M_{t}))}{realized iterate error}\\
	&\quad+ \errtext{\sum_{t=r}^s c_t(\hat{x}_t(M_{t-h:t-1}), \hat{u}_t(M_{t}) - \inf_{M \in \calM}\sum_{t=r}^s c_t(\hat{x}_t(M), \hat{u}_t(M))}{regret $:= \widehat{\Regret}_I$}\\
	&\quad+\errtext{\inf_{M \in \calM}\sum_{t=r}^s c_t(\hat{x}_t(M), \hat{u}_t(M)) - \inf_{M \in \calM}\sum_{t=r}^s c_t(x_t(M), u_t(M))}{comparator error}
	\end{align*}

First let us bound the realized iterate error which bounds the difference between what actually happened and what would have happened in the fictive $(\hat{G}_t, \hat{x}_t^{\mathrm{nat}})$ system (without exploration).

\begin{lemma}\label{lem:xnat_diff}
We can bound the difference between the true $\xnat_t$ and the extracted $\hat{x}_t^{\mathrm{nat}}$ as:

$$\|\xnat_t - \hat{x}_t^{\mathrm{nat}} \| \leq R_{\mathrm{nat}}R_\calM \left(\|\bar{G}_{t-1} - \hat{G}_{t-1}\|_{\ell_1, op} + \psi(h)\right)$$
\end{lemma}

\begin{proof} Since $\xnat \in \mathbb{B}_{R_{\mathrm{nat}}}$ and because (as argued earlier) $\|u_t\| \leq R_{\mathrm{nat}} R_\calM$, we have:
\begin{align*}
\|\xnat_t - \hat{x}_t^{\mathrm{nat}} \| &= \bigg|\bigg| \xnat_t - \Proj_{\mathbb{B}_{R_{\mathrm{nat}}}}\left[ x_{t} - \sum_{i=0}^{h-1} \hat{G}_{t-1}^{[i]} u_{t-1-i}\right]\bigg|\bigg| \\
&\leq \bigg|\bigg| \xnat_t - x_{t} + \sum_{i=0}^{h-1} \hat{G}_{t-1}^{[i]} u_{t-1-i}\bigg|\bigg| & \text{(Pythagoras)} \\
&= \bigg|\bigg| \sum_{i=0}^{h-1} (G_{t-1}^{[i]} - \hat{G}_{t-1}^{[i]}) u_{t-1-i} + \sum_{i=h}^{t-2} G_{t-1}^{[i]}u_{t-1-i}\bigg|\bigg| \\
&\leq R_{\mathrm{nat}} R_\calM \left(\|\bar{G}_{t-1} - \hat{G}_{t-1}\|_{\ell_1, op} + \psi(h)\right)
\end{align*}
\end{proof}

\begin{lemma}[Realized Iterate Error]\label{lem:realized_iterate_err}
For $I=[r,s]$ with $r-1, s$ divisible by $h$, we can bound the realized iterate error as:
$$\sum_{t=r}^s c_t(x_t, u_t) - c_t(\hat{x}_t(M_{t-h:t-1}), \hat{u}_t(M_{t})) \leq 2 h \sum_{\tau_1 = (r-1)/h}^{s/h} b_{\tau_1 - 1} + 4 L \sqrt{d_u} \frac{R_{\mathrm{sys}}^2}{R_G} \left(\sum_{t=r}^{s}\|\bar{G}_{t-1} - \hat{G}_{t-1}\|_{\ell_1, op} + \psi(h) |I|\right) $$ 
where we denote $R_\mathrm{sys} \doteq R_G \radnat R_\calM$.
\end{lemma}

\begin{proof}
Consider the cost difference on a $h$-block indexed by $t = \overline{\tau_1\cdot h+1, \tau_1 \cdot h + h}$. Consider the following three cases:

\begin{enumerate}
    \item $b_{\tau_1} = 1$: in this case we are exploring and cannot give a better guarantee than $$c_t(x_t, u_t) - c_t(\hat{x}_t(M_{t-h:t-1)}, \hat{u}_t(M_t)) \leq 1$$
    \item $b_{\tau_1} = 0, b_{\tau_1 - 1} = 1$: while we are not exploring during the current round, and hence $u_t = \hat{u}_t(M_t)$ for $t = \overline{\tau_1\cdot h+1, \tau_1 \cdot h + h}$, we have explored in the previous round and therefore $u_t$ and $\hat{u}_t(M_t)$ may be arbitrarily far for $t \leq \tau_1 \cdot h$. This can induce $x_t$ and $\hat{x}_t(M_{t-h:t-1})$ to be quite far, especially the closer we get to $\tau_1 \cdot h + 1$. As such, in this event, we will also simply bound
    $$c_t(x_t, u_t) - c_t(\hat{x}_t(M_{t-h:t-1)}, \hat{u}_t(M_t)) \leq 1$$
    \item $b_{\tau_1} = 0, b_{\tau_1 - 1} = 0$: finally, in this case we have that $u_t = \hat{u}_t(M_t)$ for $t = \overline{\tau_1\cdot h - h + 1, \tau_1 \cdot h + h}$. We can expand:
    $$x_t = \xnat_t + \sum_{i=0}^{h-1} G_{t-1}^{[i]} u_{t-1-i} + \sum_{i=h}^{t-2} G_{t-1}^{[i]} u_{t-1-i}$$
    and $$\hat{x}_t(M_{t-h:t-1}) = \hat{x}_t^{\mathrm{nat}} + \sum_{i=0}^{h-1} \hat{G}_{t-1}^{[i]} \hat{u}_{t-1-i}(M_{t-1-i})$$
    By the observation above, for $t = \overline{\tau_1\cdot h + 1, \tau_1 \cdot h + h}$, we have
    \begin{align*}
    \|x_t - \hat{x}_t(M_{t-h:t-1})\| &\leq \| \xnat_t - \hat{x}_t^{\mathrm{nat}}\| + R_{\mathrm{nat}} R_\calM \left(\|\bar{G}_{t-1} - \hat{G}_{t-1}\|_{\ell_1, op} + \psi(h)\right) \\
    &\leq 2R_{\mathrm{nat}} R_\calM \left(\|\bar{G}_{t-1} - \hat{G}_{t-1}\|_{\ell_1, op} + \psi(h)\right) & \Cref{lem:xnat_diff}
    \end{align*}
    Hence, by the sub-quadratic Lipschitzness of the cost and \Cref{eq:x_bdd} we have
    \begin{align*}
     c_t(x_t, u_t) - c_t(\hat{x}_t(M_{t-h:t-1)}, \hat{u}_t(M_t))   &\leq 4L \sqrt{d_u} R_G \radnat^2 R_\calM^2 \left(\|\bar{G}_{t-1} - \hat{G}_{t-1}\|_{\ell_1, op} + \psi(h)\right)
    \end{align*}
\end{enumerate}

So for any $\tau_1 = \overline{0, T/h}$, we have

\begin{align*}
\sum_{t=\tau_1 \cdot h+1}^{\tau_1 \cdot h+h} c_t(x_t, u_t) - c_t(\hat{x}_t(M_{t-h:t-1}), \hat{u}_t(M_t)) &\leq \mathbf{1}_{b_{\tau_1} = 1} \cdot h + \mathbf{1}_{b_{\tau_1} = 0, b_{\tau_1 - 1} = 1} \cdot h \\ 
& + \mathbf{1}_{b_{\tau_1} = 0, b_{\tau_1 - 1} = 0} 4 L \sqrt{d_u} R_G \radnat^2 R_\calM^2 \sum_{t= \tau_1\cdot h + 1}^{\tau_1\cdot h+h}\|\bar{G}_{t-1} - \hat{G}_{t-1}\|_{\ell_1, op} \\
& + \mathbf{1}_{b_{\tau_1} = 0, b_{\tau_1 - 1} = 0} 4 L \sqrt{d_u} R_G \radnat^2 R_\calM^2 h\cdot \psi(h) \\
&\leq (b_{\tau_1} + b_{\tau_1 - 1}) \cdot h \\
&+ 4 L \sqrt{d_u} R_G \radnat^2 R_\calM^2 \left(\sum_{t= \tau_1\cdot h}^{(\tau_1 +1)\cdot h - 1}\|\bar{G}_{t} - \hat{G}_{t}\|_{\ell_1, op} + h\psi(h)\right) 
\end{align*}
summing over $\tau_1$ yields the desired result.

\end{proof}

\begin{lemma}[Comparator Error] We can bound the comparator error as:
\[\text{(comparator error)} \leq 8L \sqrt{d_u} \radnat^2 R_G^2 R_\calM^3 h \left(m \sum_{t=r-h-m}^s\|\bar{G}_t - \hat{G}_t\|_{\ell_1, op} + \psi(h) \right)\]
\end{lemma}

\begin{proof}
Let $M^\star = \arg\min_{M\in\mathcal{M}} \sum_{t=r}^s \sum_{t=r}^s c_t(x_t(M), u_t(M))$. We have that:

\begin{align*}
\text{(comparator error)} &\leq \sum_{t=r}^s c_t(\hat{x}_t(M^\star), \hat{u}_t(M^\star)) - c_t(x_t(M^\star), u_t(M^\star)) \\
&\leq 2L \sqrt{d_u} \radnat R_G R_{\calM} \sum_{t=r}^s \left(\|\hat{x}_t(M^\star) - x_t(M^\star)\| + \|\hat{u}_t(M^\star)) - u_t(M^\star)\|\right) 
\end{align*}
We have that 
\begin{align*}
\|\hat{u}_t(M^\star) - u_t(M^\star)\| &= \| \sum_{i=0}^{m-1} M^{\star, [i]} (\hat{x}^\mathrm{nat}_{t-i} - \xnat_{t-i})\| \\
&\leq \radnat R_{\calM}^2 \left(\sum_{\tau = t-m}^{t-1} \|\bar{G}_\tau - \hat{G}_\tau\|_{\ell_1, op} + \psi(h)\right) 
\end{align*}
With a bit more computation, we can also bound the difference in the states. First, let us expand the expression of the states:

\[x_t(M^\star) = \xnat_t + \sum_{i=0}^{h-1} G_{t-1}^{[i]} u_{t-1-i}(M^\star) + \sum_{i=h}^{t-2} G_{t-1}^{[i]} u_{t-1-i}(M^\star)\]
\[\hat{x}_t(M^\star) = \hat{x}_t^{\mathrm{nat}} + \sum_{i=0}^{h-1} \hat{G}_{t-1}^{[i]} \hat{u}_{t-1-i}(M^\star)\]

The only new thing we need to bound is:

\begin{align*}
\sum_{i=0}^{h-1} G_{t-1}^{[i]} u_{t-1-i}(M^\star) - \sum_{i=0}^{h-1} \hat{G}_{t-1}^{[i]} \hat{u}_{t-1-i}(M^\star) &= \sum_{i=0}^{h-1} G_{t-1}^{[i]} u_{t-1-i}(M^\star) - \sum_{i=0}^{h-1} \hat{G}_{t-1}^{[i]} u_{t-1-i}(M^\star) \\
&\quad + \sum_{i=0}^{h-1} \hat{G}_{t-1}^{[i]} u_{t-1-i}(M^\star) - \sum_{i=0}^{h-1} \hat{G}_{t-1}^{[i]} \hat{u}_{t-1-i}(M^\star) \\
&\leq \radnat R_{\calM} \| \bar{G}_{t-1} - \hat{G}_{t-1}\|_{\ell_1, op} + R_G \sum_{i=1}^h\|u_{t-i}(M^\star) - \hat{u}_{t-i}(M^\star)\|\\
&\leq 2 R_G \radnat R_{\calM}^2  h \left(m \sum_{\tau = t-h-m}^{t-1}\|\bar{G}_\tau -\hat{G}_\tau\|_{\ell_1, op} + \psi(h)\right)
\end{align*}
Plugging this into our expressions for $x_t(M^\star)$, and using previous bounds we have 
\begin{align*}
\|\hat{x}_t(M^\star) - x_t(M^\star)\| &\leq 3 R_G \radnat R_{\calM}^2 h \left(m \sum_{\tau = t-h-m}^{t-1} \|\bar{G}_\tau -\hat{G}_\tau\|_{\ell_1, op} + \psi(h)\right)    
\end{align*}

Hence we can finalize that:

\[\text{(comparator error)} \leq 8L \sqrt{d_u} \radnat^2 R_G^2 R_\calM^3 h \left(m \sum_{t=r-h-m}^s\|\bar{G}_t - \hat{G}_t\|_{\ell_1, op} + \psi(h) \right)\]
\end{proof}

\subsection{Proof of \Cref{thm:control_drc_full}}
\begin{proof}[Proof of \Cref{thm:control_drc_full}]
We have that 

\begin{align*}
\Regret_I &\leq \errtext{18L \sqrt{d_{\mathrm{min}}}R_G^2 R_{\calM}^2 \radnat^2 m (h+1)^{5/4}\sqrt{T}}{Known System Regret} + \errtext{2h \sum_{\tau=(r-1)/h}^{s/h}b_{\tau - 1}}{Exploration Penalty}  \\
&\quad + \errtext{12L\radnat^2 R_G^2 R_{\calM}^3 hm \left(\sum_{t=r}^s\|\bar{G}_t -\hat{G}_t \|_{\ell_1, op}+ 2R_G(h+m)\right)}{System Misspecification Induced Error} \\
&\quad + \errtext{18L \sqrt{d_u} R_G^2 R_{\calM}^3 \radnat^2 \psi(h) h (|I| + h + m)}{Truncation Error}
\end{align*}
Taking expectation and plugging in \Cref{prop:l1_op_err} we get: 
\begin{align}
\E[\Regret_I] &\leq \widetilde{\mathcal{O}}\bigg(L\sqrt{d_u} R_\mathrm{sys}^2 m \sqrt{T} + p|I| \notag\\
&\quad+ L R_\mathrm{sys}^2 R_\calM m \left(|I|^{1/2} \sqrt{\Exp[\var^{\mathrm{tot}}_I(\Gseq)]} + d_u R_{\mathrm{sys}} |I|^{1/2}p^{-1/2} + R_G m\right) \notag\\
&\quad + L \sqrt{d_u} R_{\mathrm{sys}}^2 R_\calM R_G m\bigg) \label{eq:full_constant_thm} \\
&= \widetilde{\mathcal{O}}^{\star}\left(Lm\left(|I| \sqrt{ \Exp[\var_{I}(\Gseq)]} + d_u T^{2/3}\right)  \right) \notag
\end{align}
where $R_\mathrm{sys} = R_G R_\calM R_{\mathrm{nat}}$, using $\var_{I}(\Gseq) = |I| \var^{\mathrm{tot}}_{I}(\Gseq)$, and that for the chosen $h$ we have $\psi(h) \leq R_G T^{-1}$.
\end{proof}

\newcommand{\calV}{\mathcal{V}}
\newcommand{\Z}{\mathbb{Z}}
\newcommand{\ubar}{\bar{u}}
\newcommand{\rhostar}{\rho_{\star}}
\newcommand{\cstar}{\mathcal{C}_{\star}}
\newcommand{\Rstar}{R_{\star}}
\newcommand{\Rcalk}{R_{\calK}}
\newcommand{\expthree}{\textsf{Exp}\text{3}}

\section{Sublinear Regret for State Feedback}\label{sec:sublinear_reg_state_feedback}
We demonstrate that it is (information-theoretically) possible to achieve sublinear (though large) regret against a benchmark of stabilizing static feedback control policies. 

We suppose there is a subset $\calK \subset \R^{\dimu \times \dimx}$ of feedback policies $K$, and our goal is to obtain regret compared to the best $K \in \calK$:
\begin{align*}
\Reg_T(\calK) := \sum_{t=1}^T c_t(x_t,u_t) - \inf_{K \in \calK}\sum_{t=1}^Tc_t(x_t^K, u_t^K),
\end{align*}
where $(x_t^K, u_t^K)$ are the iterates arising under the control law $u_t = K x_t$. 

For this setting, we propose an algorithm the classic \expthree  exponential weights algorithm (see, e.g. Chapter 3 of \citep{bubeck2012regret}) on an $\veps$-cover $\calK_{\veps}$ of $\calK$ in the operator norm. We maintain a constant controller $K$ on intervals of length $H$, and feed the losses on those intervals to the \expthree{} algorithm. Pseudocode is given in \Cref{alg:K_exp_weights}. 

\begin{algorithm}
\caption{Exponentially Weighted Control}\label{alg:K_exp_weights}
\begin{algorithmic}[1]
    \State\textbf{Input:} window length $H$, step size $\eta > 0$, finite $\veps$-cover $\calK_{\veps} \subset \calK$, initial estimate $K_{1} \in \calK_{\veps}$
    \State{}\textbf{Initialize} $\mathcal{L}_{1}(K) = 0$ and $p_1(K) = 1/|\calK_{\veps}|$,  $~\forall K \in \calK_{\veps}$,
    \For{$t=1, \ldots, T$}
    \State{}\textbf{Play} select $u_t = K_t x_t$.
    \State{}\textbf{Recieve} $\hat{c}_t = c_t(x_t,u_t)$
    \If{$t \mod H = 0$}
    \State{}\textbf{Set} $n = t/H$, $\ell_n = \sum_{i=t-H+1}^t \hat{c}_i$
    \State{}\textbf{Set} $\mathcal{L}_{n+1}(K_t) = \frac{1}{p_n(K_t)}\ell_n + \mathcal{L}_{n}(K_t)$
    \State{}\textbf{Set} $\mathcal{L}_{n+1}(K) = \mathcal{L}_{n}(K)$ for all $K \in \calK_{\veps} / \{K_t\}$
    \State{}\textbf{Set} $p_{n+1}(K) = \frac{\exp( \eta \mathcal{L}_{n+1}(K))}{\sum_{K' \in \calK_{\veps}}\exp( \eta \mathcal{L}_{n+1}(K'))}$
    \State{}\textbf{Sample} $K_{t+1} \sim p_{n+1}(\cdot)$.
    \Else
    \State{}Set $K_{t+1} = K_t$
    \EndIf
    \EndFor
\end{algorithmic}
\end{algorithm}

We state our regret bound under the (quite restrictive) assumption that all policies $K \in \calK$ are sequentially stabilizing. Formally, given a sequence of controllers $K \in \calK$, we define
\begin{align*}
\Phi_{s:t}(K) := \prod_{i=s}^t (A_i + K_i B_i) = (A_t + K B_t) \cdot (A_{t-1} + K B_{t-1})\cdot \dots \cdot (A_s + K B_s).
\end{align*}
We assume that $\Phi_{s:t}(K)$ exhibits geometric decay uniformly over all times for any fixed $K$:
\begin{assumption}\label{asm:Kstab} There exists $c_{\star} \ge 1$ and $\rho_{\star} \in (1/2,1)$ such that for any indices $s \le t$ and any fixed $K \in \calK$, $\|\Phi_{s:t}(K)\|_{\op} \le c_{\star} \rho_{\star}^{t-s}$.  We define the constant 
\begin{align*}
 \Rcalk := 1+\max\{\|K\|: K \in \calK\}
\end{align*}
\end{assumption}

\begin{theorem}\label{thm:Kstab_regret} Suppose \Cref{asm:Kstab,asm:cost} holds,  and for some $R_w \ge 1$ and $R_B \ge 0$, $\max_t\|w_t\| \le R_w$, and $\max_t \|B_t\| \le R_B$. In addition, suppose $T$ is large enough that $\cstar \rhostar^{T/4} \le 1/2$. 
Then, \Cref{alg:K_exp_weights} with horizon $H = \ceil{T^{1/4}}$, appropriate an step size $\eta$ and minimal $\veps$-covering $\calK_{\veps}$ of $\calK$ enjoys the following regret bound:
\begin{align*}
\Exp[\Reg_T(\calK)] \le L \mathcal{C}_1 (5 \Rcalk)^{\dimx \dimu/2}\cdot T^{1-\frac{1}{2(\dimx \dimu +3)}},
\end{align*}
where $\mathcal{C}_1 = \BigOh{\frac{\cstar^5}{(1-\rhostar)^3}\Rcalk^3 R_w^2(1+R_B)\sqrt{\dimu \dimx}}$. 
\end{theorem}
The theorem is established by a reduction to online multi-arm bandits in \Cref{sec:thmKstab_reg_proof} below.
\begin{remark}[Extensions of \Cref{thm:Kstab_regret}] The following analysis extends to policies of the form $u_t = (\Kstab_t + K)x_t + v$, where $(\Kstab_t)$ is a fixed sequence of control policies determined a priori, $K \in \calK \subset \R^{\dimu \times \dimx}$ is a feedback parameter, and $v \in \calV \subset \R^{\dimu}$ is a bounded affine term. Letting $(x_t^{K,c},u^{K,c}_t)$ denote the iterates produced by such a policy, our notion of regret is
\begin{align*}
\Reg_T(\calK\times \calV) := \sum_{t=1}^T c_t(x_t,u_t) - \inf_{(K,c) \in \calK \times \calV}\sum_{t=1}^T c_t(x_t^{K,c}, u_t^{K,c}),
\end{align*}
The only assumptions we require in general is that $\calV$ is bounded, and that $\calK$, combined with $(\Kstab_i)$, are sequentially stabilizing in the sense that, for any $s \le t$, the fixed $(\Kstab_i)$ sequence, and any $\Kstab_{s:t} \in \calK^{t-s+1}$, it holds that the products
\begin{align*}
\Phistab_{s:t}(K_{s:t}) := \prod_{i=s}^t (A_i + (\Kstab_i + K_i) B_i) 
\end{align*}
exhibit geometric decay. \qed
\end{remark}
\newcommand{\uualg}{u^{\alg}}
\newcommand{\xalg}{x^{\alg}}
\newcommand{\xbarn}[2]{\bar{x}_{n;#1}(#2)}
\newcommand{\ubarn}[2]{\bar{u}_{n;#1}(#2)}

\subsection{Proof of \Cref{thm:Kstab_regret} \label{sec:thmKstab_reg_proof}}

In what follows, assume that $H = T^{1/4}$ evenly divides $T$. For every index $n \in \N$, define $t_n = 1 + (n-1)H$. To avoid confusion, we $\xalg_t,\uualg_t$ denote the iterates produced by the algorithm. We define the sequence which begins at state $\xalg_{t_n}$ at time $t_n$, and rolls forward under controller $K$ for future times:
\begin{align*}
\xbarn{t_n}{K} &= x^{\alg}_{t_n}, \quad \xbarn{t+1}{K} = (A_t + B_t K) \xbarn{t}{K} + w_t, \quad t \ge t_n\\
\ubarn{t_n}{K} &= K\xbarn{t_n}{K}.
\end{align*}
Observe that, since we select a new controller $K_{t_n}$ just before each time $t_n$, we have
\begin{align*}
(\xbarn{t_n}{K_{t_n}},\ubarn{t_n}{K_{t_n}}) = (\xalg_t,\uualg_t), \quad \forall t \in [t_n,t_{n+1}-1].
\end{align*}
Therefore, defining the losses,
\begin{align*}
\ell_n(K) = \sum_{t= t_n}^{t_{n+1}-1} c_t(\xbarn{t_n}{K},\ubarn{t_n}{K}),
\end{align*}
we have
\begin{align*}
\ell_n(K_{t_n}) = \sum_{t= t_n}^{t_{n+1}-1} c_t(\xalg_t,\uualg_t).
\end{align*}
Therfore, we may decompose the regret as 
\begin{align*}
\Exp[\Reg_T(\calK)] &= \underbrace{\Exp\left[\sum_{n=1}^{T/H}\ell_n(K_{t_n})\right] - \inf_{K \in \calK_{\veps}}\Exp\left[\sum_{n=1}^{T/H}\ell_n(K)\right]}_{R_1} \\
&+\underbrace{\inf_{K \in \calK_{\veps}}\Exp\left[\sum_{n=1}^{T/H}\ell_n(K)\right] - \inf_{K \in \calK_{\veps}}\sum_{t=1}^T c_t(x_t^K,u_t^K)}_{R_2}\\
&+ \underbrace{\inf_{K \in \calK_{\veps}}\sum_{t=1}^T c_t(x_t^K,u_t^K)] - \inf_{K \in \calK}\sum_{t=1}^T c_t(x_t^K,u_t^K)}_{R_3}.
\end{align*}
Here, $R_1$ is the \emph{simple regret} on the $\ell_n$ sequence, $R_2$ is the extend to which the $\ell_n$ sequence approximates regret against controller $K \in \calK_{\veps}$ in the covering, and finally $R_3$ bounds the regret of the covering against the full set $\calK$. Here, expectations are over the randomness in the algorithm, and  due to the obliviousness of the adversary, we may assume that $(c_t,x_t^K,u_t^K)$ are deterministic and chosen in advance. We bound each of the three terms in sequence. Before proceeding, we use the following estimates:
\begin{lemma}[Key Term Bounds]\label{lem:key_K_stab_bounds} Suppose that $H = T^{1/4}$ is sufficiently large that $\cstar \rhostar^H \le 1/2$. Moreover, let $\Rstar = \frac{\cstar}{1-\rhostar}$. Then, 
\begin{enumerate}
	\item[(a)] For all $K \in \calK$ and $t \in [T]$, $\|x^K_t\| \le \Rstar R_w$
	\item[(b)] For any $t \ge t_n$ and $K \in \calK$, $\|\xbarn{t}{K}\| \le 2\cstar\Rstar R_w$
	\item[(c)] $\|(x^K_t,u^K_t)\| \le \Rcalk\Rstar R_w$ and $\|(\xbarn{t}{K},\ubarn{t}{K})\| \le 2\Rcalk\cstar\Rstar R_w$. 
\end{enumerate}
\end{lemma}
\begin{proof} \emph{Part a:} Unfolding the dynamics, and bounding $\|w_t\| \le R_w$ and $\Phi$ via \Cref{asm:Kstab},
\begin{align*}
\|x^K_t\| &= \left\|\sum_{s=1}^{t-1} \Phi_{s+1:t}(K) w_t\right\| \le R_w \cstar  \sum_{s \ge 0} \rhostar^s = \frac{\cstar R_w }{\rhostar} := \Rstar R_w.
\end{align*}
Next, we bound $\|\xalg_{t_n}\|$ for some $n$,
\begin{align*}
\|\xalg_{t_n}\| &= \left\|\sum_{i=1}^H \Phi_{t_{n-1}+i:t_n -1}(K_{t_{n-1}})w_{t_n - i} + \Phi_{t_{n-1}:t_n} (K_{t_{n-1}}) \xalg_{t_{n-1}} \right\|\\
&\le \Rstar R_w + \cstar \rhostar^H \|\xalg_{t_n-1}\|. 
\end{align*}
If $H$ is sufficiently large that $\cstar \rhostar^H \le 1/2$, then the above is just
\begin{align*}
\|\xalg_{t_n}\| \le \Rstar R_w + \frac{1}{2}\|\xalg_{t_n-1}\|,
\end{align*}
yielding the bound $\|\xalg_{t_n}\|  \le 2\Rstar R_w $ for all $n$. 
\emph{Part b:} Next, let us bound $\|\xbarn{t}{K}\|$ for some $t \ge t_n$.  We have
\begin{align*}
\|\xbarn{t}{K}\| &= \left\|\sum_{i = t_{n}+1}^t \Phi_{i:t}(K) w_tt + \Phi_{t_n:t}(K)\xalg_{t_n}\right\|\\
&\le R_w (\sum_{0}^{t - t_n - 1} \cstar \rhostar^i) + \cstar \rhostar^{t_n}\|\xalg_{t_n}\|\\
&\le R_w\sum_{0}^{t - t_n - 1} \cstar \rhostar^i +  2\rhostar^{t_n} \cstar R_{\star}  R_w.
\end{align*}
Using $R_{\star} = \frac{\cstar}{\rhostar}$, $\cstar \ge 1$, and $\sum_{0}^{t - t_n - 1} + \frac{\rhostar^{t_n}}{1-\rhostar} = \frac{1}{1-\rhostar}$, the above simplifies to $2\frac{\cstar^2 R_w}{1-\rhostar} = 2\cstar R_w\Rstar$. 

\emph{Part c:} This follows from the fact that, for any $x \in \R^{\dimx}$ and $K \in \calK$, $\|(x,Kx)\| \le (1+\|K\|)\|x\| \le \Rcalk\|x\|$. 
\end{proof}

\paragraph{Bounding $R_1$}
The term $R_1$ corresponds to the simple regret on the sequence of losses $\ell_n(K)$ over the discrete enumeration of controllers $K \in \calK_{\veps}$. Examining $\Cref{alg:K_exp_weights}$, we simply run the \expthree{} algorithm on these losses. By appealing to a standard regret bound for this algorithm with appropriate step size $\eta$, we ensure that
\begin{align*}
R_1 \le 2B\sqrt{\frac{T}{H}|\calK_{\veps}|\log |\calK_{\veps}|},
\end{align*}
provided that, for all $n$ and $K \in \calK_{\veps}$, $\ell_n(K) \in [0,B]$. To find the appropriate bound $B$, we note that from the growth condition on the costs, \Cref{asm:cost}, we have
\begin{align*}
0 &\le \ell_n(K) =  \sum_{t= t_n}^{t_{n+1}-1} c_t(\xbarn{t_n}{K},\ubarn{t_n}{K}) \le H \max_{t \ge t_n} c_t(\xbarn{t_n}{K},\ubarn{t_n}{K})\\
&\le LH \max\{1,\|(\xbarn{t_n}{K},\ubarn{t_n}{K})\|^2\} \le  8LH(\Rcalk\cstar\Rstar R_w)^2,
\end{align*}
where the last inequality uses \Cref{lem:key_K_stab_bounds}. Hence, 
\begin{align*}
R_1 \le 16L(\Rcalk\cstar\Rstar R_w)^2\sqrt{TH|\calK_{\veps}|\log |\calK_{\veps}|}.
\end{align*}
\paragraph{Bounding $R_2$:}
To bound $R_2$, it suffices to find a probability-one upper bound on 
\begin{align*}
\sup_{K \in \calK_{\veps}}\left|\sum_{n=1}^{T/H}\ell_n(K) - \sum_{t=1}^T c_t(x_t^K,u_t^K)\right| &= \sup_{K \in \calK_{\veps}}\left|\sum_{n=1}^{T/H}\sum_{t=t_{n}}^{t_{n+1}-1}c_t(\xbarn{t}{K},\ubarn{t}{K}) - c_t(x_{t}^K,u_{t}^K)\right|\\
&\le \frac{T}{H}\sup_{K \in \calK_{\veps}}\max_n\sum_{t=t_{n}}^{t_{n+1}-1}\left|c_t(\xbarn{t}{K},\ubarn{t}{K}) - c_t(x_{t}^K,u_{t}^K)\right|.
\end{align*}
Using the Lipschitz conditions on $c_t$, the bounds from \Cref{lem:key_K_stab_bounds}, and the bound $1+\|K\| \le \Rcalk$,
\begin{align*}
\sum_{t=t_{n}}^{t_{n+1}-1}\left|c_t(\xbarn{t}{K},\ubarn{t}{K}) - c_t(x_{t}^K,u_{t}^K)\right| &\le L(\Rcalk\cstar\Rstar R_w)\sum_{t=t_{n}}^{t_{n+1}-1}\|(\xbarn{t}{K},\ubarn{t}{K}) - (x_{t}^K,u_{t}^K)\|\\
&= L(\Rcalk\cstar\Rstar R_w)\sum_{t=t_{n}}^{t_{n+1}-1}\|(\xbarn{t}{K},K\xbarn{t}{K}) - (x_{t}^K,Kx_{t}^K)\|\\
&= L(\Rcalk\cstar\Rstar R_w)\Rcalk\sum_{t=t_{n}}^{t_{n+1}-1}\|\xbarn{t}{K} - x_{t}^K\|.
\end{align*}
Finally, we can compute that the difference  $\xbarn{t}{K} - x_{t}^K = \Phi_{t_n}^t (\xalg_{t_n} - x_{t_n}^K)$ depends only on the response to the state difference at time $t_n$. Hence, using \Cref{asm:Kstab} and \Cref{lem:key_K_stab_bounds},the above is at most
\begin{align*}
L(\Rcalk\cstar\Rstar R_w)\Rcalk\cdot \sum_{i \ge 0} \cstar \rhostar^i \|(\xalg_{t_n} - x_{t_n}^K)\| &\le L(\Rcalk\cstar\Rstar R_w) \Rcalk \cdot \underbrace{\frac{\cstar}{1-\rhostar}}{=\Rstar} \cdot \cstar\Rstar R_w\\
&\le L(\Rcalk\cstar\Rstar R_w)^2 \Rstar.
\end{align*} 
Concluding, we find
\begin{align*}
R_2 \le L(\Rcalk\cstar\Rstar R_w)^2 \cdot \frac{T\Rstar}{H}.
\end{align*}

\paragraph{Bounding $R_3$.} We now turn to bounding $R_3$, which captures the approximation error of approximating $\calK$ with $\calK_{\veps}$. We require the following technical lemma:
\begin{lemma}\label{lem:approx_quality_K} Let $K,K' \in\calK$ satisfy $\|K - K'\|_{\op} \le \veps$. Then, for all $t \ge 1$,
\begin{align*}
\|x_t^K - x_t^{K'}\| \le 4 \veps R_w \Rstar^2 R_B. 
\end{align*}
Hence, 
\begin{align*}
\|c_t(x_t^K,u_t^K) - c_t(x_t^{K'},u_t^{K'})\| \le  4\veps L R_w^2 \Rcalk^2 \Rstar^3 (   1+R_B).
\end{align*}
\end{lemma}
Using the fact that $\calK_{\veps}$ is an $\veps$-covering of $\calK$ in the operator norm means that for any $K \in \calK$, we can find a $K' \in \calK_{\veps}$ for which $\|K - K'\|_{\op} \le \veps$. Hence, from the above lemma
\begin{align*}
|\sum_{t=1}^T c_t(x_t^K,u_t^K) - c_t(x_t^{K'},u_t^{K'})| \le 4T\veps L R_w^2 \Rcalk^2 \Rstar^3 (1+R_B).
\end{align*}
Since $R_3 \le \sup_{K \in \calK}\inf_{\calK' \in \calK_{\veps}}|\sum_{t=1}^T c_t(x_t^K,u_t^K) - c_t(x_t^{K'},u_t^{K'})|$, we conclude
\begin{align*}
R_3 \le 4T\veps L R_w^2 \Rcalk^2 \Rstar^3 (1+R_B).
\end{align*}

\paragraph{Concluding the proof}
In sum, we found
\begin{align*}
\Exp[\Reg_T(\calK)] &= R_1 + R_2 + R_3 \\
&\le \BigOh{L(\Rcalk\cstar\Rstar R_w)^2}\left(\sqrt{TH|\calK_{\veps}|\log |\calK_{\veps}|} + \frac{T\Rstar}{H} + (1+R_B)T\veps \right),\\
&\le \BigOh{L(\Rcalk\cstar\Rstar R_w)^2(1+R_B)\Rstar}\left(\sqrt{TH|\calK_{\veps}|\log |\calK_{\veps}|} + \frac{T}{H} + T\veps \right),
\end{align*}
where in the last line, we use $\Rstar \ge 1$ and $1+R_B \ge 1$.
Setting $H = T^{1/4}$,
\begin{align*}
\Exp[\Reg_T(\calK)] &\le \BigOh{L(\Rcalk\cstar\Rstar R_w)^2(1+R_B)\Rstar}T^{3/4}\cdot \left(\sqrt{|\calK_{\veps}|\log |\calK_{\veps}|} + T^{1/4}\veps \right),
\end{align*}
We bound the cardinality of $\calK_{\veps}$. It suffices to ensure $\calK_{\veps}$ is an $\veps$-covering in the larger Frobenius norm, which is just the Euclidean norm on $R^{\dimx \dimu}$. Since $\calK$ is a bounded subset of this space, with radius at most $\Rcalk$, we can find a covering such that $|\calK_{\veps}| \le (\frac{5\Rcalk}{\veps})^{\dimx \dimu}$ (see, e.g. Chapter 4.2 in \cite{vershynin2018high}). This yields
\begin{align*}
|\calK_{\veps}| \log |\calK_{\veps}| \le \dimx \dimu \left(\frac{5\Rcalk}{\veps}\right)^{\dimx \dimu}  \log \left(\frac{5\Rcalk}{\veps}\right) \le \dimx \dimu \left(\frac{5\Rcalk}{\veps}\right)^{\dimx \dimu+1} = 5\Rcalk \cdot (5\Rcalk)^{\dimx\dimu} \veps^{-\dimx \dimu+1}
\end{align*}
where we use $\log x \le x$. Hence, we can bound
\begin{align*}
\Exp[\Reg_T(\calK)] \le \BigOh{L\Rcalk\Rstar(\Rcalk\cstar\Rstar R_w)^2(1+R_B)\sqrt{\dimx \dimu}}(5 \Rcalk)^{\dimx \dimu/2}\cdot T^{3/4}\left(\left(\frac{1}{\veps}\right)^{(1+\dimx \dimu)/2} + \veps T^{1/4}\right)
\end{align*}
Setting $\veps = T^{-\frac{1}{2(\dimx \dimu +3)}}$ gives
\begin{align*}
\Exp[\Reg_T(\calK)] &\le \BigOh{L\Rstar(\Rcalk\cstar\Rstar R_w)^2(1+R_B)\sqrt{\dimx \dimu}}(5 \Rcalk)^{\dimx \dimu+1}\cdot T^{1-\frac{1}{2(\dimx \dimu +3)}}\\
&= \BigOh{L\frac{\cstar^5}{(1-\rhostar)^3}\Rcalk^3 R_w^2(1+R_B)\sqrt{\dimx \dimu}}(5 \Rcalk)^{\dimx \dimu/2}\cdot T^{1-\frac{1}{2(\dimx \dimu +3)}},
\end{align*}

\subsection{Ommited Proofs}

\begin{proof}[Proof of \Cref{lem:key_K_stab_bounds}] \emph{Part a:} All such iterates can be realized by dynamics of the form $x_1 = 0$, $x_{t+1} = (A_t + B_t K_t) x_t + w_t$ and $u_t = K_t x_t$ for any appropriate sequence $(K_1,K_2,\dots)$ of elements of $\calK$. For such dynamics, we find
\begin{align*}
x_{t} = \sum_{s=1}^{t-1} \left(\prod_{i=s+1}^t (A_i + B_i K_i) \right)w_s = \sum_{s=1}^{t-1} \Phi_{s+1;t}(K_{s+1:t}) w_t.
\end{align*} 
Using $\|w_s\| \le R_w$ and the assumption $\Phi_{s+1;t}(K_{s+1:t}) w_t \le c_{\star} \rho_{\star}^{t-s -1}$ from \Cref{asm:Kstab}, we find 
\begin{align*}
\|x_t\| &\le R_w c_{\star}\sum_{s=1}^{t-1}  \rho_{\star}^{t-s -1} \le \frac{c_{\star}R_w}{1-\rho_{\star}} = \Rstar R_w\\
\|u_t\| +\|x_t\| &= \|x_t\| +\|K_t x_t\| \le \|x_t\|(1+\|K_t\|) \le \Rstar \Rcalk R_w 
\end{align*}

\emph{Part b:} Since the closed-loop dynamics for $\xbar_{k;t}^K$ and $x_t^K$ concincide for $t \ge t_k$ and are given by $x_{t+1} = (A_t + B_t K) x_t + w_t$, we can compute
\begin{align*}
\xbar_{k;t}^K - x_t^K = \left(\prod_{i=t_k}^t(A_t + B_t K)\right)(\xbar_{t_k;k}^K - x_t^K).
\end{align*}
Bounding $\|\xbar_{k;t}^K - x_t^K)\| \le 2R_w R_K$ from part (a) and $\|(A_t + B_t K)^{t - t_k}\| \le c_{\star} \rho_{\star}^{t-t_k}$ from \Cref{asm:Kstab} yields $\|\xbar_{k;t}^K - x_t^K)\| \le 2 R_x c_{\star} \rho_{\star}^{t-t_k}$. Summing over $t \ge t_k$  yields $\sum_{t \ge t_k}\|\xbar_{k;t}^K - x_t^K\| \le 2 R_w \Rstar^2$. Finally, using $\ubar_{k;t}^K = K\xbar_{k;t}^K$ and $u_{t}^K = K x_{t}^K$ gives
\begin{align*}
\sum_{t \ge t_k}\|\xbar_{t;k}^K - x_t^K\| + \|\ubar_{t;k}^K - u_t^K\|  \le (1+\|K\|) \sum_{t \ge t_k}\|\xbar_{t;k}^K - x_t^K\| \le 2 R_w \Rcalk \Rstar^2.
\end{align*}
\end{proof}

\begin{proof}[Proof of \Cref{lem:approx_quality_K}]
 Introducing the short hand $X_i = A_i + B_i K$ and $Y_i = A_i + B_i K'$, and expanding the dynamics, and introducing the short hand 
\begin{align*}
\|x_t^K - x_t^{K'}\| &= \left\|\sum_{s=1}^{t-1} \left(\prod_{i=s+1}^{t}(\underbrace{A_i + B_i K}_{=X_i}) - \prod_{i=s+1}^{t}(\underbrace{A_i + B_i K'}_{=Y_i}) \right) w_s\right\|\\
&\le R_w \sum_{s=1}^{t-1} \left\|\prod_{i=s+1}^{t}X_i - \prod_{i=s+1}^{t}Y_i \right\|_{\op} 
\end{align*}
Using an elementary matrix telescoping identiy,  
\begin{align*}
\prod_{i=s+1}^t X_i - \prod_{i=s+1}^t Y_i = \sum_{j=s+1}^{t} \left(\prod_{i=j+1}^{t} X_j\right)(X_j - Y_j)\prod_{i=s+1}^{j-1} Y_j.
\end{align*}
Thus, invoking stability assumption, \Cref{asm:Kstab}, and setting $R_B \ge \max_t \|B_t\|_{\op}$,
\begin{align*}
\left\|\prod_{i=s+1}^t X_i - \prod_{i=s+1}^t Y_i\right\|_{\op} &\le \sum_{j=s+1}^{t} \left\|\prod_{i=j+1}^{t} X_j\right\|_{\op}\left\|\prod_{i=s+1}^{j-1} Y_j\right\|_{\op} \|X_j - Y_j\|_{\op}\\
 &= \sum_{j=s+1}^{t} c_{\star} \rho_{\star}^{t - j+1}c_{\star} \rho_{\star}^{j - 1 - (s+1)} \|B_j (K-K')\|_{\op}\\
 &=  \frac{c_{\star}^2}{\rho_{\star}^2} \sum_{j=s+1}^{t} \rho_{\star}^{t  - (s+1)} \|B_j (K-K')\|_{\op}\\
 &\le \veps R_B \frac{c_{\star}^2}{\rho_{\star}^2} (t-s+1) \rho_{\star}^{t  - (s+1)} 
\end{align*}
Thus, we find
\begin{align*}
\|x_t^K - x_t^{K'}\| &\le  \veps R_w R_B \frac{c_{\star}^2}{\rho_{\star}^2} \sum_{s=1}^{t-1} (t-s+1) \rho_{\star}^{t  - (s+1)}\\
&\le  \veps R_w R_B \frac{c_{\star}^2}{\rho_{\star}^2(1-\rho_{\star})^2} \le  4\veps R_w R_B \frac{c_{\star}^2}{(1-\rho_{\star})^2} = 4 \veps R_w \Rstar^2 R_B, 
\end{align*}
where in the last step, we use $\rho_{\star} \ge 1/2$. Thus, applying \Cref{asm:cost} and \Cref{lem:key_K_stab_bounds},
\begin{align*}
\|c_t(x_t^K,u_t^K) - c_t(x_t^{K'},u_t^{K'})\| &\le L \max\{1,\|(x_t^K,u_t^K)\|,\|(x_t^{K'},u_t^{K'})\|\}\cdot\|(x_t^{K},u_t^{K}) - (x_t^{K'},u_t^{K'})\|\\
&\le L R_w \Rcalk \Rstar \cdot\|(x_t^{K},u_t^{K}) - (x_t^{K'},u_t^{K'})\|.
\end{align*}
Continuing, we bound
\begin{align*}
\|(x_t^{K},u_t^{K}) - (x_t^{K'},u_t^{K'})\| &= \|(x_t^{K},Kx_t^{K}) - (x_t^{K'},K'x_t^{K'})\|\\
 &\le \|(x_t^{K},Kx_t^{K}) - (x_t^{K'},Kx_t^{K'})\| + \|(K - K')x_t^{K'}\|\\
&\le \Rcalk\|x_t^{K'} - x_t^{K'}\| + \|(K - K')x_t^{K'}\|.\end{align*}
Finally, using the bound $\|x_t^{K'} - x_t^{K'}\| \le 4 \veps R_w \Rstar^2 R_B$ derived above, and bounding $\|(K - K')x_t^{K'}\| \le \veps\|x_t^{K'}\| \le \veps R_w \Rstar$ in view of \Cref{lem:key_K_stab_bounds}. Hence, 
\begin{align*}
\|c_t(x_t^K,u_t^K) - c_t(x_t^{K'},u_t^{K'})\| &\le  L R_w \Rcalk \Rstar (  4 \veps \Rcalk R_w \Rstar^2 R_B + \veps R_w \Rstar)\\
&\le  4\veps L R_w^2 \Rcalk^2 \Rstar^3 (   R_B + 1),
\end{align*}
where above we use that $\Rcalk, \Rstar \ge 1$ by assumption.
\end{proof}

\newcommand{\seq}{\mathcal{Z}}
\newcommand{\Picalk}{\Pi_{\mathcal{K}}}

\newcommand{\calE}{\mathcal{E}}
\newcommand{\rmd}{\mathrm{d}}
\newcommand{\lquad}{c_{\mathrm{q}}}
\newcommand{\lquadst}{c_{\mathrm{q};\star}}
\newcommand{\calFbar}{\overline{\calF}}
\newcommand{\calEtil}{\tilde{\calE}}

\newcommand{\calC}{\mathcal{C}}
\section{Lower Bounds and Separations}

\subsection{Separation between policy classes}\label{sec:app:separation}
	Let $\seq = (c_t,w_t,A_t,B_t)_{t \ge 1}$ denote sequences over costs,disturbances, and dynamics. We let $J_T(\pi;\seq)$ denote the cost of policy $\pi$ on the sequence $\seq$. Our lower bounds hold even against sequences which enjoy the following \emph{regularity} condition.
	\begin{definition}
	We say that $\seq$ is \emph{regular} if, for all $t$, $c_t(\cdot,\cdot)$ satisfies \Cref{asm:cost} with $L \le 1$,  and that for all $t$, $\|w_t\| \le 1$, $\|B_t\|_{\op} \le 1$ and  $\|A_t\|_{\op} \le 1/2$. 
	\end{definition}

	We define the policy classes
		\begin{align*}
		\Pidrc^+(h) &:= \left\{\pi:  u^\pi_t = u_0 + \sum_{i=0}^{h-1} M^{[i]}\xnat_{t-i} ,~ \forall t\right \}\\
		\Pidac^+(h) &:= \left\{\pi:  u^\pi_t = u_0 + \sum_{i=0}^{h-1} M^{[i]}w_{t-i-1} ,~ \forall t\right\}\\
		\Pifeed^+(h) &:= \left\{\pi:  u^\pi_t = u_0 + \sum_{i=0}^{h-1} M^{[i]} x^{\pi}_{t-i},~  \forall t\right\}\\
		\Pifeed(h,R) &:= \left\{\pi:  u^\pi_t = u_0 + \sum_{i=0}^{h-1} M^{[i]} x^{\pi}_{t-i}, ~\forall t, \quad \sum_{i=0}^{h-1}\|M^{[i]}\|_{\op} \le R.  \right\}.
		\end{align*}
		That is, $\Pidrc^+(h)$ are all length $h$ \drc{} policies of unbounded norm and allowing affine offsets, $\Pidac^+(h)$ are all length $h$ \dac{} policies of unbounded norm allowing affine offsets, and $\Pifeed^+(h)$ are all static feedback policies of unbounded norm allowing affine offsets, and $\Pifeed(h,R)$ are feedback policies of bounded norm and horizon. 

		The following theorem demonstrates that the \dac{}, \drc{}, and feedback parametrizations are fundamentally incommensurate. 

		\begin{theorem} Let $\calC_0 > 0$ denote a universal constant. There exists three regular sequences $\seq_1,\seq_2,\seq_3$ in $\dimu = \dimx = 1 $ which separate $\dac$, $\drc$, and feedback controllers, in the following sense:
		\begin{itemize}
			\item[(a)] Under $\seq_1$, the static feedback policy $\pi$ selecting $u^\pi_t = \frac{1}{4}x^\pi_t$ satisfies $J_T(\pi;\seq_1) = 0$, but
			\begin{align*}
			\inf_{\pi \in \Pidrc^+(h) \cup \Pidac^+(h)} J_T(\pi;\seq_1) \ge \calC_0 (T-h-2)
			\end{align*}
			\item[(b)]Under $\seq_2$, the \dac{} policy $\pi$ selecting $u^\pi_t = w_t$ satisfies $J_T(\pi;\seq_2) = 0$, but 
			\begin{align*}
			\inf_{\pi \in \Pidrc^+(h) \cup \Pifeed^+(h)} J_T(\pi;\seq_2) \ge \calC_0 (T-h-2)
			\end{align*}
			\item[(c)] Under $\seq_3$, the \drc{} policy $\pi$ selecting $u^{\pi}_t = \xnat_t$ satifies $J_T(\pi;\seq_3) = 0$, but 
			\begin{align*}
			\inf_{\pi \in \Pidac^+(h) } J_T(\pi;\seq_3) \ge \calC_0 (T-h-3)
			\end{align*} Moreover, for any $h \in N$, $R > 0$ and $T \ge 10h$, we have
			\begin{align*}
			\inf_{\pi \in \Pifeed(h,R) } J_T(\pi;\seq_3) \ge \calC_0 \frac{T}{h\max\{R,1\}}. 
			\end{align*}
		\end{itemize}
		\end{theorem}
			\begin{proof} We establish the separations for each part with different constant factors. One can choose $\calC_0$ to be the minimum of all constants which arise.

			\paragraph{Proof of part a.} We set $\seq_1$ to be the sequence with $A_t = 0$ for all $t$, $w_t = 1$ for all $t$, $c_t(x,u) = \frac{1}{8}(u - \frac{1}{4}x)^2$, and 
			\begin{align*}
			B_t = \begin{cases} 1 & t \text{ odd }\\
			-1 & t \text{  even}.
			\end{cases}
			\end{align*}
			This sequence is clearly regular, and is clear that the policy $u_t^\pi = \frac{1}{4}x_t^\pi$ has $J_T(\pi;\seq_1) = 0$. On the other hand, let $\pi \in \Pidrc^+(h) \cup \Pidac^+(h)$, $J_T(\pi;\seq_1)$. Since $A_t \equiv 0$, $\xnat_t = w_t$ so $\Pidrc^+(h) = \Pidac^+(h)$. Moreover, since $w_t = 1$ for all $t \ge 1$, any $\pi \in \Pidac^+(h)$ has $u^\pi_t = \bar{u}$ for some fixed $\bar{u}$ for all $t > h$. Then, for all $t > h+1$, $x_{t} = 1 + B_{t-1}\bar{u}$. Thus, $c_t(x^\pi_t,u^\pi_t) = \frac{1}{8}( \bar{u}- \frac{1+B_{t-1}\bar{u}}{4})^2$. Using the definition of $B_t$, 
			\begin{align*}
			c_t(x^\pi_t,u^\pi_t) + c_{t+1}(x^\pi_{t+1},u^\pi_{t+1}) = \frac{1}{8}( \bar{u}- \frac{1 - \bar{u}}{4})^2 + \frac{1}{8}( \bar{u}- \frac{1 + \bar{u}}{4})^2 = \frac{1}{128}\left( (3\bar{u}+1)^2 + (5\bar{u}+1) \right)^2 = \Omega(1).
			\end{align*}
			The bound follows.

			\paragraph{Proof of part b.} Set $c_t(x,u) = (u - w_{t-1})^2$. Then the \dac{} policy $u^\pi_t = w_{t-1}$ has zero cost. Further, set $B_t \equiv 0$, thus, $x_t \equiv \xnat_t$, so $\Pifeed^+(h)$ and $ \Pidrc^+(h)$ are equivalent on this system. 
			Finally, let $n = 2m+1$, and set $w_1 = 1$, and for $t \ge 1$, set
			\begin{align*}
			(A_{t},w_{t}) = \begin{cases} 
			(\frac{1}{2},\frac{1}{2}) & t \text{ is even }\\
			(\frac{1}{4},\frac{3}{4}) & t \text{ is odd } 
			\end{cases}.
			\end{align*}
			 Then, one can verify via induction that $x_t = \xnat_t = 1$ for all $t \ge 2$. Hence, for all $t \ge h+1$, any $\pi \in \Pifeed^+(h)\cup \Pidrc^+(h)$ has a constant input $u_t^\pi = \bar{u}$. However, $c_{t}(x^\pi_t,u^\pi_t) + c_{t+1}(x^\pi_t,u^\pi_t) = (\bar{u} - \frac{1}{2})^2 + (\bar{u} - \frac{3}{4})^2$, which is greater than a universal constant. Hence, the regret must spaces as $\Omega(T - (h+2))$.\\

			 \paragraph{Proof of part c.} Fix policity $\pi_{\star}$ to select $u^{\pi_\star}_t = \xnat_{t-1}$. Denote the sequences that arise from this policy as $(x^{\star}_t,u^{\star}_t)$. We set
			 \begin{align*}
			 c_t(x,u) = \frac{1}{4}\left((u - \frac{1}{2}u^{\star}_t)^2+|x - x^{\star}_t|\right)
			 \end{align*}
			By construction $\pi_{\star}$ has zero cost on $c_t$. Now, set $w_t = 1$ for all $t$, and
			\begin{align*}
			A_{t} = \begin{cases} \frac{1}{4} & t \mod 3 = 1\\
			\frac{1}{4} & t \mod 3 = 2\\
			0 & t \mod 3 = 0.
			\end{cases}, \quad B_t = \begin{cases} -\frac{1}{4} & t \mod 3 = 1\\
			-\frac{1}{20} & t \mod 3 = 1\\
			0 & t \mod 3 = 0.
			\end{cases}
			\end{align*}
			\begin{align*}
			u^{\star}_{3k+1} &= \xnat_{3k + 1} = 1\\
			u^{\star}_{3k+2} &= \xnat_{3k+2} = \frac{5}{4}\\
			u^{\star}_{3k+3} &= \xnat_{3k+3} =  \frac{21}{16}.
			\end{align*}
			Hence, a similar argument as in part (b), using the fact that that $w_t$ is constant but $u^{\star}_t$ is periodic, shows that any $\pi \in \Pidac^+(h)$ suffers cost $\Omega(T - h -3)$. 

			Let  us now analyze the performance of policies $\pi \in \Pifeed(m,R_M)$. First, observe that $x^{\star}_{t} = 1$ for all $t \ge 2$. 
			\begin{align*}
			x^{\star}_{3k+1} &= 1\\
			x^{\star}_{3k+2} &= \frac{5}{4} - \frac{1}{4}u^{\star}_{3k+1} = 1.\\
			x^{\star}_{3k+2} &= \frac{21}{16} - \frac{1}{4}u^{\star}_{3k+1} - \frac{1}{20}u^{\star}_{3k+2} \\
			&= \frac{21}{16} - \frac{1}{4} - \frac{1}{20}\cdot\frac{5}{4} = 1.
			\end{align*}
			Next, observe that for $\pi \in \Pifeed(h,R)$, and $t \ge h+1$,
			\begin{align*}
			u^{\pi}_t = c+ \sum_{i=0}^{h-1} M^{[i]} x^{\pi}_t = c+ \sum_{i=0}^{h-1} M^{[i]} x^{\star}_t + \sum_{i=0}^{h-1} M^{[i]} ( x^{\pi}_t - x^{\star}_t) &= \underbrace{(c + \sum_{i=0}^{h-1}M^{[i]})}_{:= \bar{u}} + \sum_{i=0}^{h-1} M^{[i]} ( x^{\pi}_t - x^{\star}_t)
			\end{align*}
			Defining $\epsilon_t = u^{\pi}_t  - \bar{u}$, we have
			\begin{align*}
			|\epsilon_t| \le \left|\sum_{i=0}^{h-1} M^{[i]} ( x^{\pi}_t - x^{\star}_t)\right|\le R \max_{i=0}^{h-1}| x^{\pi}_{t-i} - x^{\star}_{t-i}|.
			\end{align*} 
			Hence, for integers $k$,
			\begin{align}
			\max_{i\in [3]}|\epsilon_{3k+i}|  \le R \max_{t=3k-h+2}^{3k+3}| x^{\pi}_{t} - x^{\star}_{t}|. \label{eq:eps_k_bound_thing}
			\end{align}
			We now argue a dichotomoty on the size of $\max_{i\in [3]}|\epsilon_{3k+i}|$. First, we show that if the epsilons are large, the costs incurred on a past window of $h$ must be as well. This is \Cref{eq:eps_k_bound_thing} would necessitate that $x^\pi_t$ differs from $x^{\star}_t$ over the previous window.
			\begin{claim} Suppose $\max_{i\in [3]}|\epsilon_{3k+i}| \ge \frac{1}{32}$. Then, $\sum_{t=3k-h+2}^{3k+3}c_t(x^\pi_t,u^\pi_t) \ge \frac{1}{2^7 R}$.
			\end{claim}
			\begin{proof} By \Cref{eq:eps_k_bound_thing}, we have that if $\max_{i\in [3]}|\epsilon_{3k+i}| \ge \frac{1}{32}$, then $\max_{t=3k-h+2}^{3k+3}| x^{\pi}_{t} - x^{\star}_{t}| \ge \frac{1}{32 R}$. Since $c_t(x^\pi_t,u^\pi_t) \ge \frac{1}{4}| x^{\pi}_{t} - x^{\star}_{t}|$, the bound follows by upper bounding the maximum with the sum.
			\end{proof}
			On the other hand, we show that if the $\epsilon$-terms are small, then the costs on $t\in\{3k+1,3k+2,3k+3\}$ are at least a small constant. This is because the inputs selected by $\pi$, $\bar{u} + \epsilon_t$, are close to constant, and therefore can fit the periodic values of $u^{\pi_\star}_t$. 
			\begin{claim} Suppose $\max_{i\in [3]}|\epsilon_{3k+i}| \le \frac{1}{32}$. Then, $\sum_{i=1}^3 c_{3k+i}(x^\pi_t,u^\pi_t) \ge 2^{-12}$.
			\end{claim}
			\begin{proof}
			We expand
			\begin{align*}
			\sum_{i=1}^3 c_{3k+i}(x^\pi_t,u^\pi_t) &\ge \sum_{i=1}^3 \frac{1}{4}\left(u^{\star}_t - u^{\pi}_t\right)^2\\
			&\ge \sum_{i=1}^3 \frac{1}{4}\left(u^{\star}_{3k+i} - \bar{u} - \epsilon_{3k+i}\right)^2.\\
			&= \frac{1}{4}\left(\left(1 - \bar{u} - \epsilon_{3k+1}\right)^2 + \left(1 + \frac{1}{4} - \bar{u} - \epsilon_{3k+1}\right)^2 + \left(1 + \frac{1}{4} + \frac{1}{16}- \bar{u} - \epsilon_{3k+i}\right)^2\right).
			\end{align*}
			In particular, suppose $\max_{i\in [3]}|\epsilon_{3k+i}| \le \frac{1}{32}$. Then unless $|\bar{u} - 1| \le \frac{1}{16}$, the above is at least $\frac{1}{4}\left(\frac{1}{32}\right)^2 = 2^{-12}$. On the other hand, if $|\bar{u} - 1| \le \frac{1}{16}$. Then, $\frac{1}{4}\left(1 + \frac{1}{4} - \bar{u} - \epsilon_{3k+1}\right)^2 \ge (\frac{1}{4} - \frac{1}{16} - \frac{1}{32})^2 \ge \frac{1}{4}(\frac{1}{8})^2 \ge 2^{-12}$. 
			\end{proof}
			Combining both cases, we find that, for all $k$ such that $3k-h+2 \ge 1$,
			\begin{align*}
			\sum_{t=3k-h+2}^{3k+3}c_t(x^\pi_t,u^\pi_t) \ge \frac{2^{-17}}{\max\{R,1\}}.
			\end{align*}
			In particular, we find that for  $J_T(\pi;\seq_3) \ge\Omega(\frac{T}{h \max\{1,R\}})$ provided (say) $T \ge 10h$.
	\end{proof}
\subsection{Linear Regret Against \drc{} and \dac{} }\label{sec:app:drc_lower}
In this section, we demonstrate linear regret against \drc{} and \dac{}. We consider a distribution over instances of the following form
\begin{align}
    A_t =0, \quad B_t = \begin{bmatrix} 1 & 0 & 0\\
    0 & \beta_t & 0\\ 0 & 0 & 1\end{bmatrix}, \quad w_t = -\begin{bmatrix}  \omega_{t-1}\\
     \omega_t \\
     1
    \end{bmatrix}, t \ge 0 . \quad
\end{align}
Define the constant $\alpha = 1/4$. For a given $\sigma \in (0,1/8)$, let $\mathcal{D}_{\sigma}$ denote the distribution over $(A_t,B_t,w_t)$ induced by drawing
\begin{align*}
\beta_t \iidsim [1-\sigma,1+\sigma], \quad \omega_t \iidsim \{1-\alpha\sigma,1+\alpha\sigma\}.
\end{align*}
Note that these instances are (a) controllable, (b) stable, and (c) have variance scaling like $T\sigma^2$, and (d) the \drc{} and \dac{} parametrizations coincide. Letting $v[i]$ denote the $i$-th coordinate of vectors $v$, we consider cost of the form
\begin{align}
c_f(x,u) = x[2]^2 + u[2]^2 + f(x[1]). \label{eq:cf}
\end{align}
for either $f(z) =|z|$ or $f(z) = z^2$. Note that both choices of $f$ ensure that $c_f$ satisfies \Cref{asm:cost}, and the latter choice ensures that $c_f(x,u)$ is second order smooth.

\begin{theorem} Let $\alg$ be any online learning algorithm. Let $c_f(x,u)$ as in \eqref{eq:cf}. Then, for any $\sigma > 0$, there exist a \drc{} policy $\pi_{\star} \in \Pidrc(1,1)$\footnote{Equivalently, in $\Pidac(1,1)$ since $A_t \equiv 0$} such that expected regret incurred by $\alg$ under the distribution $\calD_{\sigma}$ and cost $c_f(x,u)$ is at least
\begin{align*}
\Exp_{\calD_{\sigma},\alg}  [J_T(\alg) - J_T(\pi_{\star})] \ge   \cdot \begin{cases} C_1T\sigma^2 & \text{for } f(z) =z^2, ~~T \ge \tfrac{C_0 }{\sigma^2}\\
C_1T\sigma & \text{for } f(z) = z, ~~T \ge \tfrac{C_0 }{\sigma},
\end{cases}
\end{align*}
where above, $C_0,C_1$ are universal, positive constants.
\end{theorem}

\begin{proof} In what follows, $\Exp[\cdot]$ denotes expectation under the nstance from $\calD_{\sigma}$, and any algorithmic randomness. 

We expand $x_t$ and $u_t$ into its coordinates $x_t = (x_{t;1},x_{t;2},x_{t;3})$ and of $u_t = (u_{t;1},u_{t;2},u_{t;3})$.  For any policy $\pi$, we can decompose its cost as 
\begin{align*}
J_T(\pi) &= \sum_{t=1}^T c_f(x_{t}^\pi,u_{t}^\pi)  = (u_{T;2}^\pi)^2 +  f(x_{T;1}^\pi) + \sum_{t=1}^{T-1} \lquad(x_{t+1;2}^\pi,u_{t;2}^\pi) + f(x_{t;1}^\pi) \\
&\quad \text{ where } \lquad(x_{t+1;2}^\pi,u_{t;2}^\pi) =  (x_{t+1;2}^\pi)^2 +(u_{t;2}^\pi)^2.
\end{align*}
Note that $x_1 = 0$ in the above.  The following lemma characterizes the conditional expectation of the $\lquad$
\begin{lemma}\label{lem:lquad_lem} There exist a constant $\lquadst$ such that 
\begin{align*}\Exp[\lquad(x_{t+1;2},u_{t;2})]  - \lquadst = (2+\sigma^2/3)\Exp[(u_{t;2} - \ubar)^2],
\end{align*} 
where $\ubar = \frac{1}{2(1+\sigma^2/6)}$.
\end{lemma}
The proof of \Cref{lem:lquad_lem} is deferred to the end of the section. 
\paragraph{Bounding the cost of $\pi^{\star}$}
We select $\pi^{\star}$ to be \dac{} (or equivalently, $\drc{}$ since $A_t \equiv 0$) policy given by 
\begin{align*}
M = \begin{bmatrix} 0 & 1 & 0 \\ 
0 & 0 & \bar{u} \\ 
0 & 0 & 0
\end{bmatrix}, \quad u_t^M = Mw_{t-1}.
\end{align*}

We find that
\begin{align*}
&\Exp[\lquad(x_{t+1;2}^{\pi_{\star}},u_{t;2}^{\pi_{\star}}) ] - \lquadst = (2+\sigma^2/3)\Exp[(u^{M}_{t;2} - \bar{u})^2] = 0, \quad 2 \le t \le T-1\\
&\Exp[f(x_{t;1}^M)] =  \Exp[f(u_{t-1;1}^M -\omega_{t-2} )] = 0, \quad 3 \le t \le T.
\end{align*}
Since the noise is uniformly bounded independent for all $\sigma, \alpha \le 1$, we conclude
\begin{align}
\Exp[J_T(\pi^M)] - (T-1)\lquadst &= \underbrace{\Exp[f(x_{t;2}^\pi)]}_{\Exp[f(\omega_t)]} + (u_{T;2}^{\pi_{\star}})^2 + \underbrace{\Exp[f(x_{T;1}^\pi)]}_{=0} + \underbrace{\Exp[c_f(x_1,u_1)]}_{=0}\nonumber\\
&= \frac{f(1-\alpha \sigma) + f(1+\alpha \sigma)}{2} + \underbrace{(\ubar \cdot w_{t-1;3})^2}{=\ubar^2}  \le 2. \label{eq:pistar_cost},
\end{align}
where we use $\alpha \le 1/24$, $\sigma \le 1/8$, and $\ubar \le \frac{1}{2(1+\sigma^2/6)}$, and $f(z) \le \max\{|z|,|z|^2\}$ to achieve the bound \Cref{eq:pistar_cost}. 

\paragraph{Bounding the cost of adaptive policies} Fix any online learning algorithm $\alg$; we  lower bound its performance. Because $B_t$ and $w_t$ are drawn from a fixed probability distribution, and are therefore oblivious to the learner's actions, we may assume without loss of generality that $\alg$ is deterministic. 

The first step is to argue that any algorithm with small cost must select inputs where are bounded away from zero. Specifically,  define the devent $\calE_t:= \{u_{t;2}^{\alg} \ge 1/6\}$. Using \Cref{lem:lquad_lem} together with $\ubar\ge 1/3$,
\begin{align}
\Exp[J_T(\alg)] - T\lquadst &\ge 2\sum_{t=1}^T \Exp\left[f(x_{t;1} + (u^{\alg}_{t;2} -\ubar)^2] \right] \nonumber\\
&\ge 2\sum_{t=1}^T \Exp\left[f(x_{t;1}^\alg)\right] +  (\frac{1}{3} - \frac{1}{6})^2\Pr\left[\calE_t^c\right]\nonumber \\
&= \sum_{t=1}^T  \Exp\left[f(x_{t;1}^\alg)\right] + \frac{1}{18}\Pr\left[\calE_t^c\right]\nonumber\\
&\ge \sum_{t=3}^T  \Exp\left[f(x_{t;1}^\alg) \mid \calE_{t-1}\right]\Pr\left[\calE_{t-2}\right] + \frac{1}{18}\sum_{t=1}^3\Pr\left[\calE_t^c\right].
\end{align}
We now lower bound $\Exp\left[f(x_{t;1}^\alg) \mid \calE_{t-2}\right]$, again deferring the proof to the end of the section.
\begin{lemma}\label{lem:Ef_lb} For any executable policy $\pi$
$\Exp[f(x_{t;1}^\pi) \mid \calE_{t-2}] \ge \frac{1}{2}f(\alpha \sigma)$, provided $\alpha \le 1/24$.
\end{lemma}
Combining \Cref{lem:Ef_lb} with the above bound, we have
\begin{align*}
\Exp[J_T(\alg)] - (T-1)\lquadst  &\ge \sum_{t=1}^T  \frac{f(\alpha \sigma)}{2}\Pr\left[\calE_{t-2}\right] + \frac{1}{18}\sum_{t=1}^T\Pr\left[\calE_t^c\right]\\
&\ge \min\left\{\frac{f(\alpha \sigma)}{2},\frac{1}{18}\right\}\sum_{t=1}^{T-2}  \Pr\left[\calE_{t}\right] + \Pr\left[\calE_t^c\right]\\
&\ge (T-2)\min\left\{\frac{f(\alpha \sigma)}{2},\frac{1}{18}\right\}. 
\end{align*}
With $\sigma \le 1$ and $\alpha = \frac{1}{24}$, $\frac{f(\alpha \sigma)}{2} \le \frac{1}{18}$. Combining with \Cref{eq:pistar_cost}, we have
\begin{align*}
\Exp[J_T(\alg)] -\Exp[J_T(\pi^{\star})] \ge (T-2)\frac{f(\alpha \sigma)}{2} -2.
\end{align*}
The bound follows.
\end{proof}
\subsubsection{Omitted proofs}

\begin{proof}[Proof of \Cref{lem:lquad_lem}]
Let $(\calFbar_t)_{t\ge 1}$ denote the filtration induced by setting $\calFbar_t$ to be the sigma-algebra generated by  $(\beta_1,\dots,\beta_{t-1},\omega_1,\dots,\omega_t\}$. We have
\begin{align*}
\Exp[\beta_t^2 \mid \calFbar_t] &= \frac{1}{2\sigma}\int_{u=1-\sigma}^{1+\sigma}u^2 \rmd u = \frac{(1+\sigma)^3 - (1-\sigma)^3}{6\sigma}\\
&= \frac{1 + 3\sigma + 3 \sigma^2 + \sigma^3 - (1-  3\sigma^2 +3 \sigma^2 - \sigma^3)}{6\sigma}\\
&= 1 + \sigma^3/3.
\end{align*}
Set $\lquad(x,u) = x^2 + u^2$. 
\begin{align*}
\Exp[\lquad(x_{t+1;2},u_{t;2}) \mid \calFbar_t] &= u_{t;2}^2 +  \Exp[(\beta_t u_{t;2} - \omega_t)^2 \mid \calFbar_t] \\
&= u_{t;2}^2 (1+\Exp[\beta_t^2 \mid \calFbar_t]) - 2u_{t;2}\Exp[\beta_t w_t \mid \calFbar_t] + \Exp[\omega_t^2 \mid \calFbar_t]\\
&= u_{t;2}^2 (2+\sigma^2/3) - 2u_{t;2} + (1+c\sigma^2)
\end{align*}
at $u_{\star;2} = \frac{1}{2(1+\sigma^2/6)}$. Define $\lquadst := \min_{u_{t;2}}\Exp[\ell_2(x_{t;2},u_{t;2}) \mid \calFbar_t]$, we then have
\begin{align}
\Exp[\lquad(x_{t+1;2},u_{t;2}) \mid \calFbar_t] - \ell_{2;\star} =  (2+\sigma^2/3)(u_{t;2} - u_{\star;2})^2.
\end{align}
\end{proof}
\begin{proof}[Proof of \Cref{lem:Ef_lb}]
Let us introduce a second event, $\calEtil$, defined as
\begin{align*}
\calEtil_{t-1} := \left\{(1 + \alpha \sigma) + u^{\alg}_{t;2} (1-\sigma)\le x_{t-1;2}^\alg \le (1 - \alpha \sigma) + u^\alg_{t;2} (1+\sigma)\right\}.
\end{align*}
Then, since $f$ is non-negative,
\begin{align*}
\Exp[f(x_{t;1}^\pi) \mid \calE_{t-2}] &\ge\Exp[\I\{\calEtil_{t-1}\}f(x_{t;1}^\pi) \mid \calE_{t-2}]\\
&=\Pr[\calEtil_{t-1}\mid \calE_2] \Exp[\cdot\Exp[f(x_{t;1}^\pi) \mid \calEtil_{t-1}\cap \calE_{t-2}].
\end{align*}
Let's first lower bound $\Pr[\calEtil_{t-1}\mid \calE_2]$. Writing $x_{t-1;2}^\alg = \beta_{t-2}u^{\alg}_{t;2} + \omega_{t-2}$, $\calEtil_{t-1}$ occurs as soon as
\begin{align*}
(1 + \alpha \sigma) + u^{\alg}_{t-2;2} (1-\sigma) &\le \beta_{t-2}u^{\alg}_{t-2;2} + \omega_{t-2}\\
(1 -\alpha \sigma) + u^{\alg}_{t-2;2} (1+\sigma) &\ge \beta_{t-2}u^{\alg}_{t-2;2} + \omega_{t-2}.
\end{align*}
Using that $1-\alpha \sigma \le \omega_{t-2} \le 1+\alpha \sigma $ and rearranging the above, it is enough that
\begin{align*}
2\alpha \sigma  &\le (\beta_{t} - (1-\sigma) )u^{\alg}_{t-2;2}\\
-2\alpha \sigma &\ge -(( 1+\sigma) -\beta_{t}) u^{\alg}_{t-2;2}.
\end{align*}
Now, if $\calE_{t-2}$ holds, that $u^{\alg}_{t;2} \ge 1/6$. Furthermore, by construction $1-\sigma \le \beta_{t} \le 1+\sigma$. Therefore, if $\calE_{t-2}$ holds, then $\calEtil_{t-1}$ holds as long as 
\begin{align*}
\beta_{t} - (1-\sigma) \ge 12 \alpha \sigma \quad \text{and} \quad 
(1+\sigma) - \beta_t \ge 12 \alpha \sigma.
\end{align*}
In particular, for $\alpha = \frac{1}{24}$, then 
\begin{align}
\Pr[\calEtil_{t-1}\mid \calE_2] \ge \Pr[\beta_t \in [1 - \tfrac{2}{\sigma},1 + \tfrac{2}{\sigma}]] = \frac{1}{2}. \label{eq:Etil_midEtwo}
\end{align}
Next, we lower bound $\Exp[f(x_{t;1}^\pi) \mid \calEtil_{t-1}\cap \calE_{t-2}]$. To do so, we observe that $x_{t;1} = \omega_{t-2} - u^\alg_{t-1;1}$. Moreover, since $\alg$ is deterministic (see discussion above), $u^{\alg}_{t-1;1}$ is a deterministic function of $x^{\alg}_{1:t-1}$, and moreover, $\calEtil_{t-1}, \calE_{t-2}$ are determinied by $x^{\alg}_{1:t-1}$. Hence, 
\begin{align*}
\Exp[f(x_{t;1}^\pi) \mid \calEtil_{t-1}, \calE_{t-2}] &= \Exp[\Exp[f(x_{t;1}^\pi) \mid x^{\alg}_{1:t-1}] \mid \calEtil_{t-1} \cap \calE_{t-2}]\\
&= \Exp[\Exp[f(\omega_{t-2} - u^\alg_{t-1;1}) \mid x^{\alg}_{1:t-1}] \mid \calEtil_{t-1} \cap \calE_{t-2}]\\
&\ge \Exp[\min_{u \in \R}\Exp[f(\omega_{t-2} - u) \mid x^{\alg}_{1:t-1}] \mid \calEtil_{t-1} \cap \calE_{t-2}].
\end{align*}
Thus, it suffices to characterize the distribution of $\omega_{t-1} \mid x^{\alg}_{1:t-1}$ whenever the events $ \calEtil_{t-1} \cap \calE_{t-2}$. Indeed, we claim that $\omega_{t-2} \mid x^{\alg}_{1:t-1}$ is \emph{uniform} on $\{1- \alpha \sigma, 1+ \alpha \sigma\}$ on whenever $\calEtil_{t-1}$ holds.
\begin{claim} Let $\omega^{-} = 1 - \alpha \sigma$ and $\omega^+ = 1 + \alpha \sigma$. then (with probability one) $\Pr[\omega_{t-2} = \omega^- \mid x^{\alg}_{1:t-1}] = \Pr[\omega_{t-2} = \omega^+ \mid x^{\alg}_{1:t-1}] = \frac{1}{2}$ if $\calEtil_{t-1}$ holds.
\end{claim}
\begin{proof} If $\calEtil_{t-1}$ holds, then there exists exactly two values $\beta$ and $\beta'$ in $[1 -\sigma, 1+\sigma]$ such that
\begin{align*}
\omega^+ + \beta u^\alg_{t-2;2} = \omega^- + \beta' u^\alg_{t-2;2} = x^{\alg}_{t-1;2}.
\end{align*}
Even conditioned on $x^{\alg}_{1:t-2}$, $\omega_{t-2}$ is uniformly distributed on $\{\omega^-,\omega^+\}$, and since $\beta$ and $\beta'$ have the same probability mass under the uniform distribution of $\beta_{t-2}$, it follows that $\Pr[\omega_{t-2} = \omega^+ \mid x^{\alg}_{1:t-1}] = [\omega_{t-2} = \omega^- \mid x^{\alg}_{1:t-1}]$ when $\calEtil_{t-1}$ holds. 
\end{proof}
Hence, 
\begin{align*}
\Exp[f(x_{t;1}^\pi) \mid \calEtil_{t-1}, \calE_{t-2}] \ge \min_{u \in \R}\frac{1}{2}(f(1 + \alpha \sigma - u) + f(1 - \alpha \sigma - u)).
\end{align*}
For $f(z) = z^2$ or $f(z) = |z|$, the minimum is attained at $u = 1$, with values $\alpha^2 \sigma^2$ and $\alpha \sigma$, respectively, both equal to $f(\alpha \sigma)$. Thus,combining with \Cref{eq:Etil_midEtwo},  we conclude
\begin{align*}
\Exp[f(x_{t;1}^\pi) \mid \calE_{t-2}] \ge \frac{1}{2}f(\alpha \sigma).
\end{align*}
\end{proof}

\subsection{Lower Bound without Stability}

	\begin{theorem} Consider a scalar LTV system with $A = \rho \in [0,1]$, and $B_t$ drawn independently and uniformly at random from $\{-1\}$. Suppose that $w_1 = 1$, and $w_t = 0$ for all $t > 1$. Finally, let $c_t(x,u) = x^2$ be a fixed costs. Then, 
	\begin{enumerate}
		\item[(a)] There then $\drc{}$ policy $u_t^\pi = -\rho B_2 \xnat_{t}$, \dac{} policy $u_t^\pi = -\rho B_2  w_{t-1}$, and static-feedback policy $u_t^\pi = - \rho B_2 A $ (all chosen with foreknowledge of the $(B_t)_{t \ge 1}$ sequence) all enjoy:
		\begin{align*}
		J_T(\pi) = 1, \quad \text{with probability 1}.
		\end{align*}
		\item[(b)] Any online learning algorithm without foreknowledge of $(B_t)_{t \ge 1}$ must suffer expected cost 
		\begin{align*}
		\Exp[J_T(\pi)] = \Omega\left(\min\left\{T,\frac{1}{1-\rho}\right\}\right).
		\end{align*}
	\end{enumerate} 
	\end{theorem}  
	\begin{proof} 
	In part (a), all policies choose $u_2 = - B_2 \rho w_1$, so that $x_3 = \rho x_2 - \rho w_1 = 0$. Since $0$ is an equilibrium point and $w_t = 0$ for all $t \ge 2$, the system remains at zero. Hence, the only cost incurred is at times $1$ and $2$, which are costs of zero and $1$ respectively. In part (b), we use the unbiasedness of $B_t$ to recurse
	\begin{align*}
	\E[x_2] &= Ax_1 + \E[B_1]\E[u_1] + w_1 = 1 \\
	\E[x_{t+1}] &= A\E[x_t] + \E[B_t]\E[u_t] + \underbrace{w_t}_{= 0} = \rho \E[x_t],\quad \forall t \ge 2, 
	\end{align*}
	yielding $\E[x_{t}] = \rho^{t-2}$ for all $t \ge 2$. Hence, the expected cost of the policy is
	\begin{align*}
	\Exp[J_T(\pi)]  = \sum_{t=2}^{T} \rho^{(2t - 2)}.
	\end{align*}
	Considering the cases where $\rho \ge \frac{1}{T-2}$ and $\rho \le \frac{1}{T-2}$, we find the above is $\Omega\left(\min\left\{T,\frac{1}{1-\rho}\right\}\right)$.
	\end{proof}

\subsection{Hardness of Computing Best State Feedback Controller}\label{sec:NP_hardness}

\def\bb{\mathbf{b}}
\def\bc{\mathbf{c}}
\def\bB{\mathbf{B}}
\def\bK{\mathbf{K}}
\def\bM{\mathbf{M}}
\def\bC{\mathbf{C}}
\def\be{\mathbf{e}}
\def\etil{\tilde{\mathbf{e}}}
\def\bk{\mathbf{k}}
\def\bu{\mathbf{u}}
\def\bx{\mathbf{x}}
\def\bh{\mathbf{h}}
\def\w{\mathbf{w}}
\def\uv{\mathbf u}
\def\R{\mathbb{R}}

Consider a time-varying linear dynamical system with no noise:
$$x_{t+1} = A_t x_t + B_t u_t + w_t, \quad \forall t, \, w_t \equiv 0 $$
subject to changing convex costs $c_t(x, u)$.
We show that even in the no-noise setting properly learning the {\it optimal state feedback} policy is computationally hard. This statement holds with the control agent having full prior knowledge of the dynamics $(A_t, B_t, c_t)$. It relies on a reduction to the MAX-3SAT problem which is $\NPtime$-hard. Our lower bound is inspired by the analogous one for discrete MDPs by \cite{even2005experts}.

\begin{theorem}\label{thm:comp_lb}
There exists a reduction from \maxsat{} on $m$-clauses and $n$-literals to the problem of finding the state feedback policy $K$ {\it optimal} for the cost $\sum_{t=1}^T c_t(x_t^K, u_t^K)$, over sequentially stable dynamics given by $(A_t, B_t, c_t)$: a solution to \maxsat{} with $k$ value implies optimal cost of at most $-k$, and a solution $K$ to the control problem with $-k-\epsilon$ value implies optimal value of $\maxsat{}$ at least $k$ for any known $\epsilon>0$.
\end{theorem}

Let us first describe the construction of the dynamics that reduce the optimal control problem to the MAX-3SAT problem. Consider a 3-CNF formula $\phi$ with $m$ clauses $C_1, \dots, C_m$ and $n$ literals $y_1, \dots, y_n$. The state space is of dimensionality $d_x=n+1$ and the action space is of dimensionality $d_u=2$. The control problem is given as a sequence of $m$ episodes corresponding to the clauses of the formula $\phi$.

For a single clause $C_j$ with $j \in [m]$, let the dynamics $(A_t, B_t, c_t)_{t=1}^{n+2}$ be an episode of length $n+2$ constructed as follows. The initial state is $x_1 = [1, \bm{0}_n]^{\top}$. The state transitions are independent of the clause itself given by the following $A_t \in \R^{n+1 \times n+1}$:
\begin{itemize}
    \item for $1 \leq t < n$, $A_t(t+1) = [\bm{1}_n, 0]^{\top}$, $A_t(n+1) = [\bm{0}_n, 1]^{\top}$, $A_t(i) = [\bm{0}_{n+1}]^{\top}$ for all other $i \neq t+1, n+1$.
    \item for $t = n$, it becomes $A_t(n+1) =  [\bm{1}_{n+1}]^{\top}$ and $A_t(i) =  [\bm{0}_{n+1}]^{\top}$ for all other $i \neq n+1$. 
    \item for $t = n+1$ take $A_t(1) = [\bm{1}_{n+1}]^{\top}$ and $A_t(i) = [\bm{0}_{n+1}]^{\top}$ for all other $i \neq 1$; for $t=n+2$ take $A_t = \bm{0}_{n+1 \times n+1}$ to ensure sequential stability.
\end{itemize}
The action matrices $B_t$ along with the costs $c_t$, on the other hand, depend on the content of the clause $C_j$ itself. In particular, let $I_j$ be the set of indices of the literals that are in clause $C_j$. We define the regularity cost to be $c(x, u) = S_x(x) + (1-x(n+1))^2 \cdot S_u(u)$, where $S_x(\cdot) = \mathrm{dist}(\cdot, \Delta_{n+1})$ and $S_u(\cdot) = \mathrm{dist}(\cdot, \Delta_2)$ are the distance functions to the simplex sets of corresponding dimensionality. The action matrices and costs are given as follows:
\begin{itemize}
\item for $1 \leq t \leq n$ and $t \not \in I_j$, $B_t = \bm{0}_{n+1 \times 2}$ and $c_t(x, u) = c(x, u)$.
\item for $1 \leq t \leq n$ and $t \in I_j$, if $y_t \in C_j$: then $B_t(t+1) = [-1, 0]^{\top}$, $B_t(n+1) = [1, 0]^{\top}$ and $B_t(i) = [0, 0]^{\top}$ for all the other $i \neq t+1, n+1$; the cost is $c_t(x, u) = c(x, u) - u(1) \cdot (1-x(n+1))$ rewarding the action $[1, 0]$ which corresponds to assigning the literal a value $y_t=1$.
\item for $1 \leq t \leq n$ and $t \in I_j$, if $\neg y_t \in C_j$: then $B_t(t+1) = [0, -1]^{\top}$, $B_t(n+1) = [0, 1]^{\top}$ and $B_t(i) = [0, 0]^{\top}$ for all the other $i \neq t+1, n+1$; the cost is $c_t(x, u) = c(x, u) - u(2) \cdot (1-x(n+1))$ rewarding the action $[0, 1]$ which corresponds to assigning the literal a value $y_t=0$.
\item for $t=n+1$, $B_t = \bm{0}_{n+1 \times 2}$ and for $t=n+2$, $B_t(1) = [1, 1]^{\top}$ and $B_t(i) = [0, 0]^{\top}$ for all other $i \neq 1$; for both $t = n+1, n+2$, the costs are $c_t(x, u) = c(x, u)$.
\end{itemize}
Note that the last two rounds $t=n+1, n+2$ for a clause ensure sequential stability and identical starting state $x_1 = [1, \bm{0}_n]^{\top}$.
\begin{lemma}\label{lem:stable_convex}
The described system $(A_t, B_t)$ is sequentially stable and the costs $c_t$ are convex in $x, u$.
\end{lemma}
\begin{proof}
    By the construction of the state matrices $A_t$, we know that for any $t \geq n+2$ the operator $\Phi_t^{[n+2]} = 0$ implying sequential stability of the system. To show convexity, note that the distance function $g(\cdot) = \mathrm{dist}(\cdot, \mathcal{S})$ for any convex and compact set $\mathcal{S}$ is a convex function. More specifically, for $z \in \R^{d_z}$ it is given by $g(z) = \min_{w \in \mathcal{S}} \| z - w \|$. This is straightforward to show: take any $z_1, z_2 \in \R^{d_z}$, and let $w_1, w_2 \in \mathcal{S}$ be the closest points to $z_1, z_2$ respectively, i.e. $\| w_1 - z_1 \| = g(z_1)$ and $\| w_2  - z_2 \| = g(z_2)$. For any $\lambda \in [0, 1]$, given the convexity of the set $\mathcal{S}$, we know that $\lambda w_1 + (1-\lambda) w_2 \in \mathcal{S}$, which concludes the convexity proof for $g$:
    \begin{align*}
    \lambda g(z_1) + (1-\lambda) g(z_2) &= \lambda \|z_1-w_1\| + (1-\lambda) \| z_2 - w_2 \| \\
    &\geq \| \lambda (z_1-w_1) + (1-\lambda) ( z_2 - w_2 )\| \\
    &= \| \lambda z_1 + (1-\lambda) z_2 - (\lambda w_1 + (1-\lambda) w_2) \| \\
    &\geq g(\lambda z_1 + (1-\lambda) z_2) ~.
    \end{align*}
    This means that both $S_x(\cdot)$ and $S_u(\cdot)$ are convex functions since the simplex is convex in any dimension. The construction of the costs $c_t$ is based on these two functions as well as linear components in $x, u$, hence all $c_t$ costs are convex in $x, u$. 
\end{proof}

\begin{lemma}\label{lem:lb_sound}
If there exists an assignment of literals $y \leftarrow v$ s.t. the formula $\phi$ has $k \in [1, m]$ satisfied clauses, then there is a corresponding linear policy $K \in \R^{2 \times n+1}$ that suffers the exact cost of $-k$.
\end{lemma}
\begin{proof}
    This should be evident from the construction itself. Let $v_i = 1$ assignment correspond to $\bar{v}_i = [1, 0]^{\top}$ and $v_i = 0$ to $\bar{v}_i = [0, 1]^{\top}$. Then consider the linear policy $K = [\bar{v}_1, \dots, \bar{v}_n, \bm{0}_2^{\top}]$. Denote $e_1, \dots, e_{n+1} \in \R^{n+1}$ to be the basis vectors of the space. Note that according to the defined $K$, if $x_t$ is a basis vector of $\R^{n+1}$, then $u_t = K x_t$ is a basis vector of $\R^2$. It is straightforward to check by our construction that $u_t$ being a basis vector of $\R^2$ implies $x_{t+1}$ is a basis vector of $\R^{n+1}$. Since $x_1 = e_1$, then the state-action pairs when following policy $K$ are both basis vectors, and satisfy the regularity conditions of $c(x, u)$. This means that the policy $K$ plays $\bar{v}_t$ if $x_t(t)=1$ and plays $\bm{0}_2$ if $x_t(n+1)=1$. Hence, if the clause $C_j$ is satisfied by the assignment $y \leftarrow v$, then the cost of $K$ over the episode is  exactly $-1$, i.e. once the clause is satisfied, $-1$ is accrued and the state moves to the sink $e_{n+1}$. If the clause is not satisfied, then the cost is $0$ since $c(x, u)$ is $0$ throughout. This means that the constructed linear policy $K$ over the whole control sequence suffers cost $-k$.i k 
\end{proof}

\begin{lemma}\label{lem:lb_complete}
If there exists a state feedback policy $K \in \R^{2 \times n+1}$ s.t. following the actions $u_t = K x_t$ results in cost at most $-k-\epsilon$ for any $k \in [1, m]$ and any $\epsilon \in (0, 1)$, then there is a literal assignment $y \leftarrow v$ s.t. the formula $\phi$ has at least $k$ satisfied clauses.
\end{lemma}
\begin{proof}
    Let the linear policy matrix be given as $K = [\bar{v}_1, \dots, \bar{v}_n, \bar{v}_{n+1}]$. The proof consists of two main components: (i) we argue that the policy $K^*$ with $\bar{v}_i^* = \argmin_{v \in \Delta_2} \|v-\bar{v}_i\|$ for $1 \leq i \leq n$ and $\bar{v}_{n+1}^* = \bm{0}_2$ is at least as good as $K$ in terms of cost up to an approximation factor $\epsilon$; (ii) we show that for $K$ that satisfies the constraints, the randomized policy $\hat{K}$ that has $\hat{K}(i) = [1, 0]^{\top}$ w.p. $\bar{v}_i(1)$ and $\hat{K}(i) = [0, 1]^{\top}$ w.p. $\bar{v}_i(2)$ (as well as $\hat{K}(n+1) = \bm{0}_2$) suffers expected cost at most that of $K$ itself.
    
    Suppose these two claims are true, then the described randomized linear policy $\hat{K}$ has expected cost at most $-k$, which means that there exists a deterministic linear policy with first $n$ columns as basis vectors, i.e. $[1, 0]^{\top}$ or $[0, 1]^{\top}$, that suffers cost at most $-k$. It follows that the corresponding assignment of literals given by the first $n$ columns of the linear policy $y \leftarrow v$ satisfies at least $k$ out of the $m$ clauses, so $\phi$ has at least $k$ satisfied clauses.
    
    To prove (i), first notice that the policy $K^*$ suffers non-positive cost over the entire horizon since it satisfies the necessary constraints given by the regulatory cost $c(x, u)$. Note also that said $c(x, u)$ can be scaled by any constant $M_{\epsilon}>0$. Now suppose that under the condition $\min_i \| \bar{v}_i-\bar{v}_i^* \| \leq \ell_{\epsilon}$ the cost difference of $K^*$ and $K$ is bounded by $\epsilon$: the choice of $\ell_{\epsilon}$ can depend on any problem parameters, and since the construction is over $T=\Theta(mn)$ overall rounds (finite), such a choice is always possible. Hence, for all such $K^*$ we automatically infer that it suffers cost at most $-k$ and satisfies the necessary constraints. On the other hand, if the condition does not hold, then the distance of $\bar{v}_i$ from $\Delta_2$ is bounded from below by $\ell_{\epsilon}$, meaning that for a sufficiently large choice of $M_{\epsilon} > 0$ (given knowledge of $\epsilon$ and all other parameters), the overall cost suffered by $K$ will be positive due to $c(x, u)$, i.e. it will have a higher cost than $K^*$. Therefore, any state feedback policy $K$ can be approximately replaced by $K^*$ that satisfies the constraints ensured by $c(x, u)$.
    
    To show (ii), for a policy $K$ that does satisfy these constraints, i.e. $\bar{v}_i \in \Delta_2$ for $1 \leq i \leq n$ and $\bar{v}_{n+1} = \bm{0}_2$, we show that its randomized version $\hat{K}$ is at least just as good in terms of expected cost. Proving this claim is a matter of unrolling the dynamics for a single clause $C_j$. The order, the indices and negation or not of the literals in $C_j$ does not affect the cost, so w.l.o.g. assume we have $C_1 = y_1 \lor y_2 \lor y_3$. The cost of a general policy $K$ over first $3$ iterations is given by
    $$-\bar{v}_1(1)-(1-\bar{v}_1(1))^2 \cdot \bar{v}_2(1) - (1-\bar{v}_1(1))^2 \cdot (1-\bar{v}_2(1))^2 \cdot \bar{v}_3(1)$$
    The alternative randomized linear policy instead suffer expected cost given by
    $$-\bar{v}_1(1)-(1-\bar{v}_1(1)) \cdot \bar{v}_2(1) - (1-\bar{v}_1(1)) \cdot (1-\bar{v}_2(1)) \cdot \bar{v}_3(1)$$
    which is straightforward to show to be not larger than the original cost.
\end{proof}
\begin{proof}[Proof of \Cref{thm:comp_lb}] \Cref{lem:stable_convex} above show that the given LTV system construction along with the costs satisfies the theorem conditions. \Cref{lem:lb_sound,lem:lb_complete} indicate that the \maxsat{} problem can be reduced to the LTV control, in particular proper learning of state feedback policies in this setting. Given that \maxsat{} is $\NPtime$-Hard even in its decision form implies the computational hardness of the offline optimization of LTV optimal state feedback policies.
\end{proof}

\end{document}